\documentclass[12pt, a4paper]{article}

\usepackage[a4paper]{geometry}
\usepackage[notcite]{showkeys}
\usepackage{pgfplots}
\pgfplotsset{
    x tick style={color=black},
    y tick style={color=black}
}

\usepackage{url}
\usepackage{listings}
\usepackage{amssymb}
\usepackage{graphicx}
\usepackage{algcompatible}
\usepackage{algorithm}
\usepackage{url}
\usepackage{rotating}
\usepackage{booktabs}
\usepackage{subfig}
\usepackage{amsmath}
\usepackage{bbm}

\renewcommand{\labelenumi}{(\alph{enumi})}
\renewcommand\theenumi\labelenumi

\usepackage[english]{babel}

\usepackage[utf8]{inputenc}
\usepackage{xspace}
\usepackage{amsmath,amsthm,amssymb,mathtools}
\usepackage{lmodern}

\usepackage[algo2e,ruled,vlined,linesnumbered]{algorithm2e}

\usepackage{xcolor}
\usepackage{tikz}
\usepackage{graphicx}

\allowdisplaybreaks[4]
\newtheorem{theorem}{Theorem}
\newtheorem{lemma}[theorem]{Lemma}
\newtheorem{corollary}[theorem]{Corollary}
\newtheorem{definition}[theorem]{Definition}

\newcommand{\oea}{${(1 + 1)}$~EA\xspace}

\newcommand{\oplea}{$(1+\lambda)$~EA\xspace}
\newcommand{\mplea}{${(\mu+\lambda)}$~EA\xspace}
\newcommand{\mpoea}{$(\mu+1)$~EA\xspace}

\newcommand{\om}{\textsc{OneMax}\xspace}
\newcommand{\onemax}{\om}

\newcommand{\R}{\ensuremath{\mathbb{R}}}

\newcommand{\N}{\ensuremath{\mathbb{N}}} 
\newcommand{\Z}{\ensuremath{\mathbb{Z}}}

\newcommand{\calA}{\ensuremath{\mathcal{A}}} 
 
\DeclareMathOperator{\Bin}{Bin}
\DeclareMathOperator{\XOR}{XOR}
\DeclareMathOperator{\mrand}{mrand}
\DeclareMathOperator{\crand}{crand}

\DeclareMathOperator{\ridx}{ridx}
\DeclareMathOperator{\Bernoulli}{Bernoulli}

\DeclareMathOperator{\TRUE}{TRUE}

\newcommand{\Var}{\mathrm{Var}\xspace} 
\newcommand{\eps}{\varepsilon}

\begin{document}
\sloppy

\title{Working Principles of Binary Differential Evolution\thanks{This is a significantly extended version of the 8-page conference paper~\cite{ZhengYD18}. All authors of this version have contributed equally. The authors are given in alphabetical order as common in theoretical computer science.}}

\author{Benjamin Doerr\\ \'Ecole Polytechnique, CNRS,\\ Laboratoire d'Informatique (LIX),\\ Palaiseau, France
\and Weijie Zheng\footnote{Work done while affiliated with Department of Computer Science and Technology in Tsinghua University and partially during a research stay at \'Ecole Polytechnique's computer science lab (LIX).}\\ Luoyanghe Village, \\ Huoshan Economic Development Zone, \\ Anhui, China}

\maketitle

\begin{abstract}
We conduct a first fundamental analysis of the working principles of binary differential evolution (BDE), an optimization heuristic for binary decision variables that was derived by Gong and Tuson (2007) from the very successful classic differential evolution (DE) for continuous optimization. We show that unlike most other optimization paradigms, it is stable in the sense that neutral bit values are sampled with probability close to $1/2$ for a long time. This is generally a desirable property, however, it makes it harder to find the optima for decision variables with small influence on the objective function. This can result in an optimization time exponential in the dimension when optimizing simple symmetric functions like OneMax. On the positive side, BDE quickly detects and optimizes the most important decision variables. For example, dominant bits converge to the optimal value in time logarithmic in the population size. This enables BDE to optimize the most important bits very fast. Overall, our results indicate that BDE is an interesting optimization paradigm having characteristics significantly different from classic evolutionary algorithms or estimation-of-distribution algorithms (EDAs).

On the technical side, we observe that the strong stochastic dependencies in the random experiment describing a run of BDE prevent us from proving all desired results with the mathematical rigor that was successfully used in the analysis of other evolutionary algorithms. Inspired by mean-field approaches in statistical physics we propose a more independent variant of BDE, show experimentally its similarity to BDE, and prove some statements rigorously only for the independent variant. Such a semi-rigorous approach might be interesting for other problems in evolutionary computation where purely mathematical methods failed so far. 
\end{abstract}

\section{Introduction}
The family of differential evolution (DE) heuristics, first proposed by Storn and Price in 1995~\cite{StornP97}, has become one of the most successful branches of evolutionary computation in continuous optimization and has been applied with great success to many real world problems, see, e.g., the survey~\cite{DasMS16}.

However, compared to the abundance of results in continuous optimization, DE for discrete search spaces is much less understood. The difficulties start with how to implement the inherently continuous working principles of DE in discrete search spaces. One approach is to embed the discrete optimization problem into a continuous setting and then utilize continuous DE. For instance, Pampar\'a, Engelbrecht, and Franken~\cite{PamparaEF06} employ angle modulation to generate binary strings from floating-point individuals. Engelbrecht and Pampar\'a~\cite{EngelbrechtP07} further use the sigmoid value of the individual as the probability to generate the bit value, and also propose a normalization mapping. 

Much less effort has been put into the design of truly discrete DE algorithms. Historically the first to do so, to the best of our knowledge, are Gong and Tuson~\cite{GongT07}. They apply the rigorous forma analysis method to derive in a generic way a DE variant for binary search spaces. Moraglio and Togelius as well as Moraglio, Togelius and Silva~\cite{MoraglioT09,MoraglioTS13} define discrete versions of DE via another generic approach, namely by requiring that certain geometric properties of the operators should be maintained. They demonstrate the usefulness of this approach not only for binary representations, but also for permutations and vectors of permutations. Recently, Santucci, Baioletti and Milani~\cite{SantucciBM16} propose another differential mutation for permutation. 

To the best of our knowledge, apart from the axiomatic definitions of the different binary DE algorithms, there are no theoretical analyses of these methods so far. This contrasts the increasing theoretical understandings on other evolutionary algorithms like simple mutation-based algorithms~\cite{DrosteJW02}, the compact Genetic Algorithm (cGA)~\cite{Droste06}, ant colony optimizers~\cite{Gutjahr08,NeumannW09}, and the univariate marginal distribution algorithm
(UMDA)~\cite{ChenTCY10}. The lack of theoretical work on binary DE could be caused by the relatively complicated dependencies in the stochastic process of a run of a DE heuristic. There are two types of the stochastic dependencies in DE, one from the reusing the same individuals when generating the mutant, and the other from the selection operator. As we shall see in this work, these dependencies lead to difficulties not seen in the analysis of the other evolutionary algorithms, which often treat the different bit positions independently (apart from the fitness-based selection). 

\paragraph*{Our results:}
Since a theoretical understanding of an evolutionary algorithm can be very useful for its future use, this paper conducts a first fundamental analysis of the working principles of the binary differential evolution (BDE) algorithm proposed by Gong and Tuson~\cite{GongT07}. We concentrate on this BDE, since it is the historically first and because we feel that its derivation via forma analysis makes it most likely that it inherits the true nature of DE from the continuous world. However, we expect that our results in a similar manner hold for other variants of BDE. 

We show that the stochastic dependencies discussed above lead to a behavior significantly different from what is observed with many other nature-inspired optimization heuristics, in particular those, for which a solid theoretical understanding exists. For example, many heuristics have the property that at any time any point of the search space can be generated (possibly with a small probability). For BDE, this is substantially different. We show that from the random initial population, only an exponentially small fraction of all individuals can be generated in one iteration (see Theorem~\ref{thm:badcase1}). In a similar vein, we present an objective function $f$ and a population $P$ such that BDE from this population with probability $1$ never finds the optimum of $f$. Here $P$ can be chosen exponentially large in the dimension and for each bit position each value may occur exponentially often (Theorem~\ref{thm:badcase2}). 

Unlike most other optimization paradigms for bit-string representations, we show that BDE is stable in the sense of Friedrich et al.~\cite{FriedrichKK16}, that is, neutral bit values are sampled with probability close to $\tfrac{1}{2}$ for a long time. We prove that BDE is stable when optimizing the Needle function, in which all bits are neutral before the optimum is found. Here, precisely, we show that for a time exponential in the population size all bit values are sampled with frequencies in $[\tfrac{1}{2}-\epsilon, \tfrac{1}{2}+\epsilon]$, where $\epsilon>0$ can be any small constant (Theorem~\ref{thm:stable}). 

The inherent dependencies in BDE prevent us from mathematically extending this stability result to arbitrary neutral bits. Therefore, similar to the mean-field approach in statistical physics, we analyze a simpler but similar model called iBDE in which each bit position is treated independently when generating the mutant. We experimentally show the similarity of the behavior between BDE and iBDE in neutral bits and theoretically show the stability of iBDE (Theorem~\ref{thm:iBDENeutral}). As a contrast, extending and sharpening results from~\cite{FriedrichKK16} (partially also mentioned without proof in~\cite{SudholtW16}), we show that in the algorithms UMDA and cGA, the sampling frequency of a neutral bit hits the absorbing boundaries 0 and 1 (or the artificial boundaries $\frac1D$ and $1-\frac1D$ when these are used) in expected times $\Theta(\mu)$ and $\Theta(K^2)$, see Section~\ref{sec:neutralothers}.

As a second positive property, we show that BDE can quickly detect and optimize the most important decision variables. For instance, we prove rigorously that a dominant bit converges to the optimal value in time logarithmic in the population size (Theorem~\ref{thm:domiTime}). We theoretically discuss the runtime of BDE for the LeadingOnes function under the assumption that the frequency of the ones in the population does not drop below a small constant fraction for a sufficiently long time. In this case, BDE finds the optimum of the $D$-dimensional LeadingOnes function in $O(D)$ iterations (Theorem~\ref{thm:BDEforLOwAssumption}). Similar to the discussion for neutral bits, we mathematically verify that this assumption holds for iBDE and experimentally show the similarity of BDE and iBDE in this respect. Analogous results hold for the optimization of the BinaryValue function (Theorem~\ref{thm:BDEforBVwAssumption} and Lemma~\ref{lem:BVassumiBDE}). 

Although stability is generally a desirable property (see~\cite{FriedrichKK16,DoerrK18} for examples how stable EDAs can outperform common EDAs, which are all unstable), stability can make it hard to find the optimal values of decision variables with small influence on the objective function. We take the OneMax function as an example, and prove that the expected runtime is at least exponential in the dimension when we initialize the population by setting each bit to $1$ with probability $0.6$ (Theorem~\ref{thm:rtlargeprob}). Note that such random individuals are actually better (in terms of fitness) than the usual random individuals having ones with probability $0.5$.

This result could indicate that generally BDE has difficulties with objective functions in which each bit position has only a small influence on the fitness. Such a behavior was previously observed~\cite{DoerrK18arxiv} for some algorithms which optimize dominant bits very fast, e.g., the CSA and the sc-GA. Our experimental analysis (in Section~\ref{sec:onemaxexp}) of the BDE optimizing OneMax is not fully conclusive, but indicates that the runtime of BDE on OneMax is super-polynomial. At the same time, we observe that for reasonable problem sizes BDE with the parameters suggested in the literature still optimizes OneMax in a reasonable time. However, we also observe that BDE profits almost not at all from larger population sizes (as long as the population size is large enough to prevent premature convergence).

The organization of the remainder of the paper is as follows. In Section 2, we give a brief introduction to BDE as proposed by Gong and Tuson~\cite{GongT07}. The stochastic dependencies and the proposed mean-field approaches are discussed in Section 3. Section 4 analyzes the behavior of neutral bits, whereas dominant bits are discussed in Section 5. Section 6 discusses possible negative consequences from stability for easy objective function. Section 7 concludes our work.

\section{Binary Differential Evolution}

This paper discusses the optimization behavior of Binary Differential Evolution (BDE) as proposed by Gong and Tuson~\cite{GongT07}. We concentrate on the variant DE/res/bin~\cite{GongT07}. This BDE algorithm with binomial crossover is shown in Algorithm~\ref{alg:originalBDE}. Throughout this paper, we consider the maximization of a $D$-dimensional pseudo-Boolean function $f: \{0,1\}^D \rightarrow \R$. If not indicated differently, the initial population $P^0$ consists of $N$ randomly generated individuals. 
\begin{algorithm}[!ht]
    \caption{originalBDE}
    {\small
    \begin{algorithmic}[1]
    \STATE {Generate the random initial population $P^0=\{X_i^0,i=1,\dots,N\}$}
    \FOR {$g=0,1,2,\dots$}
    \FOR {$i=1,2,\dots,N$}
    \STATEx {\quad\quad$\%\%$ \textit{Mutation}}
    \STATE {Generate mutually different $r_{1}, r_{2}, r_{3}$ from $\{1, \dots, N\} \backslash \{ i \}$} uniformly at random
    \STATE {Generate a random number $\mrand_{j} \in [0,1]$ for each $j\in\{1,\dots,D\}$}
    \STATE {Define the mutant $V_{i}^g$ via
    \begin{equation*}
    \begin{matrix}
    \mathrm{for}\ j\in\{1,\dots,D\},
    &
    \begin{aligned}
    V_{i,j}^g = \begin{cases}
    1-X_{r_{1},j}^g, & \text{if $X_{r_{2},j}^g \neq X_{r_{3},j}^g \ \mathrm{and} \ \mrand_j < F$;}\\
    X_{r_{1},j}^g, & \text{otherwise.}
    \end{cases}
    \end{aligned}
    \end{matrix}
    \end{equation*}
    }
    \STATEx {\quad\quad$\%\%$ \textit{Binomial Crossover}}
    \STATE {Generate a random number $\crand_{j} \in [0,1]$ for each $j\in\{1,\dots,D\}$}
    \STATE {Define the trial $U_{i}^g$ via
    \begin{equation*}
    \begin{matrix}
    \mathrm{for}\ j\in\{1,\dots,D\},
    &
\begin{aligned}
  U_{i,j}^g = \begin{cases}
  V_{i,j}^g, & \text{if $\crand_j \leq C$;}\\
  X_{i,j}^g, & \text{otherwise.}
  \end{cases}
 \end{aligned}
 \end{matrix}
\end{equation*}}
    \STATEx {\quad\quad$\%\%$ \textit{Selection}}
    \STATE {Select $X_i^{g+1}$ via 
    \begin{equation*}
    X_i^{g+1}=
    \begin{cases}
    X_i^{g}, & \text{if $X_i^g$ has the better fitness;}\\
    U_i^g, & \text{if $U_i^g$'s fitness is better or as good as $X_i^g$'s.}
    \end{cases}
    \end{equation*}}
    \ENDFOR
    \ENDFOR
    \end{algorithmic}
    \label{alg:originalBDE}
    }
\end{algorithm}

In the main optimization loop, for each individual $X_i^g$ of the parent population, a \emph{mutant} $V_i^g$ is generated as follows. Three mutually different indices $r_1,r_2$ and $r_3$ are picked randomly from $\{1,\dots,N\} \backslash \{i\}$. The individual $X_{r_1}^g$ is called the \emph{base vector}. The individuals $X_{r_2}^g$ and $X_{r_3}^g$ together with the random numbers $\mrand_j$ determine whether the $j$-th bit of  $X_{r_1}^g$ is flipped ($V_{i,j}^g = 1-X_{i,j}^g$) or not. 

Then a crossover between the mutant $V_i^g$ and its parent $X_i^g$ determines the \emph{trial vector} $U_i^g$. Among the two crossover operators commonly used in DE, exponential crossover and binomial crossover, we only discuss binomial crossover as this is closer to what is commonly used in discrete evolutionary optimization. Also, the experimental results conducted in~\cite{GongT07} suggest that binomial crossover leads to better results on the typical benchmark problems of the theory community. The binomial crossover of DE is a biased uniform crossover such that, for each bit position $j \in \{1,\dots,D\}$ independently, the trial $U_{i}^g$ inherits the $j$-th bit from $V_{i}^g$ with probability $C$, otherwise we have $U_{i,j}^g = X_{i,j}^g$.

Traditionally, in DE one ensures that the trial vector inherits at least one bit position from the mutant vector. For this, a random index $\ridx \in \{1,\dots,D\}$ is chosen and $U_{i}^g$ is defined by
\begin{align*}
  U_{i,j}^g = \begin{cases}
  V_{i,j}^g, & \text{if $\crand_j \leq C$ or $j=\ridx$}\\
  X_{i,j}^g, & \text{otherwise},
  \end{cases}
\end{align*}
that is, we enforce the bit position $\ridx$ to be taken from the mutant. In this first theoretical analysis of BDE, we omit this mechanism. The main reason is that it adds another technicality, but one which most likely does not change a lot. Note that the probability that (without this mechanism) no bit is taken from the mutant, is $(1-C)^D$, that is, exponentially small in $D$. Therefore, it is highly unlikely that during a polynomial runtime of the algorithm such an event happens. Hence throughout the paper, to make the analysis simpler, we omit this additional technicality.

The final step of BDE is an \emph{offspring-parent selection}. If the trial vector $U_i^g$ is at least as good (in terms of fitness) as its parent $X_i^g$, then it replaces the parent, that is, we have $X_i^{g+1}=U_i^g$. Otherwise, the parent $X_i^g$ will enter the next generation as $X_i^{g+1}$.

\section{Stochastic Dependencies and Mean-Field Approaches}

In this section, we demonstrate that the additional stochastic dependencies present in the random process describing a run of BDE lead to a significantly different behavior than what is observed in other evolutionary approaches. Inspired by mean-field approaches in statistical physics, we then propose a BDE variant with fewer dependencies. We shall see later in this work that it gives good approximations for the true BDE process. 

\subsection{Stochastic Dependencies}

From the description of BDE in the previous section, we observe a large number of the stochastic dependencies in BDE. In the mutation operator, three other individuals are used to generate the mutant. For this reason, the bits of the mutant are far from being independent. As we shall see, this has drastic consequences on which offspring can be generated in one generation and on the convergence behavior of BDE. The second type of dependencies stems from the selection operator. Selection always is a cause for dependencies, since it does not regard bits independently, but their combined influence on the fitness. For BDE, things are made worse by the parent-offspring selection mechanism which does not enable a competition between all parents and offspring. 

It is quite likely that BDE rather profits from these dependencies as they might favor the creation and survival of building blocks (in the mutation step) and favor diversity (in the selection step). From the view-point of gaining a rigorous understanding of the working principles of BDE, these dependencies create significant challenges, unfortunately. In the remainder of this section, we prove three results which show that and how the dependencies lead to a behavior significantly different from that of many other evolutionary algorithms, in particular those, for which a substantial theoretical understanding exists.

\subsubsection{Reachable Offspring} 
We say that an individual $X$ is \emph{reachable} from a parent population $P$ if $X$ can be generated with positive (possibly very small) probability from the parent population. In many evolutionary algorithms, each search point $X$ is reachable from any population. This is immediate for all algorithms which use standard bit mutation (flipping each bit independently with some probability like $1/D$). For most distribution-based heuristics like estimation-of-distribution-algorithms (EDAs) or ant colony optimizers (ACOs), again any search point can be generated as long as none of the frequencies or pheromone values (usually initialized at $\tfrac 12$) has converged to $0$ or $1$. When, as often done, these methods are used with artificial boundaries, preventing the frequencies or pheromone values from leaving an interval like $[\frac 1D, 1-\frac 1D]$, then at all times any search point is reachable. 

We now show that BDE is substantially different in this respect. It is clear that once a bit value has \emph{converged}, that is, in all individuals of the population the value of this bit is identical, then in all future individuals this bit will have this same value (and consequently, not all individual are reachable). However, also long before this convergence, in fact, already right after the initialization, with high probability, the vast majority of the individuals cannot be reached in one generation. The following result shows that for a given target search point $X^*$, with very high probability, starting from the initial random population, this $X^*$ and all search points in Hamming distance at most $\eps D$, $\eps < \frac 18$, cannot be reached in one generation. 
%
%

\begin{theorem}
Consider using BDE with population size $N$ to optimize a $D$-dimensional function $f$. Let $X^*$ be any target search point and $c\in(0,\tfrac{1}{8})$. Then with probability at least $1-N^4\exp(-2c^2D)$, BDE can generate no search point $X$ with Hamming distance $H(X,X^*) \le (\tfrac{1}{8}-c)D$ from the random initial population.
\label{thm:badcase1}
\end{theorem}

Note that the upper bound $N^4\exp(-2c^2D)$ for the probability of being able to generate some search point $X$ with $H(X,X^*) \le (\tfrac{1}{8}-c)D$ is exponentially small in $D$ unless we work with an exponentially large population. 

\begin{proof}[Proof of Theorem~\ref{thm:badcase1}]
Let $X^* \in \{0,1\}^D$. Let $(i,r_1,r_2,r_3)$ be four mutually exclusive indices from $\{1,\dots,N\}$ and let $X_i,X_{r_1},X_{r_2},X_{r_3}$ be  the corresponding individuals from the random initial population $P^0$. For all $j \in \{1, \dots, D\}$ let $Y_j$ be the indicator random variable for the event 
\[(X_{i,j} \neq X^*_j) \wedge (X_{r_1,j}=X_{i,j}) \wedge (X_{r_2,j}=X_{r_3,j}).\] 
Note that by the definition of BDE, the event $Y_j = 1$ implies that any mutant $V$ arising from $X_{r_1},X_{r_2},X_{r_3}$ has $V_j = X_{i,j}$. Consequently, any trial vector $U$ generated from these individuals has $U_j = X_{i,j}$. Hence regardless of the result of the parent-offspring selection, the individual $X_i^{1}$ of the next generation will be different from $X^*$ in the $j$-th bit position. 

We now estimate the number $Y := \sum_{j = 1}^D Y_j$ of bit positions in which any offspring from $X_i,X_{r_1},X_{r_2},X_{r_3}$ necessarily differs from $X^*$. Since the $X_i,X_{r_1},X_{r_2},X_{r_3}$ are independently generated random individuals, we have $\Pr[Y_j = 1] = \frac 18$ and thus $E[Y]=\tfrac{1}{8}D$. Since further the bit-positions of random individual are also independent, the $Y_j, j \in \{1, \dots, D\}$, are mutually independent as well. Consequently, the classic additive Chernoff bound (see, e.g., Theorem~1.11 in~\cite{Doerr11bookchapter}), shows that  
\begin{equation}
\Pr[Y \le (\tfrac{1}{8}-c)D] \le \exp(-2c^2D).
\label{eq:Yi}
\end{equation}

Let $A(i,r_1,r_2,r_3)$ represent the event that for the given choice $(i,r_1,r_2,r_3)$, the corresponding individuals $X_i,X_{r_1},X_{r_2},X_{r_3}$ are able  to generate some offspring $X_i^1$ with $H(X^1_i,X^*) < (\tfrac{1}{8}-c)D$. From (\ref{eq:Yi}), we have 
\begin{align*}
\Pr[A(i,r_1,r_2,r_3)] = \Pr[Y \le (\tfrac{1}{8}-c)D] \le \exp(-2c^2D).
\end{align*}

The event that some individual $X$ with Hamming distance $H(X,X^*)  \le (\tfrac{1}{8}-c)D$ can be generated from the initial population is the union of the events $A(i,r_1,r_2,r_3)$ over all choices of $(i,r_1,r_2,r_3)$. Hence the probability that some individual $X$ with Hamming distance $H(X,X^*)  \le (\tfrac{1}{8}-c)D$ can be generated from the initial population, is at most
\begin{equation*}
\begin{split}
\Pr[\exists{}& (i,r_1,r_2,r_3): A(i,r_1,r_2,r_3)]\\
={}&\Pr\bigg[\bigcup\limits_{(i,r_1,r_2,r_3)} A(i,r_1,r_2,r_3)\bigg]\\
\le{}& \sum\limits_{(i,r_1,r_2,r_3)} \Pr[A(i,r_1,r_2,r_3)]\\
\le{}&N(N-1)(N-2)(N-3)\exp(-2c^2D).
\end{split}
\end{equation*}

\end{proof}

The same argument as above leads to a global view on the problem, namely that the expected number of reachable individuals is very small (compared to the size $2^D$ of the search space). Recall here that reachable does not mean that the individual is generated or it is likely to be generated, it just means that there is a theoretical chance that it shows up as offspring. Hence the following result shows that for the vast majority of individuals it is a priori clear that they cannot show up as offspring of the initial population. 

\begin{theorem}
  The expected number of individuals which are reachable from the random initial population is at most $N^4 (\frac 78)^D \cdot 2^D = N^4 1.75^D$.
\end{theorem}
\begin{proof}
Let $(i,r_1,r_2,r_3)$ be four mutually exclusive indices from $\{1,\dots,N\}$ and let $X_i,X_{r_1},X_{r_2},X_{r_3}$ be the corresponding individuals from the random initial population $P^0$. We compute the expected number of different offspring which could be generated from the fixed indices $(i,r_1,r_2,r_3)$. For all $j \in \{1, \dots, D\}$, let the random variable $W_j$ be $W_j= 1$ if all possible offspring $U$ satisfy $U_{j}=X_{i,j}$, and let $W_j=2$ otherwise. It is easy to see that 
\[\Pr[W_j=1]=\Pr[(X_{r_1,j}=X_{i,j}) \wedge (X_{r_2,j}=X_{r_3,j})]=\tfrac{1}{4}\] 
and thus $E[W_j]=\tfrac{7}{4}$. Now the number of different individuals that can be generated from $X_i,X_{r_1},X_{r_2},X_{r_3}$ is $Z=\prod_{j=1}^D W_j$. Since the $W_j$ are independent, we have $E[Z]=(\tfrac{7}{4})^D$. Via a union bound over the choices of $(i,r_1,r_2,r_3)$, we obtain that the expected number of reachable individuals is at most $N^4 (\tfrac{7}{4})^D=N^4 1.75^D$.
\end{proof}

Our two results on reachability only show that from the initial population very few individuals can be reached. Due to the complicated randomized process describing a run of BDE, we cannot show such a result for all iterations of BDE. We would suspect, though, that in a typical run of BDE on a typical optimization problem this phenomenon exists throughout the run and rather becomes stronger due to loss of diversity.

\subsubsection{Convergence}
The fact that not all search points can be generated at all times implies that the classic convergence proofs fail for BDE. We now show that not only the classic proofs fail, but that indeed BDE does not necessarily converge, and this even when the initial population is large and highly diverse (in the sense that at all bit positions all bit values occur frequently). This result again demonstrates that the stochastic dependencies inherent in the search process lead to an optimization behavior substantially different from what is observed in classic evolutionary algorithms.

Before stating and showing this result, we note that the random initial population with probability $1 - (1 - 2^{-N+1})^D \le D 2^{-N+1}$ contains a bit position in which all individuals have the same bit value (``converged bit''), which could be a trivial reason for non-convergence. However, as the above estimate shows,  for $N$ mildly larger than $\log_2 D$ the initial population with high probability contains both zeros and ones in each bit position, so this problem is easy to avoid (and it would also be easy to detect).

\begin{theorem}
There is a fitness function $f : \{0,1\}^D \rightarrow \R$ and an initial population $P^0$ without converged bits such that BDE in an arbitrary long runtime does not find the optimum of $f$. The initial population $P^0$ can be chosen of size exponential in $D$ and with all bit values appearing exponentially often at all positions.
\label{thm:badcase2}
\end{theorem}
\begin{proof}
Consider a function $f: \{0,1\}^D \rightarrow \R$ such that 
\begin{itemize}
\item the global optima (maxima) all are search points $X$ with $\| X \|_1 \ge 0.8D$, and
\item for all search points $X_1$ with $\|X_1\|_1 \in \{0.2D, \dots, 0.8D-1\}$ and $X_2$ with $\| X_2 \|_1 < 0.2D$, we have $f(X_1) < f(X_2)$.
\end{itemize}
We say a population has the property $\calA$ when all individuals $X$ in the population satisfy $\| X \|_1 < 0.2D$. 

Let  $(i, r_1,r_2,r_3)$ denote $4$ mutually different indices from $\{1,\dots, N\}$ and let $X_i^g,X_{r_1}^g,X_{r_2}^g,X_{r_3}^g$ be the corresponding individuals from a population $P^g$ with property $\calA$. Consider 
\begin{equation*}
Z^g=\{j \in \{1, \dots, D\} \mid X_{i,j}^g=X_{r_1,j}^g=X_{r_2,j}^g=X_{r_3,j}^g=0\}.
\end{equation*} 
By the definition of BDE, any mutant $V_i^g$ arising from $X_{r_1}^g,X_{r_2}^g,X_{r_3}^g$ has $V_{i,j}^g=0$ for all $j \in Z^g$. Consequently, any trial vector $U_i^g$ generated from $X_i^g$ and such a mutant has $U_{i,j}^g=0$ for $j\in Z^g$. Since $X_i^g,X_{r_1}^g,X_{r_2}^g$ and $X_{r_3}^g$ are from $P^g$ with property $\calA$, we have $\| X_{i}^g \|_1 < 0.2D,\| X_{r_1}^g \|_1 < 0.2D,\| X_{r_2}^g \|_1 < 0.2D$, and $\|X_{r_3}^g \|_1 < 0.2D$. Hence,
\begin{align*}
|Z^g| > D-0.2D-0.2D-0.2D-0.2D=0.2D.
\end{align*}
Hence there are at least $0.2D$ zeros in $U_i^g$, that is, we have $\| U_i^g \|_1 < 0.8D$.

From this, we immediately conclude that BDE cannot generate any $X$ with $\| X\|_1 \ge 0.8D$ from $P^g$. Moreover, we also observe that $P^{g+1}$ has property $\calA$. Since $\|U_i^g\|_1 < 0.8D$ and $\|X_i^g\|_1 < 0.2D$, we have $f(U_i^g) \ge f(X_i^g)$ only when $\|U_i^g\|_1 < 0.2D$. Hence the successor $X_i^{g+1}$ of $X_i^g$ in the next population in any case has $\|X_i^{g+1}\|_1 < 0.2D$.
 
Hence the next generation $P^{g+1}$ has the property $\calA$ as well. By induction, we obtain that when starting with a population having property $\calA$, we always keep a population with this property, which hence does not contain an optimal solution.

It remains to show that there are initial populations with property $\calA$ that do not have any bit converged. However, this is trivial -- we may just take the set of all $X$ with $\|X\|_1 <0.2D$. This population has size $\binom{D}{<0.2 D} = \exp(\Theta(D))$ and has a fraction of $\frac 15 - o(1)$ of ones in each bit position. However, also an initial population composed of $D$ random strings having a $1$ at each position with probability $\frac 1 {10}$ has property $\calA$ with probability $1 - \exp(-\Theta(D))$.
\end{proof}

Clearly, the construction used in the proof above is artificial. However, it points out that BDE does not necessarily converge, and, more importantly, that non-convergence can be determined already by a population that has no converged bits. 

This also shows that it is a non-trivial problem to detect if a run of BDE has entered a state from which it cannot generate the whole search space anymore. Note that this question is trivial for most EDAs and ACO algorithms since any search point can be generated if and only if there are no converged frequencies or pheromone values, a criterion that is easy to check.

\subsection{Mean-Field Approaches and Independent BDE (iDBE)}

In statistical physics often the situation arises that the stochastic interactions between different particles are too hard to grasp mathematically. A common solution, called mean-field theory, is to disregard some of the dependencies and to conduct a mathematical analysis of the simplified model. The results obtained in the simplified model, naturally, are not immediately valid for the original model, but they can point into the right direction and they can be made plausible by arguing, possibly supported by experiments, that the simplification does not lead to a significant discrepancy of the two models.

Since the dependencies caused by the mutation operator of BDE impose significant difficulties for the mathematical analysis of BDE, we shall resort to a similar approach in some of the following analyses. To this aim, we propose a variant of BDE, called \emph{independent BDE (iBDE)}, which generates the bits of a mutant independently, but is otherwise identical to BDE. More precisely, when generating a mutant $V_i$, for each bit position $j$ independently,  we select mutually different (and different from $X_i$) individuals $X_{r_1,j}, X_{r_2,j}, X_{r_3,j}$ to generate $V_{i,j}$. See Alg.~\ref{alg:iBDE} for the precise pseudocode. 

Whenever in the following sections we resort to analyzing iBDE, we shall also argue for the similarity between iBDE and the original BDE in the particular respect regarded. Note that iBDE and BDE do differ in some respects. For example, the reachability and convergence results shown in this section naturally are not valid for iBDE. When the current population has no converged bits, then any individual can be generated.
\begin{algorithm}[!ht]
    \caption{iBDE}
    {\small
    \begin{algorithmic}[1]
    \STATE {Generate the random initial population $P^0=\{X_i^0,i=1,\dots,N\}$}
    \FOR {$g=0,1,2,\dots$}
    \FOR {$i=1,2,\dots,N$}
    \STATEx {\quad\quad$\%\%$ \textit{Modified Mutation}}
    \FOR {$j=1,2,\dots,D$}
    \STATE {Generate mutually different $r_{1}, r_{2}, r_{3}$ from $\{1, \dots, N\} \backslash \{ i \}$} uniformly at random
    \STATE {Generate a random number $\mrand_{j} \in [0,1]$}
    \STATE {Generate the $j$-th bit position value of the mutant $V_{i}^g$ via
    \begin{equation*}
    \begin{aligned}
    V_{i,j}^g = \begin{cases}
    1-X_{r_{1},j}^g, & \text{if $X_{r_{2},j}^g \neq X_{r_{3},j}^g \ \mathrm{and} \ \mrand_j < F$;}\\
    X_{r_{1},j}^g, & \text{otherwise.}
    \end{cases}
    \end{aligned}
    \end{equation*}
    }
    \ENDFOR
    \STATEx {\quad\quad$\%\%$ \textit{Binomial Crossover}}
    \STATE {Generate a random number $\crand_{j} \in [0,1]$ for each $j\in\{1,\dots,D\}$}
    \STATE {Define the trial $U_{i}^g$ via
    \begin{equation*}
    \begin{matrix}
    \mathrm{for}\ j\in\{1,\dots,D\},
    &
\begin{aligned}
  U_{i,j}^g = \begin{cases}
  V_{i,j}^g, & \text{if $\crand_j \leq C$;}\\
  X_{i,j}^g, & \text{otherwise.}
  \end{cases}
 \end{aligned}
 \end{matrix}
\end{equation*}}
    \STATEx {\quad\quad$\%\%$ \textit{Selection}}
    \STATE {Select $X_i^{g+1}$ via 
    \begin{equation*}
    X_i^{g+1}=
    \begin{cases}
    X_i^{g}, & \text{if $X_i^g$ has the better fitness}\\
    U_i^g, & \text{if $U_i^g$ has the not worse fitness}
    \end{cases}
    \end{equation*}}
    \ENDFOR
    \ENDFOR
    \end{algorithmic}
    \label{alg:iBDE}
    }
\end{algorithm}

\section{Stability, Behavior of Neutral Bits}
\label{sec:neutral}

When a bit-position has no influence on the fitness, then it would make sense that its sampling frequency in EDAs or ACOs stays close to $\frac 12$ for a long time. A property trying to grasp this idea was called \emph{stable} by  Friedrich, K\"otzing, and Krejca~\cite{FriedrichKK16}. Unfortunately, as shown in~\cite{FriedrichKK16}, all classic EDAs and ACOs are not stable. The recent works of Witt~\cite{Witt17} and Lengler, Sudholt, and Witt~\cite{LenglerSW18} show that instability, more precisely, the early and unmotivated move of frequencies to boundary values can lead to a considerable performance loss when optimizing the \onemax function.

In this section, we demonstrate that BDE is more stable than the classic EDAs and ACO algorithms. To this aim, we both show stability results for BDE and iBDE and we show improved instability results for the EDA called UMDA and the compact genetic algorithm (cGA). We start by making precise what we mean by stability.

\subsection{Stability of EDAs and BDE}

Let $f : \{0,1\}^D \to \R$ be an objective function to be optimized. We say that $i \in \{1, \dots, D\}$ is a \emph{neutral} bit-position if for all $x, y \in \{0,1\}^n$ with $x_j = y_j$ for all $j \in \{1, \dots, D\} \setminus \{i\}$ we have $f(x) = f(y)$. In other words, the fitness of a search point does not depend on the value of the $i$-th bit. We note that such a bit-position was called \emph{$f$-independent} in~\cite{FriedrichKK16}.
The following formal definition of stability was given in~\cite{FriedrichKK16}.
\begin{definition}[\cite{FriedrichKK16}]
An $n$-Bernoulli-$\lambda$-EDA A is \emph{stable} if, for all $f$-independent positions $i$ of A, the limit distribution of frequency $p_i^{(t)}$, as $t \rightarrow \infty$, exists and is symmetric around $\tfrac{1}{2}$, taking its maximum at $\tfrac{1}{2}$, and is strictly monotonically decreasing from $\tfrac{1}{2}$ toward the borders.
\end{definition}

Since the main aspect of instability is that frequencies without good reason approach too fast the boundaries, we propose an alternative definition based on the time until a frequency leaves a constant-length region around the middle value $\tfrac 12$. We define this property formally for BDE and use analogous notions for other algorithms. 
\begin{definition}
A BDE with population size $N$ is \emph{stable} if there is a constant $\delta \in (0,\tfrac{1}{2})$ such that for any objective function $f$ and any $f$-independent position~$j$, the frequency $p_j^{(g)} := \frac 1N \sum_{i=1}^N X_{i,j}^g$ with high probability remains in $[\frac{1}{2}-\delta,\frac{1}{2}+\delta]$ for a super-polynomial (in $N$) number of iterations.
\end{definition}

\subsection{Stability of BDE When Optimizing the Needle Function}

As our first argument for the stability of BDE, we prove rigorously that when optimizing the Needle function via BDE, then the bit frequencies stay close to $\frac 12$ for a time exponentially long in the population size $N$. This result stands in sharp contrast to our later results showing, e.g., that the expected time until a neutral bit hits one of the boundary values in a run of the cGA is $O(K^2)$ iterations (where $K$ is the hypothetical population size of the cGA) and is $O(\mu^2)$ iterations for UMDA.

We recall that the $D$-dimensional Needle function is the fitness function $f: \{0,1\}^D \to \{0,1\}$ defined by $f(X) = 1$ if and only if $X = (1,\dots,1)$. Hence up to the hitting time of the optimum, all bits behave neutrally. 

To not obscure the main proof by two lengthy calculations, we formulate their results as separate lemmas before the main proof. These might, nevertheless, be results of independent interest as they compute the dynamics of a single bit subject to mutation and crossover. Since this analysis does not consider selection, it is independent of the fitness function and thus applies to all fitness functions. The main finding in the following lemma (see also Figure~\ref{fig:hx}) is that there is a strong drift towards the middle value of an equal number of zeros and ones. This is the main reason for the fact that bits without a clear fitness-signal stay close to this undecided situation in BDE, unlike for many other algorithms.
 
\begin{lemma}
Consider one iteration of BDE with population size $N$ optimizing some $D$-dimensional function. Let $Y_g$ denote the number of ones in a certain bit position among all individuals of the population of generation $g$. Let $\tilde{Y_g}$ denote the number of ones in this position in the trial population $\{U_1^g, \dots, U_N^g\}$. Then
\begin{align*}
E[\tilde{Y_g}\mid Y_g]=\frac{4FCY_g^3-6FCNY_g^2+((2FC+1)N^2-3N+2)Y_g}{(N-1)(N-2)}.
\end{align*}
\label{lem:expectation}
\end{lemma}

\begin{proof}
%

Without loss of generality, let the certain bit be the first bit. For a given parent $X_i^g$, we recall that $U_i^g$ is generated via a bit-wise recombination of $X_i^g$ and $V_i^g$. We determine the distribution of $U_{i,1}^g$ in the two cases that $X_{i,1}^g=1$ and $X_{i,1}^g=0$.

When $X_{i,1}^g=0$, in order to have $U_{i,1}^g=1$, $U_{i,1}^g$ must stem from $V_{i,1}^g$ and $V_{i,1}^g$ must be 1. This happens in exactly the following three cases.
\begin{itemize}
\item $X_{r_1,1}^g=1, X_{r_2,1}^g = X_{r_3,1}^g, \crand_1 \le C$.
\item $X_{r_1,1}^g=1, X_{r_2,1}^g \ne X_{r_3,1}^g, \mrand_1\ge F, \crand_1 \le C$.
\item $X_{r_1,1}^g=0, X_{r_2,1}^g \ne X_{r_3,1}^g, \mrand_1< F, \crand_1 \le C$.
\end{itemize}
Hence, recalling that $Y_g$ represents the number of ones in the first bit among all individuals of the generation $g$, we obtain
\begin{equation}
\begin{split}
\Pr[U_{i,1}^g{}{}&=1 \mid X_{i,1}^g=0]\\
={}&\frac{Y_g ((Y_g-1)(Y_g-2)+(N-Y_g-1)(N-Y_g-2) )}{(N-1)(N-2)(N-3)}C\\
{}&+\frac{Y_g(N-Y_g-1)(Y_g-1)}{(N-1)(N-2)(N-3)}2(1-F)C\\
{}&+\frac{(N-Y_g-1)Y_g(N-Y_g-2)}{(N-1)(N-2)(N-3)}2FC.
\end{split}
\label{eq:on0}
\end{equation}

Similarly, for $X_{i,1}^g=1$, the possible cases are the following.
\begin{itemize}
\item $X_{r_1,1}^g=1, X_{r_2,1}^g=X_{r_3,1}^g$.
\item $X_{r_1,1}^g=1, X_{r_2,1}^g\ne X_{r_3,1}^g, \mrand_1 \ge F$.
\item $X_{r_1,1}^g=1, X_{r_2,1}^g\ne X_{r_3,1}^g, \mrand_1 < F, \crand_1 > C$.
\item $X_{r_1,1}^g=0, X_{r_2,1}^g=X_{r_3,1}^g, \crand_1 > C$.
\item $X_{r_1,1}^g=0, X_{r_2,1}^g\ne X_{r_3,1}^g, \mrand_1 < F$.
\item $X_{r_1,1}^g=0, X_{r_2,1}^g\ne X_{r_3,1}^g, \mrand_1 \ge F, \crand_1 > C$.
\end{itemize}
Thus we have
\begin{equation}
\begin{split}
\Pr[U_{i,1}^g{}{}&=1\mid X_{i,1}^g=1]\\
={}&\frac{(Y_g-1) ((N-Y_g)(N-Y_g-1)+(Y_g-2)(Y_g-3) )}{(N-1)(N-2)(N-3)}\\
{}&+\frac{(Y_g-1)(Y_g-2)(N-Y_g)}{(N-1)(N-2)(N-3)}2(1-F+F(1-C))\\
{}&+\frac{(N-Y_g) ((N-Y_g-1)(N-Y_g-2)+(Y_g-1)(Y_g-2) )}{(N-1)(N-2)(N-3)}(1-C)\\
{}&+\frac{(N-Y_g)(N-Y_g-1)(Y_g-1)}{(N-1)(N-2)(N-3)}2(F+(1-F)(1-C)).
\label{eq:on1}
\end{split}
\end{equation}

Based on the conditional probabilities (\ref{eq:on0}) and (\ref{eq:on1}), since $\tilde{Y_g}=\sum_{i=1}^N U_{i,1}^g$, we obtain
\begin{equation}
\begin{split}
E[{}\tilde{Y_g}{}& \mid Y_g]\\
={}&(N-Y_g)\Pr[U_{i,1}^g=1 \mid X_{i,1}^g=0]+Y_g\Pr[U_{i,1}^g=1 \mid X_{i,1}^g=1]\\
={}&\frac{(N-Y_g)(N-Y_g-1)Y_g(N-Y_g-2)}{(N-1)(N-2)(N-3)}(2FC+C+1-C)\\
{}&+\frac{(N-Y_g)(N-Y_g-1)Y_g(Y_g-1)}{(N-1)(N-2)(N-3)}(2(1-F)C+1+2(1-C+FC))\\
{}&+\frac{(N-Y_g)Y_g(Y_g-1)(Y_g-2)}{(N-1)(N-2)(N-3)}(C+2(1-FC)+1-C)\\
{}&+\frac{Y_g(Y_g-1)(Y_g-2)(Y_g-3)}{(N-1)(N-2)(N-3)}\\
={}&\frac{(N-Y_g)(N-Y_g-1)Y_g(N-Y_g-2)}{(N-1)(N-2)(N-3)}(2FC+1)\\
{}&+3\frac{(N-Y_g)(N-Y_g-1)Y_g(Y_g-1)}{(N-1)(N-2)(N-3)}+\frac{Y_g(Y_g-1)(Y_g-2)(Y_g-3)}{(N-1)(N-2)(N-3)}\\
{}&+\frac{(N-Y_g)Y_g(Y_g-1)(Y_g-2)}{(N-1)(N-2)(N-3)}(3-2FC)\\
={}&\frac{4FCY_g^3-6FCNY_g^2+((2FC+1)N^2-3N+2)Y_g}{(N-1)(N-2)}.
\label{eq:drift}
\end{split}
\end{equation}
\end{proof}

To gain a better understanding of the quantity $E[\tilde{Y_g} \mid Y_g]$ just computed, let us define (for implicitly given $F$ and $C$) the function $H_N : [0,N]\to [0,N]$ by
\begin{equation*}
H_N(z) = \frac{4FCz^3-6FCNz^2+((2FC+1)N^2-3N+2)z}{(N-1)(N-2)},
\end{equation*}
so that $E[\tilde{Y_g} \mid Y_g] = H_N(Y_g)$. Going from absolute numbers to relative numbers, we also define $h(x)=\frac{H_N(Nx)}{N}$ for all $x \in [0,1]$. Figure~\ref{fig:hx} visualizes this function for two sets of parameter values.

\begin{figure}
\centering
\begin{tikzpicture}
\begin{axis}[
    axis lines = left,
    xlabel = $x$,
    ylabel = {$h(x)$},
    ymin = 0,
    legend style={
      cells={anchor=east},
      legend pos=outer north east,},
    ]

\addplot [
    domain=0:1, 
    samples=100, 
    color=black,
]
{(4*0.9*0.9*50^3*x^3-6*0.9*0.9*50*50^2*x^2+((2*0.9*0.9+1)*50^2-3*50+2)*50*x)/(50*49*48)};
\addlegendentry[right]{$N=50,F=C=0.9$}
 
\addplot [densely dotted,
    domain=0:1, 
    samples=100, 
    color=black,
    ]
{(4*0.9*0.9*10^3*x^3-6*0.9*0.9*10*10^2*x^2+((2*0.9*0.9+1)*10^2-3*10+2)*10*x)/(10*9*8)};
\addlegendentry[right]{$N=10,F=C=0.9$} 

\addplot [dotted,
    domain=0:1, 
    samples=100, 
    color=black,
    ]
{x};
\end{axis}
\end{tikzpicture}
\caption{A visualization of $h(x) := H(Nx) / N$, that is, the expected ratio of ones in a position of the trial population when the parent population has a ratio of $x$ ones in this position. For $N$ not too small, this function is monotonically increasing. However, for $x \in (0,\frac 12)$, it is strictly larger than $x$ and for $x \in (\frac 12,1)$ it is strictly smaller than $x$ as visible from the comparison with the dotted straight line depicting the identity function. In the absence of a strong fitness signal, this leads to a drift of the ratio to $\frac 12$, which is the reason for the stability results we prove in this section.}\label{fig:hx}
\end{figure}
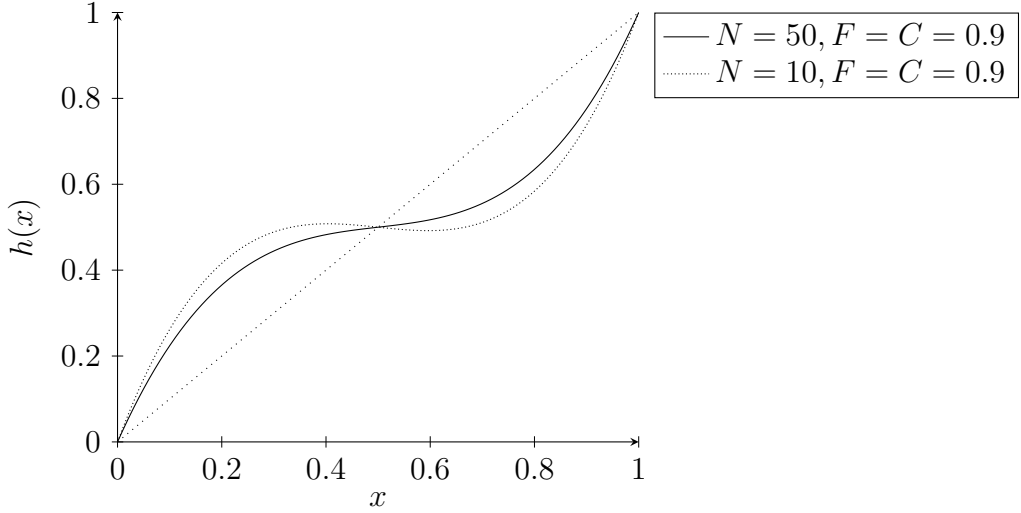

In the following Lemma~\ref{lem:mono}, we show that $E[\tilde{Y_g}\mid Y_g]$ is a monotonically increasing function with respect to $Y_g$ when $N$ is at least some constant (depending on $F$ and $C$).

\begin{lemma}
Let $F, C \in [0,1]$, $FC < 1$, and $N\in [\frac{3}{1-FC},\infty ) \cap \N$. 
Then $H_N(z)$ is monotonically increasing.
\label{lem:mono}
\end{lemma}

\begin{proof}
Let $g(z)=4az^3-6aNz^2+((2a+1)N^2-3N+2)z$, where $a=FC$. It is easy to see that $g(z)$ and $H_N(z)$ have the same monotonicity. We have
\begin{equation*}
\begin{split}
g'(z)&=12az^2-12aNz+(2a+1)N^2-3N+2,\\
g''(z)&=24az-12aN.
\end{split}
\end{equation*}
Since $a>0$, we have $g''(z)\le0$ for $z\le\frac{N}{2}$ and $g''(z)\ge 0$ for $z\ge \frac{N}{2}$. Hence $g'(z)$ has a unique minimum at $z=\frac{N}{2}$. Therefore,
\begin{equation*}
g'(z)\ge g'(\tfrac{N}{2})=(1-a)N^2-3N+2.
\end{equation*}
Since $N\ge \frac{3}{1-a}$, we have 
\begin{align*}
(1-a)N^2-3N+2 \ge 3N-3N+2=2.
\end{align*}
Hence $g'(z)$ is positive and thus $g(z)$ and $H_N(z)$ are monotonically increasing.
\end{proof}

We are now in the position to prove that BDE is stable when optimizing the Needle function, that is, that the frequencies of the ones in all bit positions stay close to $0.5$ for a long time. The precise statement in Theorem~\ref{thm:stable} is formulated for a single bit position, but it is clear that a simple union bound implies that also all bit positions stay close to $0.5$ for a time exponential in $N$ (if $D$ is sub-exponential in $N$).

\begin{theorem}\label{thm:stable}
Consider using BDE with population size $N\ge\max\{\frac{3}{1-FC},\frac{15625\ln 2}{288(FC)^2}\}$ to optimize the $D$-dimensional Needle function. Let $Y_g$ denote the number of ones in a certain bit position among all individuals of generation $g$. There is a constant $c>0$, depending on $F$ and $C$ only, such that
\begin{equation*}
\Pr [\forall g\in\{0,\dots,T\}: Y_g\in[0.4N,0.6N] ]\ge 1-2(T+1)\exp(-cN)
\end{equation*}
for all $T\in \N$.
\end{theorem}
\begin{proof}
We first consider the behavior of $Y_{g+1}$ when $Y_g \in [0.4N, 0.6N]$. Without loss of generality, let the certain bit be the first bit. Then $Y_{g+1}=\sum_{i=1}^NX_{i,1}^{g+1}$. 

By the definition of the Needle function, for a given parent $X_i^g\ne (1,\dots,1)$, we have $f(U_i^g)\ge 0 =f(X_i^g)$ regardless of the value of the trial vector $U_i^g$. Hence, due to the parent-offspring selection, we have $X_i^{g+1}:=U_i^g$ and thus $E[Y_{g+1} \mid Y_g]=E[\tilde{Y_g} \mid Y_g]=R_N(Y_g)$, where $\tilde{Y_g}=\sum_{i=1}^N U_{i,1}^g$ defined in Lemma \ref{lem:expectation}. 

From Lemma \ref{lem:mono}, we know that $4FCY_g^3-6FCNY_g^2+((2FC+1)N^2-3N+2)Y_g$ and thus $E[Y_{g+1} \mid Y_g]$ are monotonically increasing with respect to $Y_g$. For $Y_g=0.4N$, we have
\begin{equation}
\begin{split}
4FC&Y_g^3-6FCNY_g^2+((2FC+1)N^2-3N+2)Y_g\\
&=(\tfrac{2}{5}+\tfrac{12}{125}FC)N^3-\tfrac{6}{5}N^2+\tfrac{4}{5}N\\
&=\tfrac{2}{5}N(N-1)(N-2)+\tfrac{12}{125}FCN^3\\
&\ge(\tfrac{2}{5}+\tfrac{12}{125}FC)N(N-1)(N-2),
\end{split}
\label{eq:2_5N}
\end{equation}
and for $Y_g=0.6N$, we have
\begin{equation}
\begin{split}
4FC&Y_g^3-6FCNY_g^2+((2FC+1)N^2-3N+2)Y_g\\
&=(\tfrac{3}{5}-\tfrac{12}{125}FC)N^3-\tfrac{9}{5}N^2+\tfrac{6}{5}N\\
&=\tfrac{3}{5}N(N-1)(N-2)-\tfrac{12}{125}FCN^3\\
&\le(\tfrac{3}{5}-\tfrac{12}{125}FC)N(N-1)(N-2).
\end{split}
\label{eq:3_5N}
\end{equation}
From (\ref{eq:2_5N}) and (\ref{eq:3_5N}), we conclude
\begin{equation*}
(\tfrac{2}{5}+\tfrac{12}{125}FC)N \le E [Y_{g+1} \mid Y_g\in [0.4N,0.6N] ]\le (\tfrac{3}{5}-\tfrac{12}{125}FC)N.
\end{equation*}

For $i=1,2,\dots,N$, let $Z_i^g$ be the random vector that contains all random variables generated in iteration $i$ of the inner loop, that is, 
\begin{equation*}
Z_i^g=(\mrand_{1},\dots,\mrand_{D},\crand_{1},\dots,\crand_{D},r_{1},r_{2},r_{3}),
\end{equation*}
where each element is the one used in iteration $i$ (for reasons of readability, we suppressed an extra index $i$ in the definition of the algorithm).
It is easy to see that $Z_1^g,Z_2^g,\dots,Z_N^g$ are independent. Given the current population $P^g$, $Y_{g+1}$ can be considered as a function of $Z_g=(Z_1^g,\dots,Z_N^g)$, denoted by $Y_{g+1}=s(Z_g)$. Obviously,  $X_i^{g+1}$ depends only on $Z_i^g$. Since $X_i^{g+1}$ contributes at most one to $Y_{g+1}$, we see that for $Z_g, \tilde{Z_g}$ that differ only in the $Z_i^g$ part, we have $|s(Z_g)-s(\tilde{Z_g})|\le 1$. Applying Azuma's inequality (Theorem 1.15 in \cite{Doerr11bookchapter}), we compute
\begin{align*}
\Pr& [Y_{g+1}\ge 0.6N \mid Y_g\in [0.4N,0.6N] ]\\
& \le \Pr\big[Y_{g+1}\ge E [Y_{g+1} \mid Y_g\in [0.4N, 0.6N] ]+\tfrac{12}{125}FCN \mid Y_g\in [0.4N,0.6N]\big]\\
& \le \exp(-cN)
\end{align*}
and
\begin{align*}
\Pr& [Y_{g+1}\le 0.4N \mid Y_g\in [0.4N,0.6N] ]\\
& \le \Pr\big[Y_{g+1}\le E [Y_{g+1} \mid Y_g\in [0.4N, 0.6N] ]-\tfrac{12}{125}FCN \mid Y_g\in [0.4N,0.6N]\big]\\
& \le \exp(-cN),
\end{align*}
where $c=\frac{288}{15625}(FC)^2$. This shows
\begin{equation}
\Pr [Y_{g+1}\in [0.4N,0.6N] \mid Y_g\in [0.4N,0.6N] ] \ge 1-2\exp(-cN).
\label{eq:ind}
\end{equation}

Since $E[Y_0]=0.5N$, by a simple Chernoff inequality (Theorem 1.11 in \cite{Doerr11bookchapter}), we have
\begin{align*}
&\Pr[Y_0\ge 0.6N]\le \exp(-c_0N),\\
&\Pr[Y_0\le 0.4N]\le \exp(-c_0N)
\end{align*}
for $c_0=\frac{1}{50}$ and consequently
\begin{equation}
\Pr [Y_0\in [0.4N, 0.6N] ]\ge 1-2\exp(-c_0N).
\label{eq:base}
\end{equation}

With (\ref{eq:ind}) and (\ref{eq:base}), a simple induction gives
\begin{align*}
\Pr[\forall{}{}& g\in\{0,\dots,T\}: Y_g\in[0.4N, 0.6N]]\\
={}&\Pr \big[Y_T\in[0.4N, 0.6N] \mid \forall g\in\{0,\dots,T-1\}: Y_g\in[0.4N, 0.6N] \big]\\
{}&\cdot \Pr [\forall g\in\{0,\dots,T-1\}: Y_g\in[0.4N, 0.6N] ]\\
\ge{}&(1-2\exp(-cN))(1-2\exp(-c_0N))(1-2\exp(-cN))^{T-1}\\
={}&(1-2\exp(-c_0N))(1-2\exp(-cN))^{T}\ge(1-2\exp(-cN))^{T+1},
\end{align*}
where we use the fact that $c_0>c$.
Since $N\ge \frac{15625\ln 2}{288(FC)^2}$, we have $-2\exp(-cN) \ge -1$. Using Bernoulli's inequality, we obtain
\begin{align*}
\Pr [\forall &g\in\{0,\dots,T\}: Y_g\in [0.4N,0.6N] ]
\ge 1-2(T+1)\exp(-cN).
\end{align*}
\end{proof}

\subsection{The Behavior of an Arbitrary Neutral Bit}
\label{sec:Anyneutral}
In the previous subsection, we proved rigorously that BDE is very stable when optimizing the Needle function. We are not able to show a similar stability result for neutral bits of an arbitrary function. The reasons are the stochastic dependencies both from the mutation operator and the selection. Note that for the Needle function, we did not have these difficulties because the trial population always survives (until the optimum is found). 

To also have a result for the stability with respect to arbitrary neutral bits, we now resort to our mean-field approach, that is, we argue with experimental data for the fact that neutral bits behave similarly in iBDE and BDE and then prove that frequencies of neutral bits in a run of iBDE stay in the middle region for an exponential (in $N$) time. 
%

\subsubsection{Experimental Comparison of the Behavior of Neutral Bits in BDE and iBDE}

To experimentally argue for the fact that BDE and iBDE have a similar behavior in neutral bits, we regard the classic LeadingOnes benchmark function $f : \{0,1\}^n \to \Z$ defined first in~\cite{Rudolph97} by 
\begin{equation}
\begin{split}
f(X)=
\begin{cases}
0,&X=(0,...,0)\\
\max\{j \in \{1, \dots, D\} \mid \prod_{i=1}^{j} X_i=1\},&\mathrm{otherwise}
\end{cases}
\end{split}
\label{eq:LO1}
\end{equation}
for all $X = (X_1, \dots, X_D) \in \{0,1\}^D$.
The last bit position of the LeadingOnes function is a neutral bit until the optimum is found. However, selection plays an important role in the optimization of LeadingOnes, so it appears that this example is of a nature very different from the Needle function.

In our experiments we use the setting $D=1000, N=1000, F=0.2$, and $C=0.3$ (for both BDE and iBDE). For each algorithm, 100 independent runs are conducted. Among these 100 independent runs, the minimum, maximum, and $10\%,50\%,90\%$ quantiles of the frequencies of ones in the last bit (which is neutral longest) are plotted in Figure~\ref{fig:NeutralQuantiles}. Also depicted in this figure is the minimum frequency of ones among all bit positions and all runs. For one randomly picked run, Figure~\ref{fig:BDEandiBDEonNeutral} shows the frequency of ones in the last bit over time. 

These two visualizations indicate that the frequency of ones in bit positions that are still neutral oscillates in a small corridor around $0.5$ without that significant differences between the two algorithms are visible. Consequently, it appears reasonable that a behavior proven for neutral bits in a run of iBDE via mathematical means (such as Theorem~\ref{thm:iBDENeutral}) is valid for BDE as well. 

To have all experimental results on the LeadingOnes function in one subsection, we now present some more results which will be used in Section~\ref{sec:dominant}. Table~\ref{tbl:LOruntime} gives the minimum, average and maximum runtimes among the 100 independent runs. Figure~\ref{fig:BDEandiBDEfitonLOmBW} plots the average fitness over time. 

\begin{figure}[H]
  \centering
  \begin{minipage}[t]{1\textwidth}
  \centering
  \includegraphics[width=3.8in]{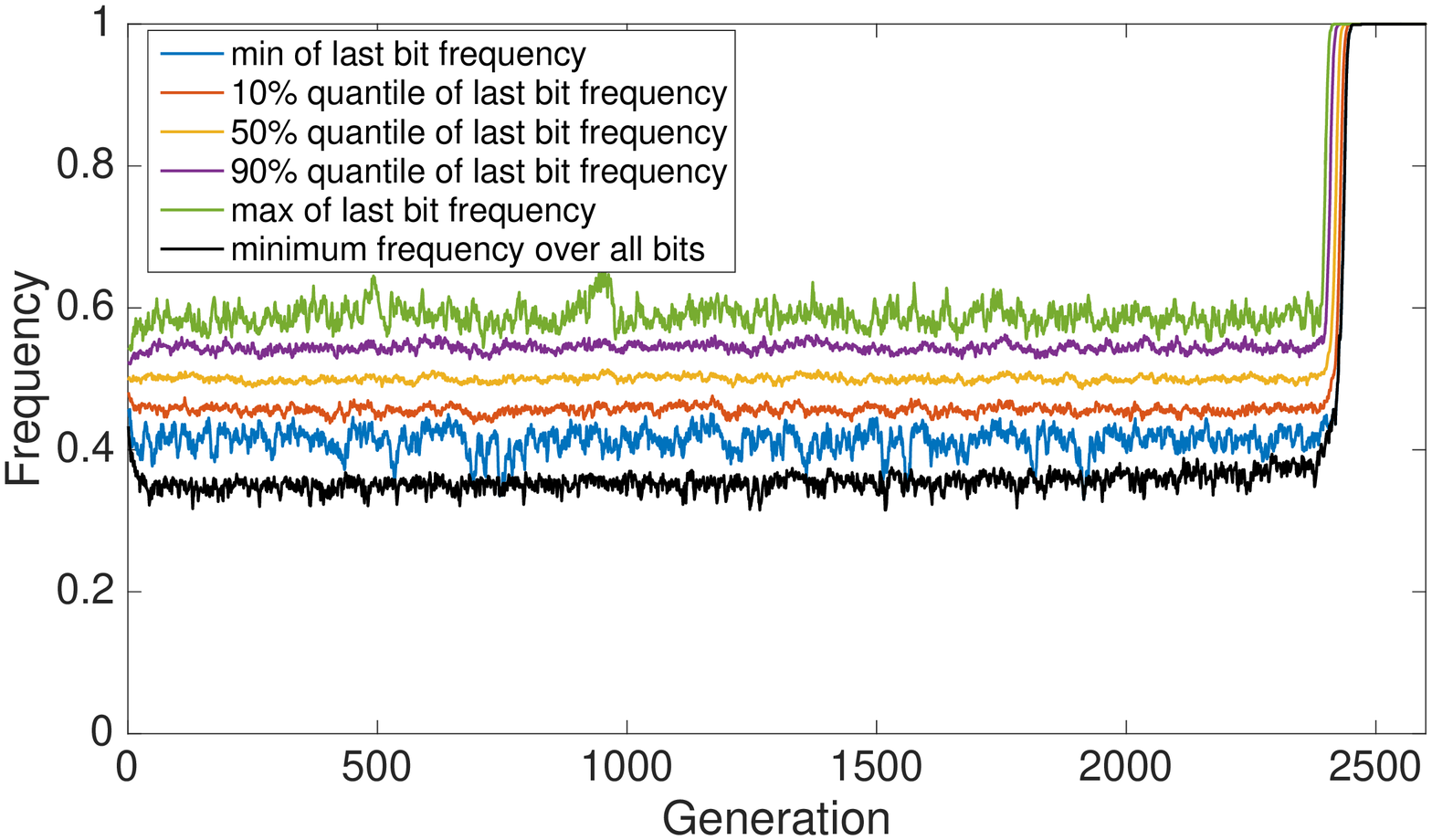}
  \end{minipage}
\vspace{1.pt}
  \begin{minipage}[t]{1\textwidth}
  \centering
  \includegraphics[width=3.8in]{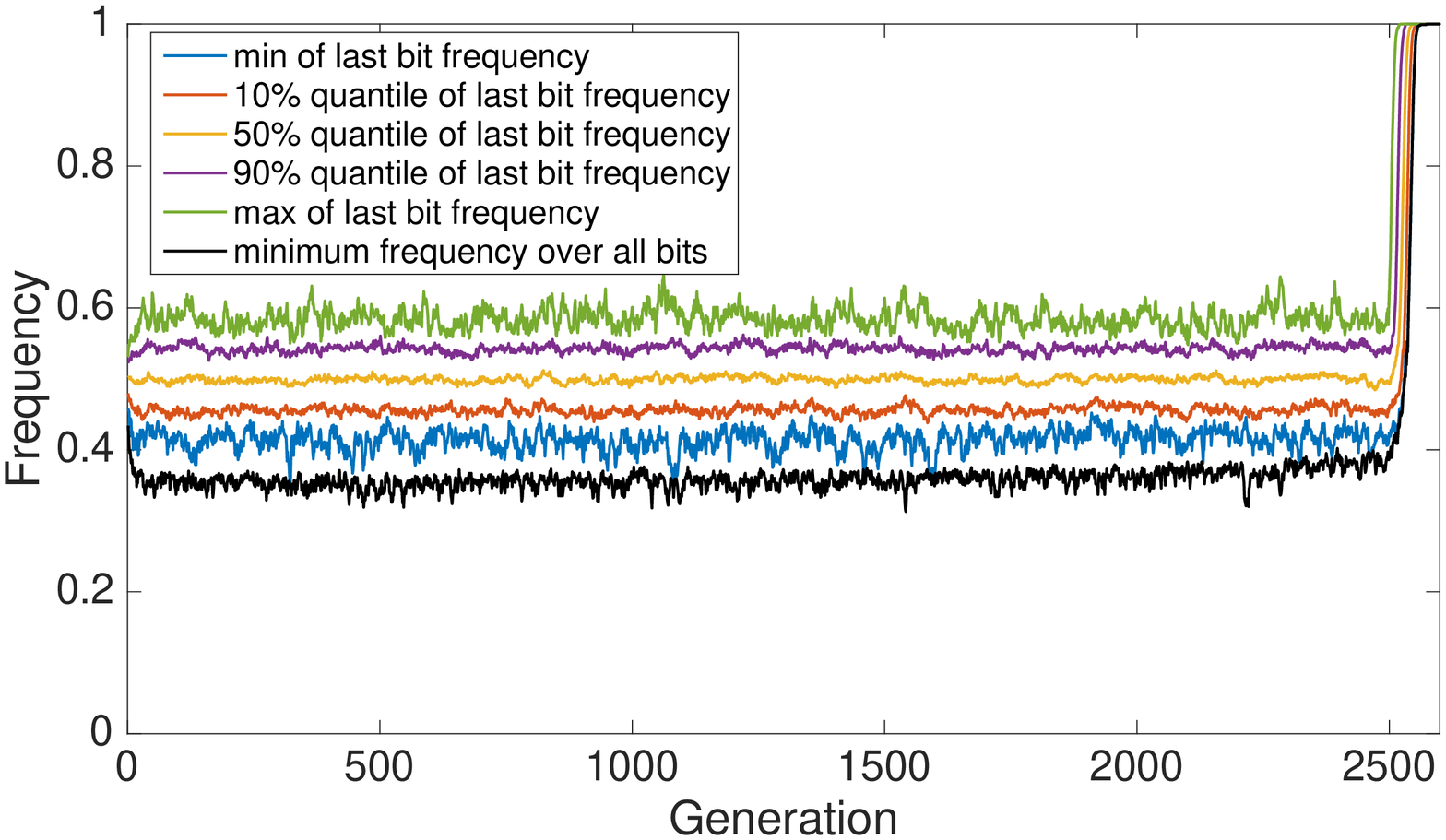}
  \end{minipage}
  \caption{The minimum, maximum, and $10\%,50\%,90\%$ quantiles of the frequency of ones in the last bit position among 100 runs of BDE (top) and iBDE (bottom) optimizing the LeadingOnes function ($D=1000, N=1000, F=0.2$, $C=0.3$). Also depicted are the minimum frequency of ones in all bit positions and all runs.}
    \label{fig:NeutralQuantiles}
\end{figure}

\begin{figure}[!ht]
\centering
\includegraphics[width=3.8in]{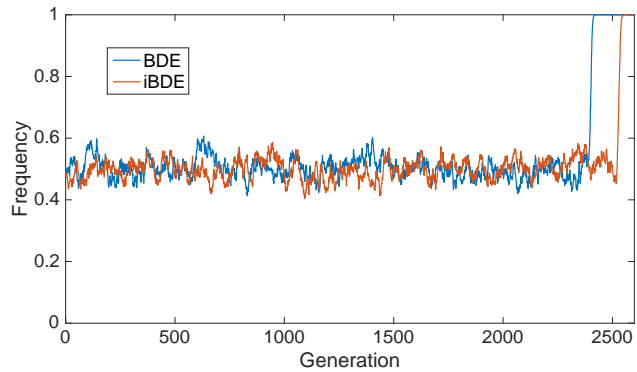}
\caption{The frequency of ones in the last bit position for exemplary runs of BDE and iBDE on the LeadingOnes function  ($D=1000, N=1000, F=0.2$, $C=0.3$).}
\label{fig:BDEandiBDEonNeutral}
\end{figure}

\begin{table}[H]
\centering
  \caption{The runtimes of BDE and iBDE optimizing the LeadingOnes function in 100 independent runs  ($D=1000, N=1000, F=0.2$, $C=0.3$).}
  \label{tbl:LOruntime}
    \begin{tabular}{cccc}
    \hline
    & minimum & average & maximum\\
   \hline
    BDE & 2359 & 2387 & 2404\\
    iBDE & 2467 & 2497 & 2515 \\
  \hline
 \end{tabular}
\end{table}%

\begin{figure}[!ht]
\centering
\includegraphics[width=3.8in]{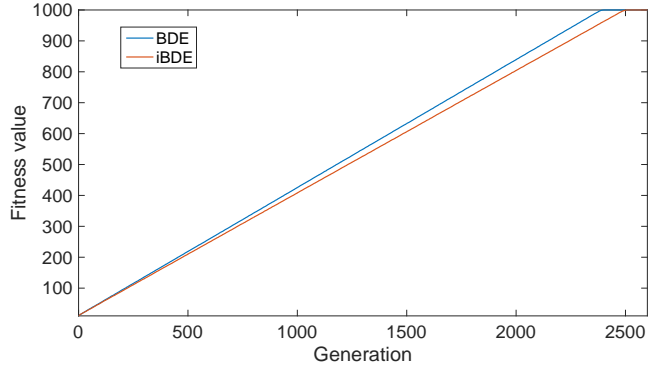}
\caption{Average fitness over time for BDE and iBDE optimizing the LeadingOnes function ($D=1000, N=1000, F=0.2$, $C=0.3$, 100 independent runs).}
\label{fig:BDEandiBDEfitonLOmBW}
\end{figure}



\subsubsection{Theoretical Analysis of the Behavior of a Neutral Bit for iBDE}

We now mathematically analyze the behavior of a neutral bit for iBDE. Naturally, this part has some similarity with the analysis of BDE on the Needle function. A crucial additional difficulty to overcome in the proof of the main result in Theorem~\ref{thm:iBDENeutral} is that, unlike for the Needle function, we cannot anymore assume that always the offspring wins the parent-offspring selection. We solve this problem by first fixing an inheritance pattern (which describes whether the parent or the offspring bit-value survives, note that this depends only on the non-neutral bits) and then analyzing the random process of the neutral bit conditional on this inheritance pattern. 

Before the main analysis, we extract some computations as lemmas to make the  core arguments more concise. Lemma~\ref{lem:mutprob} computes the probability for a particular mutant bit have the value 1. This result is true regardless of the neutrality of the bit.
\begin{lemma}
Consider an iteration of iBDE or BDE with population size $N\ge 4$. Let $i \in \{1,\dots, N\}$ and $j \in \{1,\dots, D\}$. Let $Y_g^-$ denote the number of ones in the $j$-th bit position among all individuals of generation $g$ except the $i$-th individual. Then the probability for generating value $1$ in the $j$-th position of the $i$-th mutant is 
\begin{align*}
\frac{4F(Y_g^-)^3-6F(N-1)(Y_g^-)^2+((2F+1)N^2-(5+4F)N+2F+6)Y_g^-}{(N-1)(N-2)(N-3)}.
\end{align*}
\label{lem:mutprob}
\end{lemma}
\begin{proof}
By definition of the mutation operator, we have $V_{i,j}^g=1$ if and only if one of the following cases holds.
\begin{itemize}
\item $X_{r_1,j}^g=1, X_{r_2,j}^g=X_{r_3,j}^g$.
\item $X_{r_1,j}^g=1, X_{r_2,j}^g \ne X_{r_3,j}^g, \mrand_j \ge F$.
\item $X_{r_1,j}^g=0, X_{r_2,j}^g \ne X_{r_3,j}^g, \mrand_j < F$.
\end{itemize}
Hence, recalling that $Y_g^-$ denotes the number of ones in the $j$-th bit position among the population $P^g$ except $X_i^g$, we obtain
\begin{align*}
\Pr[V_{i,j}^g{}&=1 \mid P^g]\\
={}&\frac{Y_g^-((Y_g^--1)(Y_g^--2)+(N-1-Y_g^-)(N-2-Y_g^-))}{(N-1)(N-2)(N-3)}\\
{}&+\frac{Y_g^-(Y_g^--1)(N-1-Y_g^-)2(1-F)}{(N-1)(N-2)(N-3)}+\frac{(N-1-Y_g^-)Y_g^-(N-2-Y_g^-)2F}{(N-1)(N-2)(N-3)}\\
={}&\frac{4F(Y_g^-)^3-6F(N-1)(Y_g^-)^2+((2F+1)N^2-(5+4F)N+2F+6)Y_g^-}{(N-1)(N-2)(N-3)}.
\end{align*}
\end{proof}

To gain a better understanding of the probability computed above, let us define (for implicitly given $F$) the function $R_N: [0,N-1] \rightarrow [0,\infty)$ by
\begin{align*}
R_N(y)=\frac{4Fy^3-6F(N-1)y^2+((2F+1)N^2-(5+4F)N+2F+6)y}{(N-1)(N-2)(N-3)},
\end{align*}
so that the probability we just computed is $R_N(Y_g^-)$. Going from absolute numbers of ones to relative numbers, we also define $r(x)=R_N(x(N-1))$ for all $x \in [0,1]$. Figure~\ref{fig:rx} visualizes this function for two sets of parameter values.
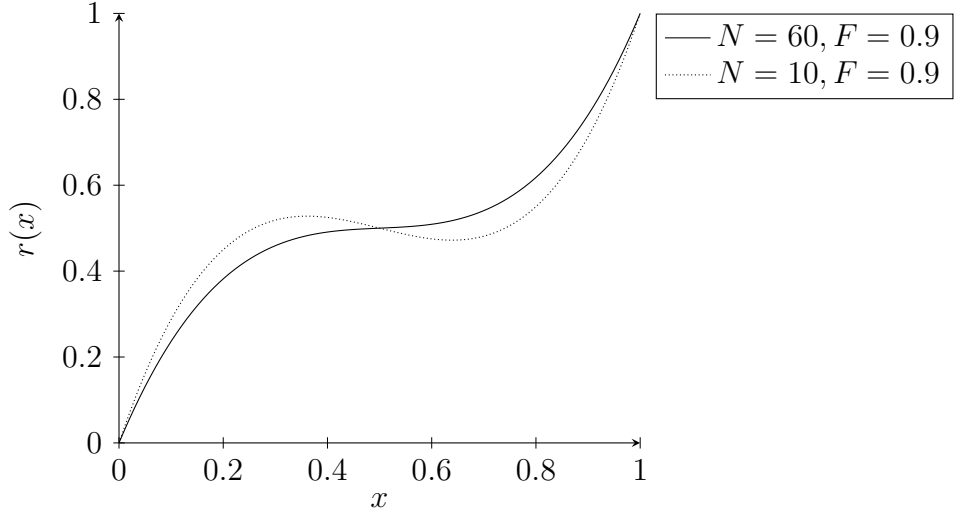
\begin{figure}
\centering
\begin{tikzpicture}
\begin{axis}[
    axis lines = left,
    xlabel = $x$,
    ylabel = {$r(x)$},
    ymin = 0,
    legend style={
      cells={anchor=east},
      legend pos=outer north east,},
    ]

\addplot [
    domain=0:1, 
    samples=100, 
    color=black,
]
{(4*0.9*59^3*x^3-6*0.9*59*59^2*x^2+((2*0.9+1)*60^2-(5+4*0.9)*60+2*0.9+6)*59*x)/(59*58*57)};
\addlegendentry[right]{$N=60,F=0.9$}
 
\addplot [densely dotted,
    domain=0:1, 
    samples=100, 
    color=black,
    ]
{(4*0.9*9^3*x^3-6*0.9*9*9^2*x^2+((2*0.9+1)*10^2-(5+4*0.9)*10+2*0.9+6)*9*x)/(9*8*7)};
\addlegendentry[right]{$N=10,F=0.9$} 
\end{axis}
\end{tikzpicture}
\caption{A visualization of $r(x) := R_N(x(N-1))$, that is, the probability for generating the value $1$ in the $j$-th position of the mutant. For $N$ not too small, this function is monotonically increasing.}\label{fig:rx}
\end{figure}

In the following Lemma~\ref{lem:mutprobin}, we collect a few useful properties of $R_N$, in particular, that $R_N$ when $N$ is at least some constant (depending on $F)$.
\begin{lemma}
Let $F \in (0,1)$ and $N \in [\tfrac{5-2F}{1-F},+\infty) \cap \N$. Then the following statements hold.
\begin{itemize}
\item $R_N(y)$ is monotonically increasing. 
\item $R_N(\tfrac{12}{25}(N-1)) > \tfrac{12}{25}$. When $N\ge \tfrac{625}{24F}$, $R_N(\tfrac{13}{25}N) <  \tfrac{13}{25}$.
\item When $N>\tfrac{3125-1224F}{625-612F}$, we have $R_N(\tfrac{8}{25}(N-1)) < \tfrac{12}{25}$ and $R_N(\tfrac{17}{25}N) > \frac{13}{25}$.
\end{itemize}
\label{lem:mutprobin}
\end{lemma}
\begin{proof}
Let $g(y)=4Fy^3-6F(N-1)y^2+((2F+1)N^2-(5+4F)N+2F+6)y$. It is easy to see that $g(y)$ and $R_N(y)$ have the same monotonicity. We have 
\begin{align*}
g'(y)={}&12Fy^2-12F(N-1)y+(2F+1)N^2-(5+4F)N+2F+6,\\
g''(y)={}&24Fy-12F(N-1).
\end{align*}
Since $F>0$, we have $g''(y) \le 0$ for $y \le \tfrac{N-1}{2}$ and $g''(y) \ge 0$ for $y \ge \tfrac{N-1}{2}$. Hence $g'(y)$ has a unique minimum at $\tfrac{N-1}{2}$. Therefore, 
\begin{align*}
g'(y)\ge{}&g'(\tfrac{N-1}{2})\\
={}&12F\tfrac{(N-1)^2}{4}-12F(N-1)\tfrac{N-1}{2}+(2F+1)N^2-(5+4F)N+2F+6\\
={}&(1-F)N^2+(2F-5)N+6-F.
\end{align*}
Since $F < 1$ and $N \ge \tfrac{5-2F}{1-F}$, we have
\begin{align*}
(1-F)N^2+(2F-5)N+6-F 
\ge{}& (5-2F)N+(2F-5)N+6-F\\
={}&6-F>0.
\end{align*}
Hence $g'(y)$ is positive and thus $g(y)$ and $R_N(y)$ are monotonically increasing.

Since 
\begin{align*}
R_N(\tfrac{12}{25}{}&(N-1))\\
={}&\tfrac{1}{(N-1)(N-2)(N-3)}(4F(\tfrac{12}{25}(N-1))^3-6F(N-1)(\tfrac{12}{25}(N-1))^2\\
{}&+((2F+1)N^2-(5+4F)N+2F+6)\tfrac{12}{25}(N-1))\\
={}&\tfrac{12}{25}+\tfrac{312F(N-1)^2}{25^3(N-2)(N-3)},
\end{align*}
we have $R_N(\tfrac{12}{25}(N-1)) > \tfrac{12}{25}$.

We compute
\begin{align*}
R_N(\tfrac{13}{25}{}&N)\\
={}&\tfrac{1}{(N-1)(N-2)(N-3)}(4F(\tfrac{13}{25}N)^3-6F(N-1)(\tfrac{13}{25}N)^2\\
{}&+((2F+1)N^2-(5+4F)N+2F+6)\tfrac{13}{25}N)\\
={}&\tfrac{13}{25}+\frac{-\tfrac{312}{15625}FN^3+(\tfrac{13}{25}-\tfrac{286}{625}F)N^2+(-\tfrac{13}{5}+\tfrac{26}{25}F)N+\tfrac{78}{25}}{(N-1)(N-2)(N-3)}\\
\le{}&\tfrac{13}{25}+\frac{-\tfrac{312}{15625}FN^3+\tfrac{13}{25}N^2-\tfrac{39}{25}N+\tfrac{78}{25}}{(N-1)(N-2)(N-3)},
\end{align*}
since $N\ge \tfrac{625}{24F}>2$, we have $-\tfrac{312}{15625}FN+\tfrac{13}{25} \le 0$ and $-\tfrac{39}{25}N+\tfrac{78}{25} < 0$, thus $R_N(\tfrac{13}{25}N) <  \tfrac{13}{25}$. 

We compute
\begin{align*}
R_N(\tfrac{8}{25}{}&(N-1))\\
={}&\tfrac{1}{(N-1)(N-2)(N-3)}(4F(\tfrac{8}{25}(N-1))^3-6F(N-1)(\tfrac{8}{25}(N-1))^2\\
{}&+((2F+1)N^2-(5+4F)N+2F+6)\tfrac{8}{25}(N-1))\\
={}&\tfrac{12}{25}+\tfrac{(2448F-2500)N^2+(12500-4896F)N-15000+2448F}{25^3(N-2)(N-3)}\\
< & \tfrac{12}{25}+\tfrac{(2448F-2500)N^2+(12500-4896F)N}{25^3(N-2)(N-3)},
\end{align*}
and
\begin{align*}
R_N(\tfrac{17}{25}{}&(N-1))\\
={}&\tfrac{1}{(N-1)(N-2)(N-3)}(4F(\tfrac{17}{25}(N-1))^3-6F(N-1)(\tfrac{17}{25}(N-1))^2\\
{}&+((2F+1)N^2-(5+4F)N+2F+6)\tfrac{17}{25}(N-1))\\
={}&\tfrac{13}{25}+\tfrac{(2500-2448F)N^2+(4896F-12500)N+15000-2448F}{25^3(N-2)(N-3)}\\
> & \tfrac{13}{25}+\tfrac{(2500-2448F)N^2+(4896F-12500)N}{25^3(N-2)(N-3)}.
\end{align*}
Since $F<1$ and $N > \tfrac{3125-1224F}{625-612F}=\tfrac{12500-4896F}{2500-2448F}$, we have $(2448F-2500)N^2+(12500-4896F)N \le 0$, showing that $R_N(\tfrac{8}{25}(N-1)) < \tfrac{12}{25}$ and $R_N(\tfrac{17}{25}(N-1)) > \frac{13}{25}$. Since $R_N(y)$ monotonically increases, we have $R_N(\tfrac{17}{25}N) \ge R_N(\tfrac{17}{25}(N-1)) > \frac{13}{25}$.
\end{proof}

Now we show the stability of iBDE. 

\begin{theorem}
Consider using iBDE with population size $N \ge\max\{\tfrac{5-2F}{1-F},\tfrac{3125-1224F}{625-612F},\tfrac{625}{24F}\}$ to optimize a $D$-dimensional function $f$ with some neutral bit. Let $Y_g$ denote the number of ones in the neutral bit position among all individuals of generation $g$. There is a constant $c'>0$, depending on $F$ only, such that
\begin{equation*}
\Pr [\forall g\in\{0,\dots,T\}: Y_g\in [0.4N,0.6N]] \ge 1-2(T+1)\exp(-c'N)
\end{equation*}
for all $T\in \N$.
\label{thm:iBDENeutral}
\end{theorem}

\begin{proof}
Without loss of generality, let the first bit be the neutral bit. Since the bit is neutral, we can first run iBDE without this bit and then analyze the process of this bit conditional on the outcome of this run. More detailedly, we now fix one run of iBDE on $f$ with all random variables sampled except the initial values $X_{i,1}^0, i=1,\dots,N$ of the first bit and the indices $r_1(i,g),r_2(i,g),r_3(i,g)$ and the random variables $\mrand_1(i,g)$ which are used for generating the first bit of $V_i^g$ in generation $g$. Since all random variables used for generating other positions of $U_i^g$ are sampled and the first bit is neutral, we know whether $U_i^g$ or $X_i^g$ will enter the next generation. Since $\crand_1(i,g)$ is already sampled as well, whether the neutral bit $U_{i,1}^g$ stems from $V_{i,1}^g$ or $X_{i,1}^g$ is also determined. Therefore, conditioning on all these random variables, we know whether $X_{i,1}^{g+1}$ is inherited from $X_{i,1}^g$ (either because $X_i^{g+1}:=X_i^g$ in the selection step or because $X_i^{g+1}:=U_i^g$, but $U_{i,1}^g$ inherited $X_{i,1}^g$ in the crossover step), or whether $X_{i,1}^{g+1}$ stems from the mutant $V_{i,1}^g$ (in this case, $X_i^{g+1}=U_i^g$ and $\crand_1(i,g) < C$). Consequently, the already sampled random variables completely determine the random process in the neutral bit.

We therefore now regard the following random process. We fix an arbitrary \emph{inheritance pattern} consisting of boolean variable $I_i^g, i \in \{1,\dots,N\}, g \in \N_0$. We then sample $\xi_1^0, \dots, \xi_N^0 \in \{0,1\}$ independently and uniformly at random. If $\xi_1^g, \dots, \xi_N^g$ are determined for some $g \in \N_0$, then we define $\xi_1^{g+1}, \dots, \xi_N^{g+1}$ as follows. Let $i \in \{1,\dots,N\}$ and $Y_g^{i,-}:=\sum_{j=1,j\ne i}^D \xi_j^g$. If $I_i^g$ is true, then $\xi_i^{g+1}=\xi_i^g$. Otherwise, we choose $\xi_i^{g+1} \in \{0,1\}$ randomly (independently for all $i\in\{1,\dots,D\}$) such that 
\begin{align*}
\Pr[\xi_i^{g+1}=1]=R_N(Y_g^{i,-}):= p_i^g,
\end{align*}
where $R_N(Y_g^{i,-})$ is defined in Lemma~\ref{lem:mutprobin}.

From the above, it is clear that this process exactly describes the values of the neutral bit discussed in a run of iBDE. 

Consider the process $\eta=(\eta_i^g)$ which is identical to the $\xi$-process except that $\eta_i^{g+1}$ is sampled (the case when $I_i^g$ is false) independently with 
$\Pr[\eta_i^{g+1}=1]=\tfrac{13}{25}$, and the process $\phi=\phi_i^g$ which is identical to the $\xi$-process except that $\phi_i^{g+1}$ is sampled independently with $\Pr[\phi_i^{g+1}=1]=\tfrac{12}{25}$. 
Let $b_0,b_1\in[0,1]$ such that $R_N(b_0(N-1))=\tfrac{12}{25}$ and $R_N(b_1N)=\tfrac{13}{25}$. Since $R_N(y)$ is strictly monotonically increasing from Lemma~\ref{lem:mutprobin}, we know $b_0$ and $b_1$ are well defined. Since $R_N(\tfrac{12}{25}(N-1)) > \tfrac{12}{25}$ and $R_N(\tfrac{8}{25}(N-1)) < \tfrac{12}{25}$, $R_N(\tfrac{13}{25}N) < \tfrac{13}{25}$ and $R_N(\tfrac{17}{25}N) > \frac{13}{25}$, and $R_N(y)$ is monotonically increasing, we have
\begin{equation}
\tfrac{8}{25} < b_0 < \tfrac{12}{25}\ \textnormal{and}\ \tfrac{13}{25} < b_1 < \tfrac{17}{25}.
\label{eq:appb01}
\end{equation}
Hence, $b_0<\tfrac{25b_0+12}{50} <\tfrac{12}{25}\ \textnormal{and}\ \tfrac{13}{25}<\tfrac{25b_1+13}{50} < b_1$.
Thus, we have $(\tfrac{12}{25}N,\tfrac{13}{25}N) \subset (\tfrac{25b_0+12}{50}N,\tfrac{25b_1+13}{50}N)$ and
\begin{equation}
R_N([\tfrac{25b_0+12}{50}(N-1),\tfrac{25b_1+13}{50}N]) \in (\tfrac{12}{25},\tfrac{13}{25}).
\label{eq:rnvalue}
\end{equation}

Let $Y_g:=\sum_{i=1}^N \xi_i^g$. 
We now show that if $Y_g \in (\tfrac{25b_0+12}{50}N,\tfrac{25b_1+13}{50}N)$, then $Y_{g+1}^{\eta}:=\sum_{i=1}^N \eta_i^{g+1}$ stochastically dominates $Y_{g+1}$ and $Y_{g+1}$ stochastically dominates $Y_{g+1}^{\phi}:=\sum_{i=1}^N \phi_i^{g+1}$. Assume that $Y_g \in (\tfrac{25b_0+12}{50}N,\tfrac{25b_1+13}{50}N)$. Then we have $Y_g^{i,-} \in [\tfrac{25b_0+12}{50}(N-1),\tfrac{25b_1+13}{50}N]$ for $i\in\{1,\dots,N\}$. For $p_i^g=R_N(Y_g^{i,-})$, from (\ref{eq:rnvalue}), we know that $p_i^g \in (\tfrac{12}{25},\tfrac{13}{25})$, that is, $\Pr[\xi_i^{g+1} = 1 \mid I_i^g \ne \TRUE] \in [\tfrac{12}{25},\tfrac{13}{25}]$. Since $\Pr[\eta_i^{g+1}=1\mid I_i^g \ne \TRUE]=\tfrac{13}{25}$ and $\Pr[\phi_i^{g+1}=1\mid I_i^g \ne \TRUE]=\tfrac{12}{25}$, due to the definition of $\eta_i^{g+1}$ and $\phi_i^{g+1}$, we know $\eta_i^{g+1}$ dominates $\xi_i^{g+1}$ and $\xi_i^{g+1}$ dominates $\phi_i^{g+1}$. Hence $Y_{g+1}^{\eta}$ dominates $Y_{g+1}$ and $Y_{g+1}$ dominates $Y_{g+1}^{\phi}$.

Finally, we argue that in the $\eta$ process, we have $Y_g^{\eta} < \tfrac{25b_1+13}{50}N$ with probability $1-\exp(c_1N)$, where the constant $c_1$ will be specified in the following discussion. Fix any generation $g$. The inheritance pattern determines in which iteration $\eta_i^g$ was sampled (including the case that it was an initial sample). In either case, we have $\Pr[\eta_i^g = 1] \le \tfrac{13}{25}$ regardless of the outcomes of $\eta_{i'}^g, i' \neq i$. Consequently, we can apply the multiplicative Chernoff bound and obtain that
\begin{align*}
\Pr[Y_g^{\eta} \ge \tfrac{25b_1+13}{50}N] \le \exp(-c_1N),
\end{align*}
where $c_1=\tfrac{1}{12}(b_1-\tfrac{13}{25})^2$. Due to the dominance, we have
\begin{equation}
\begin{split}
\Pr[Y_g <{}& \tfrac{25b_1+13}{50}N \mid Y_{g-1} < \tfrac{25b_1+13}{50}N] \\
\ge{}& \Pr[Y_g^{\eta} < \tfrac{25b_1+13}{50}N] \ge 1-\exp(-c_1N).
\end{split}
\label{eq:YgUp}
\end{equation}
Since $E[Y_0]=0.5N$, by a simple Chernoff inequality (Theorem 1.11 in \cite{Doerr11bookchapter}), we have
\begin{equation}
\Pr[Y_0\ge \tfrac{25b_1+13}{50}N]\le \exp(-c_0N),
\label{eq:Y06}
\end{equation}
where $c_0=\tfrac{(25b_1-12)^2}{1250}$. 

With (\ref{eq:YgUp}) and (\ref{eq:Y06}), a simple induction gives
\begin{align*}
\Pr[\forall{}{}& g\in\{0,\dots, T\}: Y_g < \tfrac{25b_1+13}{50} N]\\
={}&\Pr [Y_T< \tfrac{25b_1+13}{50} N \mid \forall g\in\{0,\dots,T-1\}: Y_g< \tfrac{25b_1+13}{50} N \big]\\
{}&\cdot \Pr [\forall g\in\{0,\dots,T-1\}: Y_g< \tfrac{25b_1+13}{50} N]\\
\ge{}&(1-\exp(-c_0N))(1-\exp(-c_1N))^{T} \ge (1-\exp(-c_1N))^{T+1},
\end{align*}
where we use the fact that $c_0>c_1$. Using Bernoulli's inequality, we obtain
\begin{equation}
\begin{split}
\Pr[\forall g{}&\in\{0,\dots, T\}: Y_g < \tfrac{25b_1+13}{50} N] \ge (1-\exp(-c_1N))^{T+1}\\
 \ge{}& 1-(T+1)\exp(-c_1N).
 \end{split}
\label{eq:allYgUp}
\end{equation}

Similarly, for $\phi$ process, we have
\begin{equation*}
\Pr[Y_0\ge \tfrac{25b_0+12}{50}N]\le \exp(-c_2N),
\end{equation*}
\begin{align*}
\Pr[Y_g > \tfrac{25b_0+13}{50}N \mid Y_{g-1} > \tfrac{25b_0+12}{50}N] \ge \Pr[Y_g^{\phi} < \tfrac{25b_0+12}{50}N] \ge 1-\exp(-c_3N)
\end{align*}
and
\begin{align}
\Pr[\forall{}{}& g\in\{0,\dots, T\}: Y_g > \tfrac{25b_0+12}{50} N] \ge 1-(T+1)\exp(-c_3N),
\label{eq:allYgLow}
\end{align}
where $c_2=\tfrac{(13-25b_0)^2}{1250}$ and $c_3=\tfrac{1}{8}(\tfrac{12}{25}-b_0)^2$.

With (\ref{eq:allYgUp}) and (\ref{eq:allYgLow}), we have 
\begin{align*}
\Pr[\forall g \in \{0,\dots,T\}: Y_g \in (\tfrac{25b_0+12}{50}N,\tfrac{25b_1+13}{50}N)] \ge 1-2(T+1)\exp(-c'N),
\end{align*}
where $c'=\min\{c_1,c_3\}$. Since $b_0$ and $b_1$ can either be a constant or a constant only depends on $F$, we can say $c'$ only depends on $F$.

With (\ref{eq:appb01}), we have $\tfrac{2}{5} < \tfrac{25b_0+12}{50}<\tfrac{25b_1+13}{50} < \tfrac{3}{5}$ and thus
\begin{align*}
\Pr [\forall g \in\{0,\dots,T\}: Y_g\in [0.4N,0.6N]] \ge 1-2(T+1)\exp(-c'N).
\end{align*}
\end{proof}

The above theorem shows that for iBDE, with high probability, the frequency of a neutral bit will stay in $[0.4,0.6]$ for a number of generations an exponential in the population size.

The experimentally observed similarity between BDE and iBDE and this theoretical result about the stability of iBDE indicate that BDE is stable in arbitrary neutral bits.

\subsection{The Behavior of Neutral Bits in Classic EDAs}
\label{sec:neutralothers}

Different from the stable behavior of BDE and iBDE discussed above, many nature-inspired optimization heuristics are unstable, that is, the frequencies of neutral bits approach the boundary values relatively fast. We quantify this effect asymptotically precise for the two EDAs, UMDA and cGA, by showing that the expected time until the sampling frequency of a neutral bit is $0$ or $1$ is $\Theta({\mu})$ for UMDA and is $\Theta(K^2)$ for cGA.

\begin{theorem}\label{thm:stableEDA}
For UMDA without margins, the expected first time the frequency in the neutral bit is absorbed in 0 or 1 is $\Theta(\mu)$, and it is $\Theta(K^2)$ for cGA.
\end{theorem}

Since the $n$-Bernoulli-$\lambda$-EDA framework proposed in~\cite{FriedrichKK16} covers many well-known EDAs including UMDA and cGA, we use it to make precise these two EDAs. 

\begin{algorithm}[!ht]
\caption{$n$-Bernoulli-$\lambda$-EDA with a given update scheme $\varphi$ maximizing a function $f: \{0,1\}^D \rightarrow \R$}
{\small
 \begin{algorithmic}[1]
 \STATE{$p^0=(\tfrac{1}{2}, \tfrac{1}{2},\dots,\tfrac{1}{2})\in [0,1]^D$}
 \FOR {$t=1,2,\dots$}
 \FOR {$i=1,2,\dots,\lambda$}
 \STATEx {$\quad\quad\%\%\quad$\textsl{Sampling of individual $X_i^t=(X_{i,1}^t,\dots,X_{i,D}^t)$}}
 \FOR {$j=1,2,\dots,D$}
 \STATE $X_{i,j}^t\sim \Bernoulli(p_{i,j}^{t-1})$;
 \ENDFOR
 \ENDFOR
 \STATEx {$\quad\%\%\quad$\textsl{Update of the frequency vector}}
 \STATE $p^t\leftarrow\varphi(p^{t-1}, (X_i, f(X_i))_{i=1,\dots,\lambda})$;
 \ENDFOR
 \end{algorithmic}
 \label{alg:EDA}
}
\end{algorithm}

The $n$-Bernoulli-$\lambda$-EDA framework is shown in Alg.~\ref{alg:EDA}. By suitably specifying the update scheme $\varphi$, we derive UMDA and cGA. For UMDA with parameters $\mu$ and $\lambda$ and without margins (that is, without artificial boundaries like $\frac 1D$ and $1 - \frac 1D$ for the frequencies), the update scheme is
\begin{equation}
p_j^t=\varphi(p^{t-1}, (X_i, f(X_i))_{i=1,\dots,\lambda})_j=\frac{1}{\mu}\sum\limits_{i=1}^{\mu}\tilde{X_{i,j}^t},
\label{eq:pbilupdate}
\end{equation}
where $\tilde{X_1^t},...,\tilde{X_\mu^t}$ are the selected $\mu$ best individuals from the $\lambda$ offspring.

To obtain cGA with hypothetical population size $K$, we use $\lambda=2$ and the update scheme 
\begin{equation}
    \begin{split}
    p_j^t=\varphi(p^{t-1}, (X_i, f(X_i))_{i=1,\dots,\lambda})_j= \begin{cases}
    p_j^{t-1}+\tfrac{1}{K}, & \text{if $X_{(1),j}^t>X_{(2),j}^t$}\\
    p_j^{t-1}-\tfrac{1}{K}, & \text{if $X_{(1),j}^t<X_{(2),j}^t$}\\
    p_j^{t-1}, & \text{if $X_{(1),j}^t=X_{(2),j}^t$},\\
    \end{cases}
    \end{split}
    \label{eq:cgaupdate}
\end{equation}
where $\{X_{(1)}^t,X_{(2)}^t\} = \{X_1^t,X_2^t\}$ such that $f(X_{(1)}^t) \ge f(X_{(2)}^t)$. We shall always assume that $K$ is even, so that the initial frequency $\frac 12$ is also a multiple of $\frac 1K$. 

As discussed in \cite{FriedrichKK16}, UMDA and cGA are not stable. More precisely, this work shows that for cGA, the frequency of a neutral bit is arbitrary close to the borders $0$ or $1$ after $\omega(K^2)$ generations. From Corollary~9 in \cite{FriedrichKK16}, we can derive an upper bound of $O(K^2 \log K)$ for the boundary hitting time although this is not mentioned in~\cite{FriedrichKK16}. 

For UMDA, the situation is similar. After $\omega(\mu)$ iterations, the frequencies are arbitrary close to the boundaries and the expected hitting time can be shown to be $O(\mu \log \mu)$ via similar arguments as above.

Sudholt and Witt's work~\cite{SudholtW16} mentions that the boundary hitting time of cGA is $\Theta(K^2)$, but without a clear proof (in particular, because they do not discuss what happens once the frequency exceeds $5/6$). 
Although Krejca and Witt's recently work~\cite{KrejcaWitt17} focuses on the lower bound of the runtime of UMDA on OneMax, we can derive from it that the hitting time of the boundary $0$ is at least $\Omega(\mu)$. This follows from the drift of $\phi$ in Lemma~9 in~\cite{KrejcaWitt17} together with the additive drift theorem~\cite{HeY01}. 

While the results above give some indication on the degree of stability of UMDA and cGA, a sharp proven result is still missing. We overcome this shortage with a 
simultaneous analysis of UMDA and cGA which determines these hitting times as $\Theta(\mu)$ for UMDA and $\Theta(K^2)$ for cGA, see Theorem~\ref{thm:stableEDA}.

\subsubsection{Notation}

Without loss of generality, let the first bit of $f$ be neutral. Since the first bit is not relevant for the fitness, we can simply assume that $\tilde{X_{i,1}^t}=X_{i,1}^t,i=1,\dots,\mu$ in (\ref{eq:pbilupdate}), and $X_{(1),1}^t=X_{1,1}^t, X_{(2),1}^t=X_{2,1}^t$ in (\ref{eq:cgaupdate}).
Let $p_t=p_1^t$ be the frequency of the neutral bit after generation $t$. Then for UMDA, we have
\begin{equation*}
\begin{split}
p_t=
\begin{cases}
\frac{1}{2}, &t=0\\
\frac{1}{\mu}\sum\limits_{i=1}^{\mu}X_{i,1}^t, &t\geq 1,
\end{cases}
\end{split}
\end{equation*}
where the $X_{i,1}^t$ are independent $0,1$ random variables with $\Pr[X_{i,1}^t=1]=p_{t-1}$. 

For cGA, we have
\begin{equation*}
\begin{split}
p_t=
\begin{cases}
\frac{1}{2}, &t=0\\
 \begin{cases}
    p_{t-1}+\frac{1}{K}, & \text{if $X_{1,1}^t>X_{2,1}^t$}\\
    p_{t-1}-\frac{1}{K}, & \text{if $X_{1,1}^t<X_{2,1}^t$}\\
    p_{t-1}, & \text{if $X_{1,1}^t=X_{2,1}^t$}\\
    \end{cases}
, &t\geq 1,
\end{cases}
\end{split}
\end{equation*}
where $X_{1,1}^t$ and $X_{2,1}^t$ are independent $0,1$ random variables with $ \Pr[X_{1,1}^t=1]= \Pr[X_{2,1}^t=1]=p_{t-1}$. 

The random process $(p_t)$ is independent of $f,D$, and, in the case of UMDA, $\lambda$. We have
\begin{equation*}
E[p_t\mid p_{t-1}]=p_{t-1},
\end{equation*}
that is, both UMDA and cGA are balanced in the sense of~\cite{FriedrichKK16}. 

Finally, let  $T=\min\{t\mid p_t \in \{0,1\}\}$ be the hitting time of the absorbing state 0 or 1.

We are now ready to prove matching upper and lower bounds for the hitting time $T$. Naturally, the upper bounds  are more interesting since they show that UMDA and cGA are not very stable. We start nevertheless with the lower bounds as these are easier to prove and thus a good warm-up for the upper bound proofs.

\subsubsection{Lower Bounds}

We now prove the following lower bound on the hitting time of the absorbing states.

\begin{theorem}
Consider using an $n$-Bernoulli-$\lambda$-EDA to optimize some function $f$ with a neutral bit. Let $T$ denote the first time the frequency of the neutral bit is absorbed in state 0 or 1. For UMDA without margins, we have $E [T]=\Omega({\mu})$ regardless of $\lambda$. For cGA, we have $E[T]=\Omega(K^2)$. 
\label{thm:lower}
\end{theorem}

\begin{proof}
For UMDA, building on the notation introduced above, we consider the random process
\begin{equation*}
Z_{t\mu+a}=p_{t}(\mu-a)+\sum\limits_{i=1}^aX_{i,1}^{t+1},
\end{equation*}
where $t=0,1,\dots$, and $a=0,1,\dots,\mu-1$. For $a=0$, we obviously have $Z_{t\mu} / \mu = p_{t}$, that is, the $Z$-process contains the process $(p_t)$ we are interested in. 

Noting that $Z_{(t+1)\mu}$ can also be written as $Z_{t\mu + \mu} = p_{t}(\mu-\mu)+\sum_{i=1}^\mu X_{i,1}^{t+1}$, it is also not difficult to see that for all $k=0,1,\dots$, we have
\begin{equation}
\begin{split}
\Pr[Z_{k+1}={}&Z_k+1-p_t\mid Z_{1},\dots,Z_{k}]=p_t,\\
\Pr[Z_{k+1}={}&Z_k+0-p_t\mid Z_1,\dots,Z_k]=1-p_t.
\end{split}
\label{eq:nprob}
\end{equation}
Consequently, 
\begin{equation*}
E[Z_{k+1}\mid Z_1,\dots,Z_k]=Z_k
\end{equation*}
and the sequence $Z_0,Z_1,Z_2,\dots$ is a martingale. For $k=1,2,\dots$, let $R_k=Z_k-Z_{k-1}$ define the martingale difference sequence.
By (\ref{eq:nprob}),
\begin{equation*}
|R_k|\le \max\{(1-p_t), p_t\} \le 1.
\end{equation*}
By the Hoeffding-Azuma inequality for maxima and minima (Theorem 3.10 and (41) in \cite{McDiarmid98}, note that in~(41) the absolute value should be inside the maximum, that is, $\max_k |\sum_{i=1}^k Y_i|$ as can be seen from the proof), we have
\begin{equation}\label{eq:azumaMM}
\Pr \left[ \max\limits_{k=1,\dots,t\mu} \left|\sum\limits_{i=1}^k{R_i} \right|\ge M \right]\le 2\exp \left(-\tfrac{M^2}{2t\mu} \right).
\end{equation}
Recalling $Z_0 = \frac {\mu}2$ and $p_t = Z_{t\mu}/\mu$, we have 
\begin{equation}\label{eq:pZ}
\Pr \left[ \max\limits_{k=1,\dots,t} \left| p_k - \tfrac 12  \right| \ge M/\mu \right]
\le 
\Pr \left[ \max\limits_{k=1,\dots,t\mu} \left|\sum\limits_{i=1}^k{R_i} \right|\ge M \right].
\end{equation}
Combining \eqref{eq:azumaMM} and \eqref{eq:pZ} with $M=\frac{\mu}{4}$, we obtain
\[\Pr \left[ \max\limits_{k=1,\dots,t} \left| p_k - \tfrac 12  \right| \ge \tfrac 14 \right]
\le 2 \exp\left(-\frac{\mu}{32t}\right).\]
Consequently, with $T_0=\min \{t \mid |p_t-\tfrac{1}{2} | \ge \tfrac{1}{4} \}$, we have
\begin{equation*}
E [T ]\ge E [T_0 ] \ge  (1-2\exp (-\tfrac{\mu}{32t} ) )(t+1),
\end{equation*}
and taking, e.g., $t=\frac{\mu}{32}$, gives the desired result $E [T]=\Omega({\mu})$.

For cGA, we may simply regard the process $Z_k=p_k$. Since for all $k=0,1,\dots$, 
\begin{align*}
\Pr[Z_{k+1}={}&Z_k+\tfrac{1}{K}\mid Z_{1},\dots,Z_{k}]=p_k(1-p_k),\\
\Pr[Z_{k+1}={}&Z_k-\tfrac{1}{K}\mid Z_1,\dots,Z_k]=p_k(1-p_k),\\
\Pr[Z_{k+1}={}&Z_k\mid Z_1,\dots,Z_k]=1-2p_k(1-p_k),
\end{align*}
we have $E [Z_{k+1}\mid Z_1,\dots,Z_k ]=Z_k$. The martingale difference sequence $R_k:=Z_k-Z_{k-1}$ satisfies $|R_k|\le\tfrac{1}{K}$. By the Hoeffding-Azuma inequality, we have
\begin{equation*}
\Pr \left[ \max\limits_{k=1,\dots,t} \left|p_k - \tfrac 12\right | \ge M \right] 
= \Pr \left[ \max\limits_{k=1,\dots,t} \left|\sum\limits_{i=1}^k{R_i} \right | \ge M \right] 
\le 2 \exp \left(-\tfrac{M^2K^2}{2t} \right).
\end{equation*}
With $M = \frac 14$, $t=\frac{K^2}{32}$, and $T_0=\min \{t \mid |p_t-\tfrac{1}{2} | \ge \tfrac{1}{4}\}$, we have $E[T] \ge E[T_0] =\Omega(K^2)$.
\end{proof}

\subsubsection{Upper Bounds}

To prove of our upper bounds, we use the following two auxiliary lemmas.
  
\begin{lemma}
For all $z\ge 0$ and $z_0>0$, we have
\begin{equation*}
\sqrt z\leq \sqrt z_0+\tfrac{1}{2}z_0^{-\frac{1}{2}}(z-z_0)-\tfrac{1}{8}z_0^{-\frac{3}{2}}(z-z_0)^2+\tfrac{1}{16}z_0^{-\frac{5}{2}}(z-z_0)^3.
\end{equation*}
\label{lem:z}
\end{lemma}
\begin{proof}
For the convenience of proof, let $x=\sqrt z$ and $a=\sqrt {z_0}$. We consider function
\begin{equation*}
\begin{split}
g(x)={}&x-a-\tfrac{1}{2}a^{-1}(x^2-a^2)+\tfrac{1}{8}a^{-3}(x^2-a^2)^2-\tfrac{1}{16}a^{-5}(x^2-a^2)^3\\
={}&-\tfrac{1}{16}a^{-5}x^6+\tfrac{5}{16}a^{-3}x^4-\tfrac{15}{16}a^{-1}x^2+x-\tfrac{5}{16}a
\end{split}
\end{equation*}
and show that $g(x) \le 0$.
Since
\begin{equation*}
g'(x)=-\tfrac{3}{8}a^{-5}x^5+\tfrac{5}{4}a^{-3}x^3-\tfrac{15}{8}a^{-1}x+1
\end{equation*}
and
\begin{equation*}
\begin{split}
g''(x)={}&-\tfrac{15}{8}a^{-5}x^4+\tfrac{15}{4}a^{-3}x^2-\tfrac{15}{8}a^{-1}\\
={}&-\tfrac{15}{8}a^{-5}(x^4-2a^2x^2+a^4)=-\tfrac{15}{8}a^{-5}(x^2-a^2)^2\le 0,
\end{split}
\end{equation*}
we know that $g'(x)$ is monotonically decreasing. Since $g'(0)=1$ and $g'(a)=0$, we observe that $g(x)$ increases on $[0,a)$ and decreases on $[a,\infty)$. Therefore, $g(x)\le g(a)=0$.
\end{proof}

An easy calculation gives the following second-order and third-order central moments of the frequency of a neutral bit in UMDA and cGA.
\begin{lemma}
For UMDA, we have
\begin{equation*}
\begin{split}
\Var[p_t\mid p_{t-1}]={}&\tfrac{1}{\mu}p_{t-1}(1-p_{t-1}),\\
E[(p_t-E[p_t\mid p_{t-1}])^3\mid p_{t-1}]={}&\tfrac{1}{\mu^2}p_{t-1}(1-p_{t-1})(1-2p_{t-1}).
\end{split}
\end{equation*}
For cGA, we have
\begin{equation*}
\begin{split}
\Var[p_t\mid p_{t-1}]={}&\tfrac{2}{K^2}p_{t-1}(1-p_{t-1}),\\
E[(p_t-E[p_t\mid p_{t-1}])^3\mid p_{t-1}]={}&0.
\end{split}
\end{equation*}
\label{lem:moments}
\end{lemma}
%

We are now ready to prove the following upper bound for the hitting time of the absorbing states of the frequency of a neutral bit.
\begin{theorem}
Consider using an $n$-Bernoulli-$\lambda$-EDA to optimize some function $f$ with a neutral bit. Let $T$ denote the first time the frequency of the neutral bit is absorbed in state 0 or 1. For UMDA without margins, we have $E[T]=O({\mu})$ regardless of $\lambda$. For cGA, we have $E[T]=O(K^2)$.
\label{thm:upper}
\end{theorem}

\begin{proof}
Let $q_t=\min\{p_t,1-p_t\}$ and $Y_t=\sqrt{q_t}$. Then $T=\min\{t\mid q_t=0\}$. Due to the symmetry, we just discuss the case that $q_{t-1}=p_{t-1}$. Obviously, $p_{t-1}\le \tfrac{1}{2}$ in this case. Let us assume that $p_{t-1} > 0$. Using Lemma \ref{lem:z} with $z=p_t$ and $z_0=p_{t-1}$, we have
\begin{equation*}
\begin{split}
E[\sqrt{p_t}\mid p_{t-1}]
\le{}& E[Y_{t-1}\mid p_{t-1}]+\tfrac{1}{2}p_{t-1}^{-\frac{1}{2}}E[p_t-p_{t-1}\mid p_{t-1}]\\
{}&-\tfrac{1}{8}p_{t-1}^{-\frac{3}{2}}E [(p_t-p_{t-1})^2\mid p_{t-1}]
+ \tfrac{1}{16}p_{t-1}^{-\frac{5}{2}}E [(p_t-p_{t-1})^3\mid p_{t-1} ]
\end{split}
\end{equation*}
and thus
\begin{equation}
\begin{split}
E[Y_{t-1}{}&-\sqrt{p_t}\mid Y_{t-1}]
\ge-\tfrac{1}{2}p_{t-1}^{-\frac{1}{2}}E[p_t-p_{t-1}\mid p_{t-1}]\\
+{}&\tfrac{1}{8}p_{t-1}^{-\frac{3}{2}}E [(p_t-p_{t-1})^2\mid p_{t-1} ]
-\tfrac{1}{16}p_{t-1}^{-\frac{5}{2}}E [(p_t-p_{t-1})^3\mid p_{t-1} ].
\end{split}
\label{eq:drift2}
\end{equation}
Via Lemma~\ref{lem:moments}, we have for UMDA
\begin{equation*}
\begin{split}
E[Y_{t-1}{}&-\sqrt{p_t}\mid Y_{t-1}]\\
\ge{}&\tfrac{1}{8}p_{t-1}^{-\frac{3}{2}}\tfrac{1}{\mu}p_{t-1}(1-p_{t-1})-
\tfrac{1}{16}p_{t-1}^{-\frac{5}{2}}\tfrac{1}{\mu^2}p_{t-1}(1-p_{t-1})(1-2p_{t-1})\\
={}&\tfrac{1}{16\mu}p_{t-1}^{-\frac{1}{2}}(1-p_{t-1}) (2-\tfrac{1}{\mu p_{t-1}}(1-2p_{t-1}) )\\
\ge{}&\tfrac{1}{16\mu}p_{t-1}^{-\frac{1}{2}}(1-p_{t-1}),
\end{split}
\end{equation*}
where the last estimate follows from the fact that $p_{t-1} > 0$ implies $p_{t-1} \ge \frac 1 \mu$. Since $p_{t-1}\le \tfrac{1}{2}$, we have $p_{t-1}^{-\frac{1}{2}}(1-p_{t-1})\ge \tfrac{\sqrt 2}{2}$. Hence $E[Y_{t-1}-\sqrt{p_t}\mid Y_{t-1}]\ge \tfrac{\sqrt2}{32\mu}$. Using $q_t=\min\{p_t,1-p_t\}$, we have
\begin{equation*}
E[Y_{t-1}-Y_t\mid Y_{t-1}]\ge E[Y_{t-1}-\sqrt{p_t}\mid Y_{t-1}]\ge \tfrac{\sqrt2}{32\mu}.
\end{equation*}

Via the additive drift theorem \cite{HeY01} and $Y_0=\sqrt{\tfrac{1}{2}}$, we know that the expected time of $Y$-process hitting zero is at most $Y_0 / \frac{\sqrt2}{32\mu}=O({\mu})$.

Similarly, for cGA, with Lemma~\ref{lem:moments} equation~\eqref{eq:drift2} becomes
\begin{equation*}
\begin{split}
E[Y_{t-1}-\sqrt{p_t}\mid Y_{t-1}]
\ge\tfrac{1}{8}p_{t-1}^{-\frac{3}{2}}\tfrac{2}{K^2}p_{t-1}(1-p_{t-1})
=\tfrac{1}{4}p_{t-1}^{-\frac{1}{2}}\tfrac{1-p_{t-1}}{K^2}
\ge \tfrac{1}{4}\tfrac{\sqrt2}{2}\tfrac{1}{K^2}=\tfrac{\sqrt 2}{8K^2}.
\end{split}
\end{equation*}
Hence,
\begin{equation*}
E[Y_{t-1}-Y_t\mid Y_{t-1}]\ge E[Y_{t-1}-\sqrt{p_t}\mid Y_{t-1}]\ge \tfrac{\sqrt 2}{8K^2},
\end{equation*}
and via the additive drift theorem \cite{HeY01} and $Y_0=\sqrt{\tfrac{1}{2}}$, we conclude that the expected time of the $Y$-process reaching zero is at most $Y_0 / \tfrac{\sqrt 2}{8K^2} = 4K^2$. 
\end{proof}

\section{Behavior of Dominant Bits}
\label{sec:dominant}
One particular strength of BDE, as we will see in this section, is that it optimizes the most important decision variables quickly. We shall prove rigorously that BDE lets the frequency of a dominant bit in the population grow to the optimal bit value in time logarithmic in the population size. Taking the LeadingOnes and BinaryValue functions as examples, we demonstrate that BDE is also able to find and optimize a sequence of bits having the property that they become dominant one after the other. Due to the difficulties of analyzing full runs of BDE, for these results we again need to resort to iBDE or to at least take the assumption that the frequencies of bits that are momentarily neutral do not leave the middle range. 

The LeadingOnes and BinaryValue results suggest that BDE is able to optimize in a greedy fashion the most profitable decision variables first. This appears to be a valuable property when not necessarily aiming at finding the absolute optimum, but when rather aiming at finding a reasonably good solution in reasonable time. We shall not make this formal here, but note that the problem of finding approximate solutions has been formalized via notions like fixed-budget computation~\cite{JansenZ14} or the time-to-target runtimes $T_{A,f}(a)$ defined in~\cite[Section~3]{DoerrJWZ13}.

\subsection{Convergence Time of a Dominant Bit}\label{ssec:convdominant}

We take the dominant bit as example to discuss the behavior on the most important decision variables. A \emph{dominant bit} is a bit such that the fitness is always better if the value of the bit is one than if the bit value is zero, regardless of the values of the other bits. In the following, let us assume that we optimize some $D$-dimensional function $f$ via BDE and that $f$ is such that the first bit is dominant, that is, we have $f(1,x_2,\dots,x_D)>f(0,x_2,\dots,x_D)$ for all $x_2,\dots,x_D\in\{0,1\}$. Theorem~\ref{thm:domiTime} further below shows that the frequency of the dominant bit converges to the optimal value in time logarithmic in the population size. 

To ease reading the main proof, we first show separately two technical results on the behavior of the dominant bit. Note that by the parent-offspring selection and the definition of dominant bits, an individual having a one in the dominant bit can never be replaced by an individual having a zero in the dominant bit. Therefore, the main question is how difficult it is to replace a zero by a one in the dominant bit. This is what we analyze in the following lemma.

\begin{lemma}
Consider using BDE with population size $N$ to optimize a $D$-dimensional function $f$ with the first bit being dominant. Let $Z_g$ denote the number of zeros in the first bit position among all individuals of generation $g$. Let $X_i^g$ be an individual with first bit equal to $0$. Then the probability for changing this bit value to $1$ is 
\begin{align*}
\Pr[X_{i,1}^{g+1}=1 \mid X_{i,1}^g=0]=\frac{A_1Z_g^3+A_2Z_g^2+A_3Z_g+A_4}{(N-1)(N-2)(N-3)},
\end{align*}
where
\begin{equation*}
\begin{split}
A_1&=-4FC, \\
A_2&=(6N+6)FC,\\
A_3&=-C((1+2F)N^2+(8F-5)N+6+2F),\\
A_4&=CN^3+(2F-5)CN^2+(6+2F)CN.
\end{split}
\end{equation*}
\label{lem:conp}
\end{lemma}

\begin{proof}
%
%
In the notation of Algorithm~\ref{alg:originalBDE}, we observe that to have $X_{i,1}^{g+1}=1$, $U_{i,1}^g$ must be 1 (since $X_{i,1}^{g+1}\in\{U_{i,1}^g,X_{i,1}^g\}$ and $X_{i,1}^g=0$). In order to have $U_{i,1}^g=1$, $U_{i,1}^g$ must stem from $V_{i,1}^g$, and $V_{i,1}^g$ must be 1. Hence, $X_{i,1}^{g+1}=1$ if and only if one of the following cases holds.
\begin{itemize}
\item $X_{r_1,1}^g=1, X_{r_2,1}^g = X_{r_3,1}^g, \crand_1 \le C$.
\item $X_{r_1,1}^g=1, X_{r_2,1}^g \ne X_{r_3,1}^g, \mrand_1\ge F, \crand_1 \le C$.
\item $X_{r_1,1}^g=0, X_{r_2,1}^g \ne X_{r_3,1}^g, \mrand_1< F, \crand_1 \le C$.
\end{itemize}
Recalling that $Z_g$ is the number of zeros in the first bit among all individuals of the generation $g$, we compute
\begin{equation*}
\begin{split}
\Pr[X_{i,1}^{g+1}{}{}&=1 \mid X_{i,1}^g=0]\\
={}&\frac{(N-Z_g) ((N-Z_g-1)(N-Z_g-2)+(Z_g-1)(Z_g-2) )}{(N-1)(N-2)(N-3)}C\\
{}&+\frac{(N-Z_g)(Z_g-1)(N-Z_g-1)}{(N-1)(N-2)(N-3)}2(1-F)C\\
{}&+\frac{(Z_g-1)(N-Z_g)(Z_g-2)}{(N-1)(N-2)(N-3)}2FC\\
={}&\frac{A_1Z_g^3+A_2Z_g^2+A_3Z_g+A_4}{(N-1)(N-2)(N-3)},
\end{split}
\end{equation*}
with $A_1,\dots,A_4$ as in the statement of the lemma.
\end{proof}

To gain a better understanding of the quantity $\Pr[X_{i,1}^{g+1}=1 \mid X_{i,1}^g=0]$ just computed, let us define (for implicitly given $F$ and $C$) the function $S_N : [1,N]\mapsto [0,\infty)$ by
\begin{equation*}
S_N(z) = \frac{A_1z^3+A_2z^2+A_3z+A_4}{(N-1)(N-2)(N-3)}
\end{equation*}
with $A_1,\dots,A_4$ as in Lemma~\ref{lem:conp}, 
so that $\Pr[X_{i,1}^{g+1}=1 \mid X_{i,1}^g=0]=S_N(Z_g)$. Going from absolute numbers to relative numbers, we also define $s(x)=S_N(Nx)$ for all $x \in [\tfrac{1}{N},1]$. Figure~\ref{fig:sx} visualizes this function for two sets of parameter values.

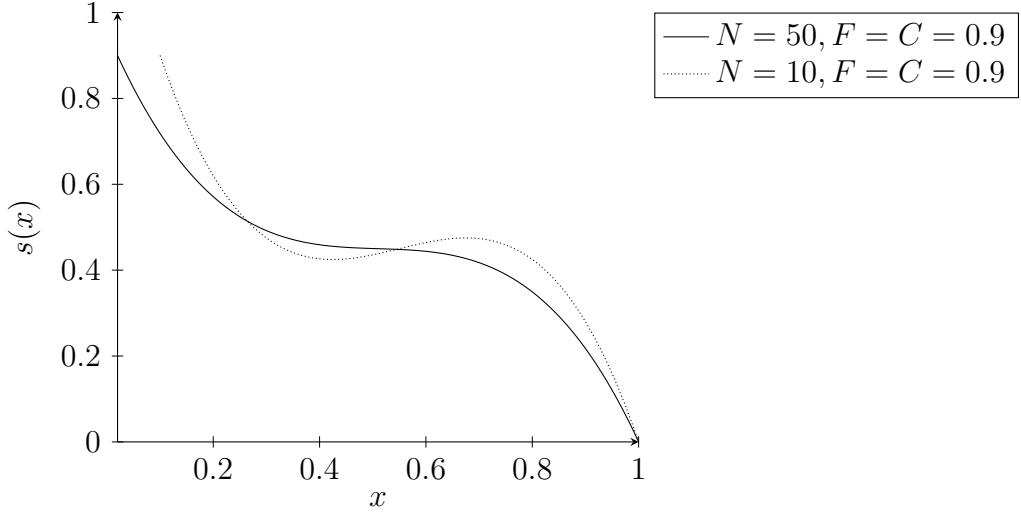
\begin{figure}
\centering
\begin{tikzpicture}
\begin{axis}[
    axis lines = left,
    xlabel = $x$,
    ylabel = {$s(x)$},
    ymin = 0,
    ymax=1,
    legend style={
      cells={anchor=east},
      legend pos=outer north east,},
    ]

\addplot [
    domain=0.02:1, 
    samples=100, 
    color=black,
]
{(-4*0.9*0.9*50^3*x^3+(6*50+6)*0.9*0.9*50^2*x^2-0.9*((1+2*0.9)*50^2+(8*0.9-5)*50+6+2*0.9)*50*x+0.9*50^3+(2*0.9-5)*0.9*50^2+(6+2*0.9)*0.9*50)/(49*48*47)};
\addlegendentry[right]{$N=50,F=C=0.9$}
 
\addplot [densely dotted,
    domain=0.1:1, 
    samples=100, 
    color=black,
    ]
{(-4*0.9*0.9*10^3*x^3+(6*10+6)*0.9*0.9*10^2*x^2-0.9*((1+2*0.9)*10^2+(8*0.9-5)*10+6+2*0.9)*10*x+0.9*10^3+(2*0.9-5)*0.9*10^2+(6+2*0.9)*0.9*10)/(9*8*7)};
\addlegendentry[right]{$N=10,F=C=0.9$} 
\end{axis}
\end{tikzpicture}
\caption{A visualization of $s(x):=S_N(Nx)$, that is, the probability for changing the first bit value to 1 when the population has a ratio of $x$ zeros in the first bit position. For $N$ not too small, this function is monotonically decreasing.}\label{fig:sx}
\end{figure}

In the following Lemma~\ref{lem:monode}, we show that the probability for changing the first bit value from 0 to 1 is monotonically decreasing with respect to $Z_g$ when $N$ is at least some constant (depending on $F$).
\begin{lemma}
Let $F, C \in (0,1)$ and $N\in [\max\{\frac{5-2F}{1-F},11\}, \infty ) \cap \N$. 
Then the following statements hold.
\begin{itemize}
\item $S_N(z)$ is monotonically decreasing.
\item For any $a\in(0,1)$, we have $S_N(aN)\ge C(1-a)(\frac{1}{2}-\frac{F}{8})$.
\end{itemize}
\label{lem:monode}
\end{lemma}

\begin{proof}
Let $g(z)=A_1z^3+A_2z^2+A_3z+A_4$. It is easy to see that $g(z)$ and $S_N(z)$ have the same monotonicity. We compute
\begin{equation*}
\begin{split}
g'(z)&=3A_1z^2+2A_2z+A_3,\\
g''(z)&=6A_1z+2A_2=-24FCz+12(N+1)FC.
\end{split}
\end{equation*}
Since $FC>0$, we know that $g''(z) \ge 0$ for $z \le \frac{N+1}{2}$ and $g''(z) \le 0$ for $z \ge \frac{N+1}{2}$. Hence $g'(z)$ has a unique maximum at $\frac{N+1}{2}$. Therefore,
\begin{align*}
g'(z)\le{}& g'(\tfrac{N+1}{2})\\
={}&3(-4FC)\tfrac{(N+1)^2}{4}+2(6N+6)FC\tfrac{N+1}{2}\\
{}&-C((1+2F)N^2+(8F-5)N+6+2F)\\
={}&C(-(1-F)N^2+(5-2F)N+F-6).
\end{align*}
With $F<1$ and $N\ge\frac{5-2F}{1-F}$, we estimate
\begin{align*}
-(1-F)N^2&+(5-2F)N+F-6\\
& \le -(5-2F)N+(5-2F)N+F-6
=F-6<0.
\end{align*} 
Hence $g'(z)$ is negative and thus $g(z)$ and $S_N(z)$ are monotonically decreasing.

When $z=aN$,
\begin{equation*}
S_N(aN)=C(1-a)\frac{(1-2Fa+4Fa^2)N^3+(2F-5-6Fa)N^2+(2F+6)N}{(N-1)(N-2)(N-3)}.
\end{equation*}
We compute
\begin{equation}
\begin{split}
(1-2F{}&a+4Fa^2)N^3+(2F-5-6Fa)N^2+(2F+6)N\\
{}&-(\tfrac{1}{2}-Fa+2Fa^2)(N-1)(N-2)(N-3)\\
={}&(\tfrac{1}{2}-Fa+2Fa^2)N^3+(2F-2-12Fa+12Fa^2)N^2\\
{}&+(\tfrac{1}{2}+2F+11Fa-22Fa^2)N+6(\tfrac{1}{2}-Fa+2Fa^2)\\
={}&(\tfrac{1}{2}-\tfrac{F}{8}+2F(a-\tfrac{1}{4})^2)N^3+(-F-2+12F(a-\tfrac{1}{2})^2)N^2\\
{}&+(\tfrac{1}{2}+\tfrac{27}{8}F-22F(a-\tfrac{1}{4})^2)N+6(\tfrac{1}{2}-\tfrac{F}{8}+2F(a-\tfrac{1}{4})^2)\\
\ge{}&(\tfrac{1}{2}-\tfrac{F}{8})N^3+(-F-2)N^2+(\tfrac{1}{2}+\tfrac{27}{8}F-22F(1-\tfrac{1}{4})^2)N\\
={}&(\tfrac{1}{2}-\tfrac{F}{8})N^3+(-F-2)N^2+(\tfrac{1}{2}-9F)N\\
\ge{}&\tfrac{3}{8}N^3-3N^2-\tfrac{17}{2}N=\tfrac{N}{8}(3(N-4)^2-116)\\
\ge{}&\tfrac{N}{8}(3(11-4)^2-116)=\tfrac{31}{8}N>0,
\end{split}
\label{eq:comp}
\end{equation}
where the last inequality in (\ref{eq:comp}) uses the fact $N \ge 11$. Consequently, 
\begin{align*}
S_N(aN)\ge{}& C(1-a)\frac{(\tfrac{1}{2}-Fa+2Fa^2)(N-1)(N-2)(N-3)}{(N-1)(N-2)(N-3)}\\
={}&C(1-a)(\tfrac{1}{2}-Fa+2Fa^2)=C(1-a)(\tfrac{1}{2}-\tfrac{F}{8}+2F(a-\tfrac{1}{4})^2)\\
\ge{}&C(1-a)(\tfrac{1}{2}-\tfrac{F}{8}).
\end{align*}
\end{proof}

We are now in the position to show that with high probability the whole population has a one in the dominant bit after $O(\ln N)$ iterations. 

\begin{theorem}
Consider using BDE with population size $N\ge\frac{5-2F}{1-F}$ to optimize a $D$-dimensional function $f$ with the first bit being dominant. Let $Z_g$ denote the number of zeros in the first bit position among all individuals of generation $g$. Let $T:=\min\{g \mid Z_g=0\}$ denote the convergence time of the first bit. 

Then there is a constant $\delta\in(0,1)$ depending on $F$ and $C$ such that conditional on $Z_0 \le 0.7 N$, which is an event that holds with probability $1 - \exp(-2N/25)$, the following statements hold.
\begin{itemize}
\item $E[T \mid Z_0] \le \tfrac{\ln(Z_0)+1}{\delta}\le \tfrac{\ln(N) +1}{\delta}$.
\item $\Pr[T\ge \tfrac{\ln(N)+r}{\delta}] \le \Pr[T\ge \tfrac{\ln(Z_0)+r}{\delta} \mid Z_0] \le e^{-r}, \forall r>0$.
\end{itemize}
\label{thm:domiTime}
\end{theorem}

\begin{proof}
We first note that, since $E[Z_0]=0.5N$, by the simple Chernoff inequality (Theorem 1.11 in \cite{Doerr11bookchapter}), we have $\Pr[Z_0 \ge 0.7N] \le \exp(-2N/25)$. Hence with probability at least $1-\exp(-2N/25)$, we have $Z_0 \le 0.7N$. In the following analysis, we condition on this event. 

As discussed already before in Lemma~\ref{lem:conp}, we have $X_{i,1}^{g+1}=1$ with probability one if $X_{i,1}^g=1$. Thus we have $Z_{g+1}\le Z_g$. A simple induction gives $Z_g \le Z_0$ for all $g\ge 0$. Therefore, applying Lemma~\ref{lem:conp} and Lemma~\ref{lem:monode}, we have
\begin{align*}
\Pr[X_{i,1}^{g+1}=1 \mid X_{i,1}^g=0] ={}& \frac{A_1Z_0^3+A_2Z_0^2+A_3Z_0+A_4}{(N-1)(N-2)(N-3)}\\
\ge {} & C(1-0.7)(\tfrac{1}{2}-\tfrac{F}{8})=\tfrac{3C(4-F)}{80}  \eqqcolon \delta
\end{align*}
for all generations $g$. Consequently, 
\begin{align*}
E[Z_{g+1} \mid Z_g]\le Z_g(1-\delta),
\end{align*}
and thus 
\begin{align*}
E[Z_g-Z_{g+1} \mid Z_g] \ge \delta Z_g.
\end{align*}
Now 
the multiplicative drift theorem with tail bounds~\cite{DoerrG13algo} shows the claim.
\end{proof}

In the result above, we showed a logarithmic convergence time assuming that we start with at most 70\% zeroes in the dominant bit, a condition that is satisfied apart from an exponentially small failure chance. We now further weaken this requirement to the (obviously necessary) condition that the dominant bit is not converged to the wrong value of zero.

\begin{corollary}
Consider using BDE with population size $N\ge\frac{5-2F}{1-F}$ to optimize a $D$-dimensional function $f$ with the first bit being dominant. Let $Z_g$ denote the number of zeros in the first bit position among all individuals of generation $g$. Let $T:=\min\{g \mid Z_g=0\}$ denote the convergence time of the first bit. There are constants $c_0>1$ and $\delta\in(0,1)$ depending on $F$ and $C$ such that, regardless of $Z_0$ (as long as $Z_0<N$), we have 
\begin{align*}
 E[T \mid Z_0] \le \tfrac{\ln(0.7N)+1}{\delta}+4\log_{c_0} \tfrac{0.3N}{N-Z_0}.
\end{align*}
\label{cor:strongerTime}
\end{corollary}

\begin{proof}
By Theorem~\ref{thm:domiTime}, it suffices to discuss how long we need to reach a $Z$-value of at most $0.7N$.

Consider the process $O_g=N-Z_g$ of the number of ones in the first bit position among all individuals of generation $g$. From Lemma~\ref{lem:conp} and Lemma~\ref{lem:monode}, we have
\begin{equation*}
\begin{split}
\Pr[X_{i,1}^{g+1}=1 \mid X_{i,1}^g=0] ={}& \frac{A_1Z_0^3+A_2Z_0^2+A_3Z_0+A_4}{(N-1)(N-2)(N-3)}\\
\ge {} & C(1-\tfrac{Z_g}{N})(\tfrac{1}{2}-\tfrac{F}{8})=\tfrac{O_g}{N}C(\tfrac{1}{2}-\tfrac{F}{8}).
\end{split}
\end{equation*}

Note that for all $i$ with $X_{i,1}^g = 0$, the events ``$X_{i,1}^{g+1} = 1$'' are independent. Also, as discussed before Lemma~\ref{lem:conp}, a one in the dominant bit of some individual $X_{i}^g$ is never replaced by a zero. Consequently, $O_{g+1}$ stochastically dominates $O_g + \Bin(Z_g,\tfrac{O_g}{N}C(\tfrac{1}{2}-\tfrac{F}{8}))$. Let $\tilde{O_g} \sim \Bin(Z_g,\tfrac{O_g}{N}C(\tfrac{1}{2}-\tfrac{F}{8}))$. Then
\begin{equation}
E[\tilde{O_g} \mid O_g]=\tfrac{O_g}{N}C(\tfrac{1}{2}-\tfrac{F}{8})Z_g.
\end{equation}

Let $Q$ denote the event that $\tilde{O_g} \ge E[\tilde{O_g} \mid O_g]$. From \cite{GreenbergM14} (see~\cite{Doerr18exceedexp} for an elementary proof), we know that a binomial random variable exceeds its expectation with probability at least $\tfrac{1}{4}$. Hence $\Pr[Q] \ge \tfrac{1}{4}$. The first time $S_Q$ that $Q$ happens, therefore is dominated by a geometric random variable with success probability $\tfrac{1}{4}$. Thus we have $E[S_Q] \le 4$. For $S_k$, the first time that $Q$ happens $k$ times, we have $E[S_k] \le 4k$. 

Since $O_{g+1}$ dominates $O_g+\tilde{O_g}$, recalling that $Q$ denotes the event that $\tilde{O_g} \ge E[\tilde{O_g} \mid O_g]$, we know that when $Q$ happens, 
\begin{equation*}
O_{g+1} \ge O_g+\tfrac{Z_g}{N}C(\tfrac{1}{2}-\tfrac{F}{8})O_g = (1+\tfrac{Z_g}{N}C(\tfrac{1}{2}-\tfrac{F}{8}))O_g.
\end{equation*}
Consequently, for $Z_g \ge 0.7N$, we have 
\begin{equation*}
O_{g+1} \ge (1+0.7C(\tfrac{1}{2}-\tfrac{F}{8}))O_g = c_0 O_g
\end{equation*}
with $c_0 = 1+0.7C(\tfrac{1}{2}-\tfrac{F}{8})$ a constant depending on $F$ and $C$ only and being greater than $1$.

Recalling that $S_k$ represents the first time that $Q$ happens $k$ times, with a simple induction, we have $O_{S_k} \ge \min\{0.3N,(c_0)^{k} O_0\}$. Since $Z_0<N$, we have $O_0 \ge 1$. For $k=\log_{c_0} \frac{0.3N}{O_0}$, we have $O_{S_k} \ge 0.3N$.

Hence the expected time to reach a $Z$-value of at most $0.7N$, is $4k=4\log_{c_0} \frac{0.3N}{O_0}$. From that point on, by Theorem~\ref{thm:domiTime}, it takes another expected number of $\tfrac{\ln(0.7N)+1}{\delta}$ iterations to have only ones in the dominant bit.
\end{proof}

\subsection{The Runtime for BDE Optimizing the LeadingOnes Function}
\label{sec:BDELO}

The above discussion shows the quick convergence of one dominant bit. It appears straight-forward to extend this result to a sequence of bits having the property that they become dominant one after the other. If such a sequence of sequentially dominating bits has length $D$, then the previous result suggests that BDE can optimize them all in time $O(D \log N)$. This is, under suitable assumptions, true. In fact, even more is true. Since BDE does not have to wait until a bit is converged, but can instead already start optimizing later bits of individuals which are further optimized, we can show a runtime of $O(D)$, that is, BDE optimizes such bits in amortized constant time. 

Let us make this precise. The classic benchmark function having sequentially dominating bits is the LeadingOnes function (defined in (\ref{eq:LO1})). Due to the stochastic dependencies in the search process, we cannot prove a runtime result for BDE on LeadingOnes without further assumptions. In fact, the sole difficulty which we have is the one we encountered already in Section~\ref{sec:neutral}, namely that we cannot prove that a neutral bit (other than the ones of the Needle function) is stable. Note that in the optimization of LeadingOnes a bit behaves neutral if there is a zero-bit to the left of it. As we have seen in Section~\ref{sec:neutral}, the frequencies of these bits stay very close to $\frac 12$. Hence taking the assumption that such neutral bits have their frequencies bounded away from zero by a constant margin, is very natural. Under such an assumption (which we will further justify below), we can prove the $O(D)$ runtime of BDE on LeadingOnes.

To put our result into perspective, let us quickly describe what is known in terms of proven runtimes for the LeadingOnes function. The LeadingOnes function was proposed by Rudolph~\cite{Rudolph97} as an example for a unimodal function that most likely is not optimized by the \oea in $O(D \log D)$ time, thus being a counterexample to the claim that all unimodal functions are that easy to optimize. 
Rudolph proved an upper bound of $O(D^2)$ and provided an experimental evidence for the $\Theta(D^2)$ runtime. The lower bound of $\Omega(D^2)$ was formally proven in~\cite{DrosteJW02}, together with a concentration result stating that the runtime is $\Theta(n^2)$ with probability $1 - \exp(-\Omega(D))$. 

A precise expression for the runtime of the \oea on LeadingOnes was given independently in~\cite{BottcherDN10,Sudholt13}. In~\cite{BottcherDN10}, also the optimal fixed and fitness-dependent mutation rates were determined. That the optimal mutation rate changes with the current fitness has spurred a number of subsequent results that determine the leading constant in the $\Theta(D^2)$ runtime for various hyperheuristics~\cite{AlanaziL14,LissovoiOW17,DoerrLOW18}. A runtime analysis for a general class of $(1+1)$ type algorithms on LeadingOnes was given in~\cite{Doerr18evocoparxiv}.

For the \mpoea with parent population size $\mu$ at most polynomial in $D$, a runtime of $\Theta(D^2 + D \mu \log \mu)$ was shown in~\cite{Witt06}. For the \oplea with offspring population size $\lambda$ at most polynomial in $D$, the runtime was determined to be $\Theta(\frac{D^2}{\lambda} + D)$ generations~\cite{JansenJW05}.  No result exists for the \mplea, for which surprisingly  few runtime results for classic benchmark problems exist~\cite{AntipovDFH18}. For the $(1+(\lambda,\lambda))$~EA proposed in~\cite{DoerrDE15}, also no formally proven result exists, but it can relatively easily be seen that with the recommended parameters $p = \lambda / D$ and $c = 1/\lambda$, a runtime guarantee of $O(D^2)$ generations holds (for this, one first observes in an iteration starting with a parent individual of fitness $k$, with probability $\Omega(1/D)$ the mutation winner has the $(k+1)$-st bit flipped, and then, that in such an iteration with constant probability the crossover winner has fitness at least $k+1$). We note that the quadratic (in terms of the number of fitness evaluations) runtimes of the \oea and \oplea remain valid under various noise assumptions, see~\cite{GiessenK16,Dang-NhuDDIN18,Sudholt18}.

For the estimation-of-distribution algorithm PBIL, an $O(\frac{D^2}{\lambda} + D \log \lambda)$ runtime (in generations) was shown in~\cite{LehreN18}. For the 1-ANT ant colony optimizer, the bounds $O(D^2 (6e)^{1/D\rho})$ and $\exp(\Omega(\min\{D,1/D\rho\}))$ were shown in~\cite{DoerrNSW11}. For either of the MMAS and MMAS$^*$ ant colony optimizers, the upper bounds $O(D^2 + D / \rho)$ and $O(D^2 \rho^{-\eps} + D / \rho\log(1/\rho))$ for an arbitrary small constant $\eps > 0$ were shown in~\cite{NeumannSW09} together with a lower bound of $\Omega(D^2 + D / \rho \log(2/\rho))$ for MMAS$^*$.

All upper bounds described above are at least of the order $O(D^2)$ fitness evaluations. A better upper bound, namely of order $O(D \log D)$, is known (for suitable parameter choices) only for the convex search algorithm (CSA)~\cite{MoraglioS17}, the \mbox{sc-GA}~\cite{FriedrichKK16}, and the sig-cGA~\cite{DoerrK18}. The black-box complexity of LeadingOnes is even smaller, namely $\Theta(D)$ for the XOR-invariant class of LeadingOnes functions~\cite{DrosteJW06} and $\Theta(D \log\log D)$ for the class of all functions having a fitness landscape isomorphic to the classic LeadingOnes function~\cite{AfshaniADDLM13}, but the algorithms behind these bounds are far from a general-purpose randomized search heuristic.

In the light of these results, our bound of $O(D)$ generations (under the assumption that the frequencies never go too low) is quite interesting. Clearly, we need an at least logarithmic population size (otherwise already the initial population would have bits converged to zero), but the analysis of the iBDE suggests that a logarithmic population size is also sufficient. Hence apart from this mean-field argument, we prove in this section that BDE with a logarithmic population size optimizes LeadingOnes with $O(D \log D)$ fitness evaluations, a runtime so far only observed for the not very common algorithms CSA, sc-GA, and sig-cGA.

\subsubsection{Runtime of BDE on LeadingOnes}

We extract the following lemma from the main proof to make it more readable.
\begin{lemma}
For all $a\in(0,\tfrac 25 \sqrt{10}]$ and $N\in [\tfrac{4}{a},+\infty)\cap \N$, we have
\begin{align*}
\frac{aN(aN-1)(aN-2)}{(N-1)(N-2)(N-3)} \ge \frac{a^3}{4}.
\end{align*}
\label{lem:sc}
\end{lemma}

\begin{proof}
Since $a\in(0,\tfrac 25 \sqrt{10}]$ and $N\ge \tfrac{4}{a}$, we calculate
\begin{align*}
aN(aN-1)&{}(aN-2)-\tfrac{a^3}{4}(N-1)(N-2)(N-3)\\
=&{} \tfrac{3}{4}a^3 N^3+(\tfrac 64 a^3-3a^2)N^2+(2a- \tfrac{11}{4}a^3)N+ \tfrac 64 a^3\\
\ge&{} 3a^2 N^2 +(\tfrac 64 a^3-3a^2)N^2+(2a- \tfrac{11}{4}a^3)N\\
\ge &{} \tfrac 64 a^3 N+(2a- \tfrac{11}{4}a^3)N=a(2-\tfrac 54 a^2)N\ge0,
\end{align*}
which proves the lemma.
\end{proof}

Now Theorem~\ref{thm:BDEforLOwAssumption} shows that under the assumption of all frequencies being bounded away from zero, BDE optimizes LeadingOnes within an expected number of $O(D)$ generations. To increase the readability of result and proof,  we give a non-asymptotic bound, namely $\tfrac{64}{\eps^4C}D$, but we did not try to optimize the constant in this $O(D)$ expression.

\begin{theorem}
Let $\eps \in(0,1)$. Consider using BDE with population size $N \ge \tfrac{8}{\eps}$ to optimize the $D$-dimensional LeadingOnes function. Assume that in each generation the number of ones in each bit is at least $\eps N$. Then the expected number of generations to find the optimum is at most $\tfrac{64}{\eps^4C}D$.
\label{thm:BDEforLOwAssumption}
\end{theorem}

\begin{proof}
Due to the parent-offspring selection strategy, the fitness of each individual 
$X_i^{g+1}$ in the next generation is greater than or equal to the fitness $f(X_i^g)$ in the current generation.
For the LeadingOnes function, we thus know that the first $f(X_i^g)$ ones in the current $X_i^g$ will be kept in all
following generations. We call these ones \emph{locked} and we call all other positions \emph{free}. Let $Z^g$ denote the total number of free positions in the population. We shall argue that for each generation with no optimum in the population $P^g$, we have $E[Z^g - Z^{g+1}] \ge \Omega(N)$, and use an additive drift argument to show that the time $T$ to first find the optimum satisfies $E[T] = O(D)$.

Let $g$ be such that $P^g$ does not contain an optimal solution. We first show that $E[Z^g - Z^{g+1}] \ge \tfrac{1}{64} \eps^4N$. Let $\eps' = \frac 12 \eps$. Let $j \in \{0,\dots,D\}$ be maximal such that at least $\eps' N$ individuals of $P^g$ have a fitness of $j$ or more. By our assumption that each bit position contains at least $\eps N$ ones, we have $j \ge 1$, and by our assumption that $P^g$ contains no optimum, we have $j < D$.

We argue that at least $\eps' N$ individuals have a fitness of less than $j$. Note that an individual $X_i$ with fitness at least $j$ such that $X_{i,j+1} = 1$ has in fact fitness at least $j+1$. If there are less than $\eps' N$ individuals with fitness less than $j$, then our assumption on the presence of ones, the fact that $j < D$, and a simple counting argument show that at least $\eps N - \eps'N = \eps'N$ of the at least $1 - \eps'N$ individuals with fitness at least $j$ have actually a fitness of at least $j+1$, in contradiction with our definition of $j$.

Let $X_i^g$ be an individual with $f(X^g_i) < j$ and let $j' = f(X^g_i)$ be its fitness. When generating $X_i^{g+1}$, we consider the event that $X_{r_1}^g, X_{r_2}^g$ and $X_{r_3}^g$ all have the fitness at least $j$. Since $j' < j$, we have $X_{r_1,j'+1}^g=X_{r_2,j'+1}^g=X_{r_3,j'+1}^g=1$, and thus we have $X_{i,j'+1}^{g+1} = U_{i,j'+1}^g = 1$ with probability $C$. Note that always we have $X_{i,k}^{g+1} = 1$ for $k \le j'$. Consequently, 
\begin{align}
\Pr[f(X_{i}^{g+1}) \ge j'+1]\ge \frac{\eps' N(\eps' N -1)(\eps' N -2)}{(N-1)(N-2)(N-3)}C \ge \frac{(\eps')^3}{4}C=\frac{\eps^3C}{32},\label{eq:improvement}
\end{align}
where the last inequality stems from Lemma~\ref{lem:sc} with $N\ge \tfrac{8}{\eps}=\tfrac{4}{\eps'}$. Recalling that $X_i^g$ contributes exactly $D-j'$ free positions to $Z^g$, hence, we know that with probability at least $\frac{1}{32}\eps^3C$, $X^{g+1}_i$ contributes at least one less free position to $Z^{g+1}$. 

Since there are at least $\eps'N$ individuals with fitness below $j$ and each of them with probability at least $\frac{1}{32}\eps^3C$ loses a free position, we have $E[Z^{g+1} - Z^g] \ge \frac{1}{32}\eps^3C \cdot\eps'N=\tfrac{1}{64}\eps^4CN$.

We finally transform this information on the expected shrinking of $Z^g$ into a drift argument bounding the runtime. Let $\tilde Z^g$ be defined by $\tilde Z^g = 0$, if $P^g$ contains an optimal solution, and $\tilde Z^g = Z^g$ otherwise. Since $\tilde Z^g \le Z^g$, for all $g$ such that $P^g$ does not contain an optimal solution we have $E[\tilde Z^g - \tilde Z^{g+1}] = E[Z^g - \tilde Z^{g+1}] \ge E[Z^g - Z^{g+1}] \ge \tfrac{1}{64}\eps^4CN$. Noting that the runtime of the BDE is $T = \min\{g \mid \tilde Z^g = 0\}$, the additive drift theorem~\cite{HeY01} and the just computed drift $E[\tilde Z^g - \tilde Z^{g+1} \mid g < T] \ge \tfrac{1}{64}\eps^4CN$ gives 
\[E[T] \le \frac{ND}{\tfrac{1}{64}\eps^4CN} = \frac{64}{\eps^4C}D.\]

\end{proof}

Note that in the computation of~\eqref{eq:improvement}, we cannot use the analyses conducted in Section~\ref{ssec:convdominant} as we not only want to generate a one in position $j'+1$ of the $i$-th individual, but we also want to have ones in all lower positions of the mutant $V_{i}^g$. Note also that the proof above heavily exploits the dependencies stemming from the way BDE generates the mutants. In other words, the proof above is not valid for the analysis of iBDE on LeadingOnes. In fact, we do not have a mathematical proof showing that iBDE optimizes LeadingOnes also in $O(D)$ iterations (under conditions similar to the ones of Theorem~\ref{thm:BDEforLOwAssumption}).

\subsubsection{The Assumption in Theorem~\ref{thm:BDEforLOwAssumption}}

It remains to verify the assumption made in Theorem~\ref{thm:BDEforLOwAssumption} that the frequencies are bounded away from zero. Due to the stochastic dependencies in BDE, we are momentarily lacking the methods to do this via a mathematical proof. We therefore consult the experiments described in Section~\ref{sec:Anyneutral}, observe that they support the assumption for both BDE and iBDE, and then formally prove the assumption to be valid in iBDE.

From the runtimes shown in Table~\ref{tbl:LOruntime} and the average fitnesses shown in Figure~\ref{fig:BDEandiBDEfitonLOmBW}, we see a generally similar optimization behavior of BDE and iBDE. The minimum frequencies depicted in Figure~\ref{fig:NeutralQuantiles} show clearly that for both BDE and iBDE, the frequencies are bounded away from zero by a constant. Also, the minimum frequencies behave similarly in both algorithms. From all this, it appears reasonable that BDE and iBDE behave similarly with respect to the assumption made in Theorem~\ref{thm:BDEforLOwAssumption}.

%
%

We now prove that the assumption made in Theorem~\ref{thm:BDEforLOwAssumption} is valid for iBDE. The main argument for this result is that the process can be coupled with the optimization process on a LeadingOnes function with a neutral bit. For the latter, we have the desired result from our understanding of neutral bits in iBDE.

\begin{lemma}\label{lem:LOassumiBDE}
Let $N \ge\max\{\tfrac{5-2F}{1-F},\tfrac{3125-1224F}{625-612F},\tfrac{625}{24F}\}$. Consider using iBDE with population size $N$ to optimize the $D$-dimensional LeadingOnes function.  For all $j\in\{1,\dots,D\}$, let $Y_g(j)$ denote the number of ones in the $j$-th bit position among all individuals of generation $g$. There is a constant $c'>0$, depending on $F$ only, such that
\begin{equation*}
\Pr [\forall g\le T,\forall j\in\{1,\dots,D\}: Y_g(j) \ge 0.4N] \ge 1-D(T+1)\exp(-c'N)
\end{equation*}
for all $T\in \N$.
\end{lemma}

\begin{proof}
Let $\ell \in \{1, \dots, D\}$. We show that we can couple the optimization process on LeadingOnes and on the LeadingOnes function with the $\ell$-th bit neutral in a way that at all times and for all individuals the first $\ell-1$ bits are identical and the $\ell$-th bit in the original process is at least as large as in the process with the neutral bit. Consequently, a lower bound on the number of ones in the $\ell$-th bit for the process with the neutral bit carries over to the true process.

To make this precise, let $X_{i,j}^g$ denote the value of the $j$-th bit of the individual $X_i^g$ in a run of iBDE on the LeadingOnes function $f$ defined in (\ref{eq:LO1}). Let $\tilde{X}_{i,j}^g$ denote the corresponding bit value in a run of iBDE on the function $\tilde{f}$ defined by 
$$
\tilde{f}(X)=f(X_1,\dots,X_{j-1},1,X_{j+1},\dots,X_{D}).
$$

We show by induction that we can couple the two processes in a way that for all $g$ and $i$ we have (i)~$X^g_{i,j} = \tilde X^g_{i,j}$ for all $j < \ell$ and (ii)~$X^g_{i,\ell} \ge \tilde X^g_{i,\ell}$. Clearly, there is nothing to show for $g=0$, that is, for the random initial population. Hence let $g \ge 0$ and assume that the desired coupling exists for this generation. We show that the desired coupling also exists for generation $g+1$. Exploiting the coupling in generation $g$, we can assume that we have concrete outcomes for $X_i^g$ and $\tilde X_i^g$ such that $X^g_{i,j} = \tilde X^g_{i,j}$ for all $j < \ell$ and $X^g_{i,\ell} \ge \tilde X^g_{i,\ell}$. Using identical randomness in the generations of the first $\ell-1$ bits (that is, by using the identity mapping as coupling), we immediately obtain that the mutants $V_i^g$ and $\tilde V_i^g$ satisfy $V^g_{i,j} = \tilde V^g_{i,j}$ for all $j < \ell$.

To analyze the $\ell$-th bit of the $i$-th individual, let $Y_g^{i,-} = \sum_{k=1,k\ne i} X_{k,\ell}^g$ and $\tilde{Y}_g^{i,-} = \sum_{k=1,k\ne i} \tilde{X}_{k,\ell}^g$. Since $X_{k,\ell}^g \ge \tilde{X}_{k,\ell}^g$ for all $k$, we have $Y_g^{i,-} \ge \tilde{Y}_g^{i,-}$. Now the probabilities of sampling $V_{i,\ell}^g$ and $V_{i,\ell}^g$ as one satisfy $R_N(Y_g^{i,-}) \ge R_N(\tilde{Y}_g^{i,-})$, since $R_N(\cdot)$ defined in Lemma~\ref{lem:mutprobin} is monotonically increasing. Hence we can couple the mutants in a way that also $V_{i,\ell}^g \ge \tilde{V}_{i,\ell}^g$.

By using identical outcomes for the random decisions in generating the trials, we can also ensure that $U^g_{i,j} = \tilde U^g_{i,j}$ for all $j < \ell$ and $U^g_{i,\ell} \ge \tilde U^g_{i,\ell}$. 

We finally argue that the selection between parent and offspring takes the desired relation between $X_i^g$ and $\tilde X_i^g$ into the next generation. Since both parents and both trials agree on the first $\ell-1$ bits, an easy case distinction shows that either both parents or both trials are selected except possibly in the case that $X^g_{i,1} = \dots = X^g_{i,\ell} = 1$ and $U^g_{i,1} = \dots = U^g_{i,\ell-1} = 1$. In this case, however, regardless of the selection, we have $X^{g+1}_{i,1} = \dots = X^{g+1}_{i,\ell} = 1$ and $\tilde X^{g+1}_{i,1} = \dots = \tilde X^{g+1}_{i,\ell-1} = 1$, and hence again the desired relation.

Let $\tilde Y_g(\ell) = \sum_{i=1}^N \tilde X^g_{i,\ell}$ and recall that $Y_g(\ell) = \sum_{i=1}^N X^g_{i,\ell}$. By the relation just proven, we have $Y_g(\ell) \ge \tilde Y_g(\ell)$. Since the $\ell$-th bit of $\tilde f$ is neutral, we know from the proof of Theorem~\ref{thm:iBDENeutral} that 
\begin{align*}
\Pr[\exists g\le T: Y_{g}(\ell) \le 0.4N] \le \Pr[\exists g\le T: \tilde Y_{g}(\ell) \le 0.4N] \le (T+1)\exp(-c'N),
\end{align*}
where $c'$ is defined in Theorem~\ref{thm:iBDENeutral}. 
By a union bound over all positions $\ell$, we have
\begin{align*}
\Pr[\exists g \le T, \exists \ell \in \{1,\dots,D\}: Y_{g}(\ell) \le 0.4N] \le D(T+1)\exp(-c'N).
\end{align*}
Hence,
\begin{equation*}
\Pr [\forall g\le T, \forall \ell \in \{1,\dots,D\}: Y_g(\ell) \ge  0.4N] \ge 1-D(T+1)\exp(-c'N).
\end{equation*}
\end{proof}

\subsection{The Runtime for BDE Optimizing the BinaryValue Function}

We now brief\/ly mention that the results shown for LeadingOnes in Section~\ref{sec:BDELO} also hold for the BinaryValue (BinVal) function $f: \{0,1\}^D \to \Z$ defined by 
\begin{align*}
f(X)=\sum_{i=1}^D 2^{D-i} X_i
\end{align*}
for all  $X=(X_1,\dots,X_D)$.
This is not totally surprising, but since not too many results exist on how complicated algorithms optimize BinaryValue and since for many algorithms the runtimes on LeadingOnes and BinaryValue differ, we feel that discussing this in less than two pages is justified. 

The few results we are aware of are the following. The BinaryValue function belongs to the class of pseudo-Boolean linear functions, which kept the field busy for quite a while. That the runtime of the \oea on any linear function (with at least $D^\eps$ non-zero coefficients) is $\Theta(D \log D)$ was first proven in the seminal paper~\cite{DrosteJW02}. Increasingly sharper results or simpler proofs have been given, e.g., in~\cite{HeY01,Jagerskupper08,DoerrJW12algo,Witt13}. 

For the \oplea with $\lambda = O(D)$, a tight runtime bound of $\Theta(\frac{D \log D}{\lambda} + D)$ generations was given in~\cite{DoerrK15}, which also showed that for this algorithm the BinaryValue function is harder than the linear function OneMax. For the \mpoea, an upper bound of $O(D \mu \log \mu + D^2)$ was shown and a lower bound of $\Omega(D \mu \log \mu + D \log D)$ was conjectured recently in~\cite{Witt18}. 

The first mathematical runtime analysis for an EDA~\cite{Droste06} gave an interesting picture of how the cGA without margins optimizes linear functions. When $K \ge D^{1+\eps}$, $\eps>0$ any constant, then for any linear function $O(KD)$ iterations suffice to find the optimum with at least constant probability. For the BinaryValue function and any $K$, with probability at least $1 - \exp(-K/48)$ the cGA needs more than $KD/3$ iterations to find the optimum. Interestingly, for the linear function OneMax, with at least constant probability the optimum is found already after $O(K\sqrt D)$ iterations (and this result is tight).

Again for the BinaryValue function, a lower bound of $\Omega(D^2)$ regardless of $K$ was shown for the expected runtime of the cGA in~\cite{Witt18}. Also, it was shown that for $K \ge cD \log D$ with $c$ a sufficiently large constant and $K = D^{O(1)}$, with high probability this runtime is $O(KD)$. For the StSt$\binom \mu 2$GA, which maintains a population with size $\mu$ and creates in each iteration two individuals via uniform crossover, Witt~\cite{Witt18} showed that the runtime on BinaryValue is $O(\mu D \log\mu)$ with high probability when $\mu \ge c D \log^2 D$ for $c$ a sufficiently large constant and $\mu = D^{O(1)}$.
For PBIL, an $O(\frac{D^2}{\lambda} + D \log \lambda)$ runtime was shown in~\cite{LehreN18}.

In~\cite{DoerrNSW11}, it was proven that the 1-ANT ant colony optimizer 
finds the optimum of the BinaryValue function in an expected time bounded by $2^{\Omega(\min\{D,1/{(D\rho)}\})}$ and $O(D^2\cdot 2^{O((\log^2D)/{(D\rho)})})$. In~\cite{KotzingNSW11}, expected runtime bounds of $O(D^2+D/\rho)$ and $O(D^2(1/\rho)^\eps+(D/\rho)/{(\log(1/\rho))})$ were shown for MMAS and MMAS*, respectively, for every constant $\eps>0$. 

Finally, we note that the black-box complexity of BinaryValue functions can be very small. For the class of all functions $f_z : \{0,1\}^D \to \Z; x \mapsto f(x \XOR z)$, the black-box complexity was shown to be exactly $2 - 2^{-D}$ in~\cite{DrosteJW06}. Even when regarding the class of all fitness functions having a fitness landscape isomorphic to the classic binary value function (that is, we also allow permutations of the bit-position), the black-box complexity is at most $\lceil \log_2 D \rceil + 2$ as shown in~\cite{DoerrW14ranking}. 

With exactly the same proof as for Theorem~\ref{thm:BDEforLOwAssumption}, we obtain that BDE optimizes also the BinaryValue function in $O(D)$ iterations when we can assume that at all times each bit position contains a constant fraction of ones. This result is interesting, among others, in that it, together with our results of Section~\ref{sec:onemax}, shows that BDE behaves very different from the cGA on linear functions. Whereas the cGA finds OneMax much easier than BinaryValue (see Droste's results described above), our results show that BDE easily optimize BinaryValue, but has some difficulties with OneMax.

\begin{theorem}
Let $\eps \in(0,1)$. Consider using BDE with population size $N \ge \tfrac{8}{\eps}$ to optimize the $D$-dimensional BinaryValue function. Assume that in each generation the number of ones in each bit is at least $\eps N$. Then the expected number of generations to find the optimum is at most $\tfrac{64}{\eps^4C}D$.
\label{thm:BDEforBVwAssumption}
\end{theorem}

As in Section~\ref{sec:BDELO} we cannot prove rigorously that we have a constant rate of ones in each bit position, so we resort to our mean-field argument. With the same proof as for Lemma~\ref{lem:LOassumiBDE}, we obtain the minimum-frequency assertion of iBDE.

\begin{lemma}
Let $N \ge\max\{\tfrac{5-2F}{1-F},\tfrac{3125-1224F}{625-612F},\tfrac{625}{24F}\}$. Consider using iBDE with population size $N$ to optimize the $D$-dimensional BinaryValue function.  For all $j\in\{1,\dots,D\}$, let $Y_g(j)$ denote the number of ones in the $j$-th bit position among all individuals of generation $g$. There is a constant $c'>0$, depending on $F$ only, such that
\begin{equation*}
\Pr [\forall g\le T,\forall j\in\{1,\dots,D\}: Y_g(j) \ge 0.4N] \ge 1-D(T+1)\exp(-c'N)
\end{equation*}
for all $T\in \N$.
\label{lem:BVassumiBDE}
\end{lemma}

It remains to argue with experimental data for the fact that BDE and iBDE behave similarly when optimizing the BinaryValue function. 

In our experiments we use the setting $D=1000, N=1000, F=0.2$, and $C=0.3$ (for both BDE and iBDE). For each algorithm, 100 independent runs are conducted. Table~\ref{tbl:BVruntime} gives the minimum, average and maximum runtimes, Figure~\ref{fig:BDEandiBDEfitonBV} plots the average number of ones in the best individual and the LeadingOnes value of the best individual of BDE and iBDE, and Figure~\ref{fig:QuantilesBV} plots the minimum, maximum, and $10\%,50\%,90\%$ quantiles of the frequency of ones in the last bit (which has the least influence on the fitness) as well as the minimum frequency of ones among all bit positions and all runs. 

\begin{table}[H]
\centering
  \caption{The runtimes of BDE and iBDE optimizing the BinaryValue function in 100 independent runs ($D=1000, N=1000, F=0.2$, $C=0.3$).}
  \label{tbl:BVruntime}
    \begin{tabular}{cccc}
    \hline
    & minimum & average & maximum\\
   \hline
    BDE & 1179 & 1195 & 1208\\
    iBDE & 1180 & 1195 & 1216 \\
  \hline
 \end{tabular}
\end{table}%


\begin{figure}[H]
  \centering
  \begin{minipage}[t]{1\textwidth}
  \centering
  \includegraphics[width=3.8in]{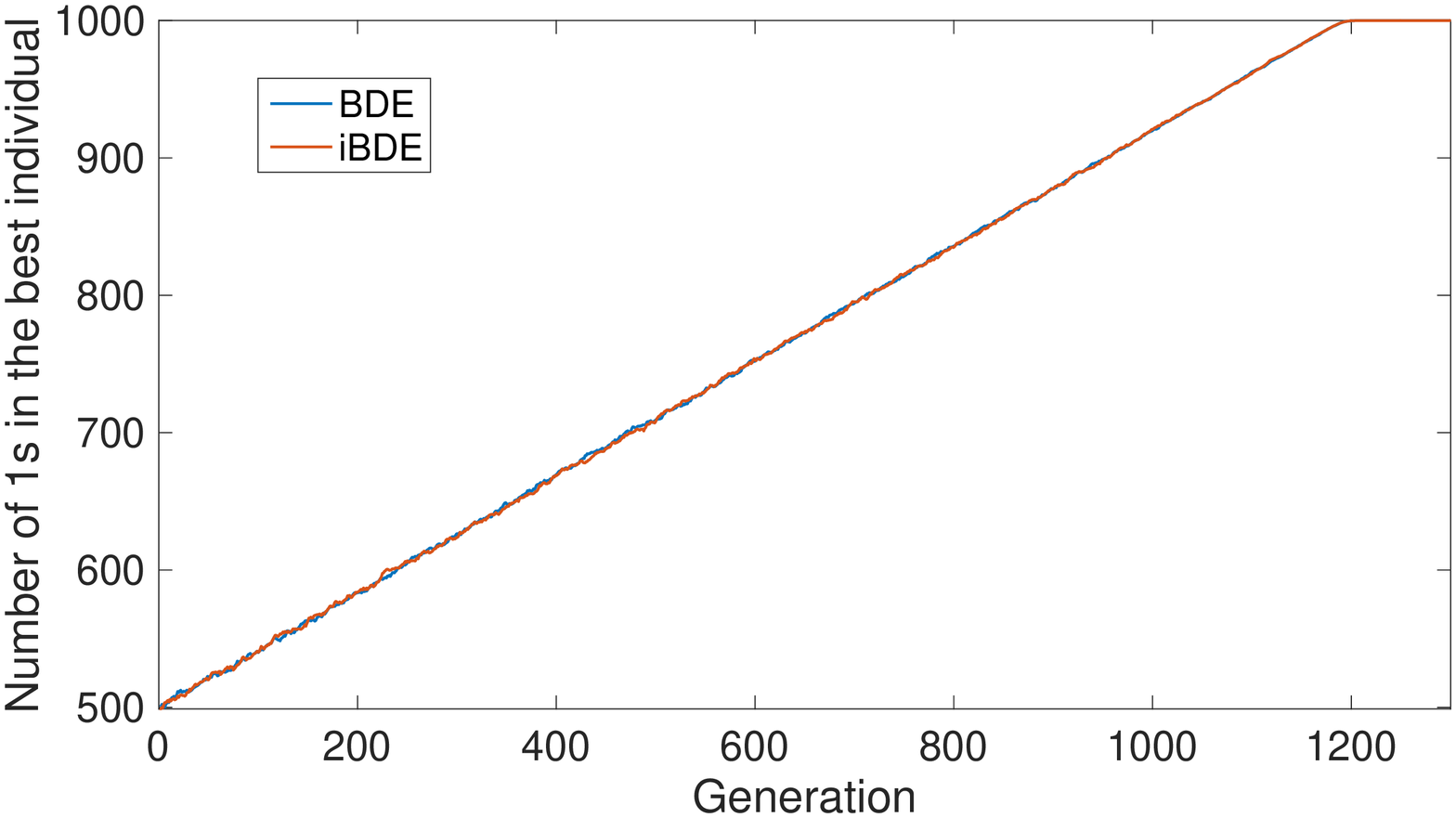}
  \end{minipage}
\vspace{1.pt}
  \begin{minipage}[t]{1\textwidth}
  \centering
  \includegraphics[width=3.8in]{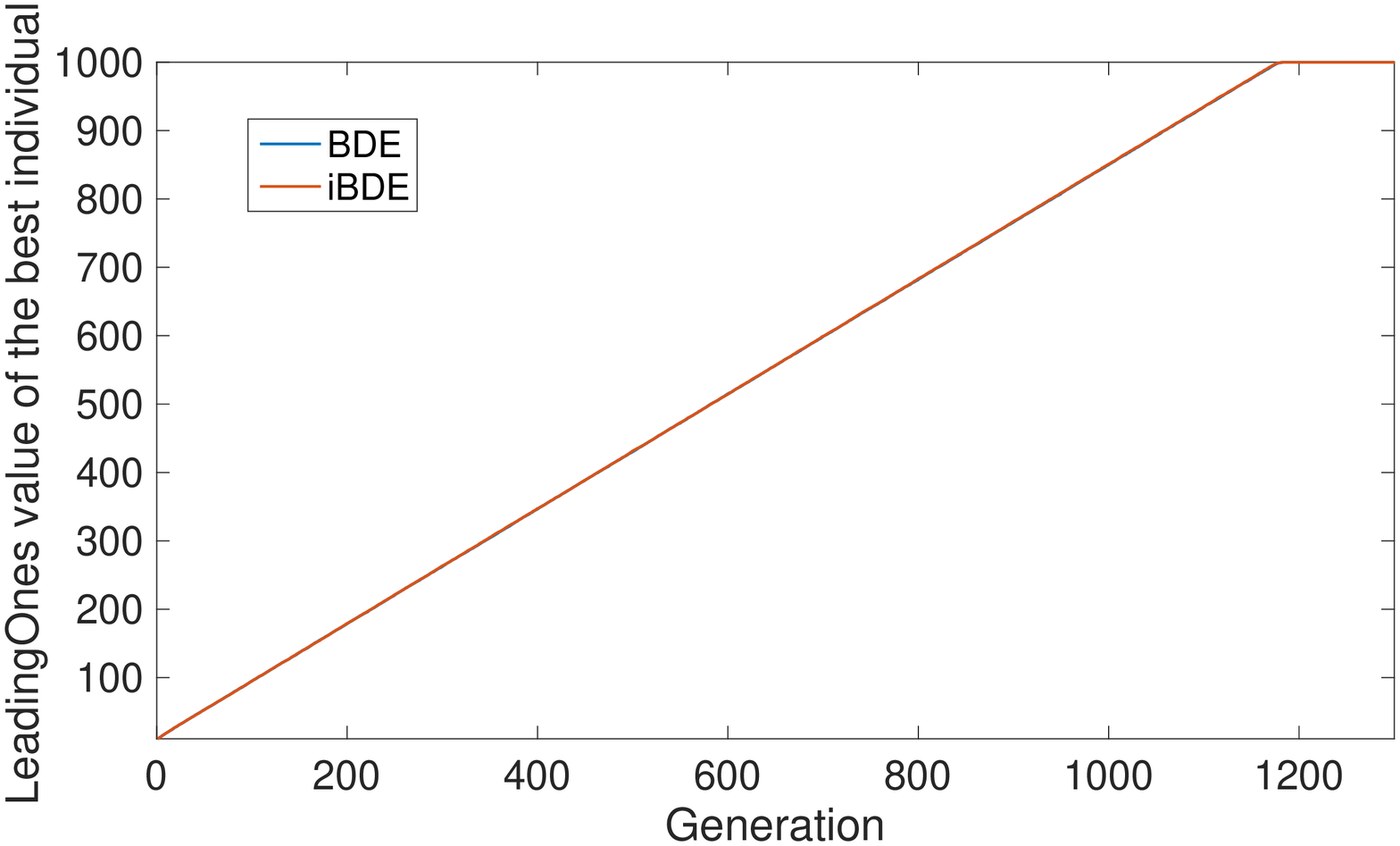}
  \end{minipage}
  \caption{Number of ones in the best individual (top) and LeadingOnes value of the best individual (bottom) among 100 runs BDE and iBDE optimizing BinaryValue function  ($D=1000, N=1000, F=0.2$, $C=0.3$).}
    \label{fig:BDEandiBDEfitonBV}
\end{figure}

\begin{figure}[H]
  \centering
  \begin{minipage}[t]{1\textwidth}
  \centering
  \includegraphics[width=3.8in]{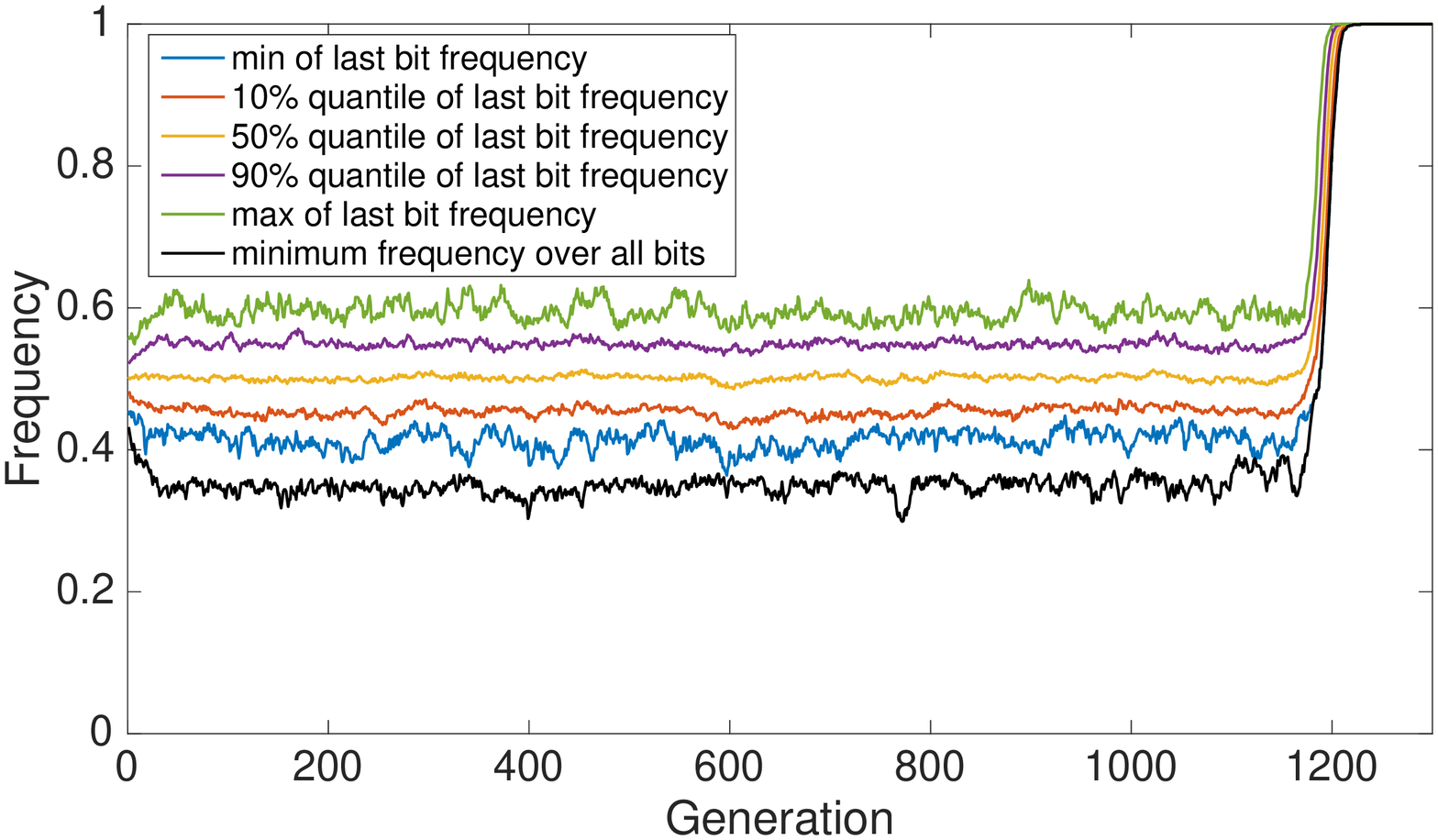}
  \end{minipage}
\vspace{1.pt}
  \begin{minipage}[t]{1\textwidth}
  \centering
  \includegraphics[width=3.8in]{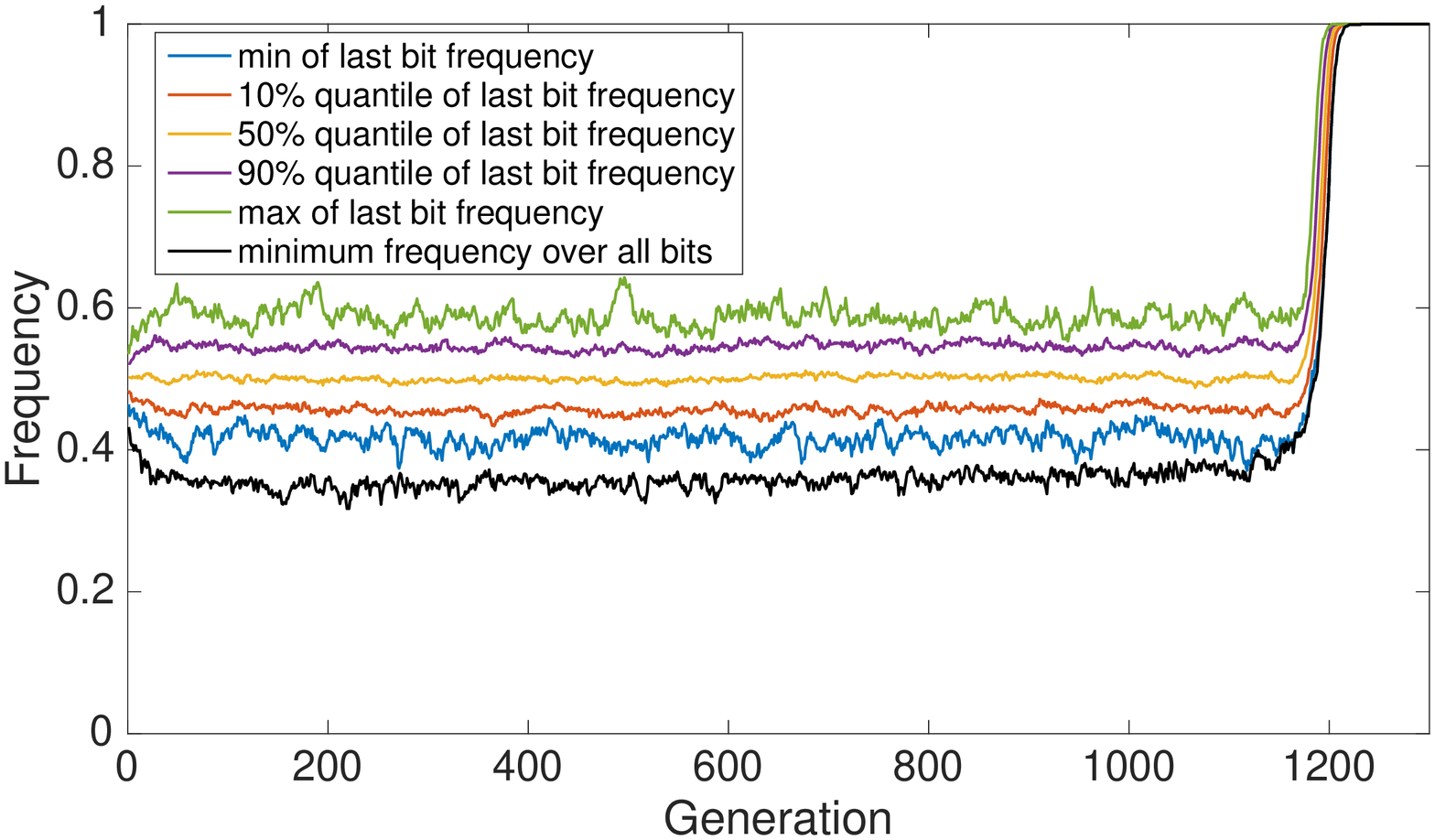}
  \end{minipage}
  \caption{The minimum, maximum, and $10\%,50\%,90\%$ quantiles of the frequency of ones in the last bit position among 100 runs of BDE (top) and iBDE (bottom) optimizing the BinaryValue function ($D=1000, N=1000, F=0.2$, $C=0.3$). Also depicted are the minimum frequency of ones in all bit positions and all runs.}
    \label{fig:QuantilesBV}
\end{figure}


All results demonstrate a very similar optimization behavior of BDE and iBDE. The minimum frequencies depicted in Figure~\ref{fig:QuantilesBV} in addition show clearly that for both BDE and iBDE, the frequencies are bounded away from zero by a constant. From all this, it appears reasonable that BDE and iBDE behave similarly with respect to the assumption made in Theorem~\ref{thm:BDEforBVwAssumption}.

\section{Negative Consequences from the Stability}\label{sec:onemax}

It has been observed that the stability of an algorithm can lead to difficulties when solving problems in which the fitness only gives a weak signal on what is the right value for a bit-position. In~\cite{DoerrK18}, it was proven that the scGA, a version of the cGA artificially made stable, has a runtime of $\exp(\Omega(\min\{n,K\}))$ on the OneMax benchmark function when the hypothetical population size is $K$. For the convex search algorithm (CSA), an at least super-polynomial runtime was shown for the optimization of OneMax~\cite{DoerrK18arxiv}. 

To see if BDE suffers from its stability in a similar manner, we now analyze its performance on the OneMax function as well. Our results will be less conclusive than those for the scGA and CSA, but still rather indicate that BDE finds it hard to optimize OneMax. As a proven result, we show that when the initial population is chosen such that each bit value is one independently with a probability strictly larger than $0.5$ (but less than one), then BDE does not profit from the better fitness of this population, but instead has a runtime exponential in the dimension~$D$. We can not prove such a result for the usual uniform random initialization. Our experiments, however, indicate a super-polynomial runtime.

\subsection{Runtime of BDE When Initialized With a Good Random Population}

In this subsection, we analyze the runtime of BDE on OneMax when initialized with each bit value being one with probability $0.5 < p < 1$ independently. Note that the expected fitness of each initial individual is $pD$, which is better than the value $0.5 D$ obtained from a uniform random initialization. Despite this fitness advantage, we can show that BDE with high probability needs an exponential time to find the optimum. The main argument is that BDE already needs that long to generate an offspring that is better than its parent. 

The following elementary estimate will be needed in our proof.

\begin{lemma}\label{lem:ineq1}
For $x \in [0,1]$,
\begin{equation}
\exp(-\tfrac{8}{3}x(1-x)(x-\tfrac{1}{2})^2) \ge x.
\label{eq:ineq1}
\end{equation}
\end{lemma}

\begin{proof}
When $x=0$, $\exp(-\tfrac{8}{3}x(1-x)(x-\tfrac{1}{2})^2)=1>0=x$. We consider $x\in(0,1]$ in the following.
Let $\ell(x)=-\tfrac{8}{3}x(1-x)(x-\tfrac{1}{2})^2 - \ln x$. Then
\begin{align*}
\ell'(x)={}&-\tfrac{8}{3}((1-2x)(x-\tfrac{1}{2})^2+x(1-x)2(x-\tfrac{1}{2}))-\tfrac{1}{x}\\
={}&-\tfrac{8(1-2x)}{3}(2x^2-2x+\tfrac{1}{4})-\tfrac{1}{x}\\
={} & \tfrac{2x-1}{3}((4x-2)^2-2)-\tfrac{1}{x}\\
\le {}& \tfrac{2x-1}{3}((4x-2)^2-2)-1.
\end{align*}
It is not difficult to see that when $x \in (0,\tfrac{1}{2})$, $(1-2x)(2-(2-4x)^2) \le 2$, when $x \in (\tfrac{1}{2},1]$, $(2x-1)((4x-2)^2-2) \le 2$, and when $x=\tfrac{1}{2}$, $(2x-1)((4x-2)^2-2) =0$. Thus for $x \in (0,1]$, we have $(2x-1)((4x-2)^2-2) \le 2$. Hence $\ell'(x) \le \tfrac{2}{3}-1<0$. Since thus $\ell(x)$ is monotonically decreasing, and we have $\ell(x)\ge \ell(1)=0$ for all $x \in (0,1]$, which gives the claim.
\end{proof}

Now we state and prove our lower bound on the runtime.

\begin{theorem}\label{thm:rtlargeprob}
Let $0.5 < p < 1$. Consider using BDE with population size $N$ to optimize the $D$-dimensional OneMax function when in the initial population each bit is one independently with probability $p$. Let $\gamma=\tfrac{8}{3}\tfrac{F^2Cp(1-p)(p-0.5)^2}{1+(1-2FCp(1-p))F(1-2p)^2}$, which is a positive constant depending on the constants $F$, $C$, and $p$ only. For all $t \in \N$, the runtime $T$ satisfies
\begin{align*}
\Pr[T \ge t] \ge 1 - t N \exp(-\gamma D).
\end{align*}
In particular, $E[T] \ge \frac 1{2N} \exp(\gamma D) -1$.
\end{theorem}

The proof involves some computations similar to those done earlier in this work, namely what is the probability to generate a $1$ in a mutant or trial. The difference, and this makes things a little easier, is that here we can assume that the bit values used in generating the mutants are independent. 

\begin{proof}
Let $t \in \N$ and let $A_t$ be the event that within the first $t$ iterations, BDE generates no trial that is at least as good as its parent. 

We first show that a trial vector $U_i$ generated from four random individuals $X_i,X_{r_1},X_{r_2}$ and $X_{r_3}$ with probability at least $1-\exp(-\gamma D)$ is worse than the parent $X_i$, where $\gamma$ is a constant specified further below. Let $X_i, X_{r_1}, X_{r_2}, X_{r_3} \in \{0,1\}^D$ such that each entry of these vectors independently is one with probability~$p$.
For each position $j\in\{1,\dots,D\}$, we have $V_{i,j}=1$ exactly if one of the following disjoint cases holds.
\begin{itemize}
\item $X_{r_1,j}=1, X_{r_2,j}=X_{r_3,j}$.
\item $X_{r_1,j}=1, X_{r_2,j} \ne X_{r_3,j}, \mrand_j \ge F$.
\item $X_{r_1,j}=0, X_{r_2,j} \ne X_{r_3,j}, \mrand_j < F$.
\end{itemize}
Since the random variables $X_{r_{1},j},X_{r_{2},j},X_{r_{3},j}$ are independent Bernoulli trials with success probability $p$, we have 
\begin{equation}
\begin{split}
\Pr[V_{i,j}=1]={}&p(p^2+(1-p)^2)+p(2p(1-p))(1-F)+(1-p)(2p(1-p))F\\
={}&p+4F\cdot p(1-p)(0.5-p).
\end{split}
\label{eq:mutprob}
\end{equation}

Note that $V_{i,j}$ is determined by $X_{r_1},X_{r_2},X_{r_3}$ and $\mrand_j$. Since $X_{r_1},X_{r_2},X_{r_3}$ and $\mrand_j$ are independent from $X_i$, we know that $V_{i,j}$ and $X_{i,j}$ are independent. Recalling the definition of $U_{i,j}$, we have
\begin{equation}
\begin{split}
\Pr[U_{i,j}{}&=1,X_{i,j}=0]\\
={}&\Pr[V_{i,j}=1,X_{i,j}=0,\crand_j \le C]
=\Pr[V_{i,j}=1]\Pr[X_{i,j}=0]\Pr[\crand_j \le C]\\
={}&(p+4F\cdot p(1-p)(0.5-p))(1-p)C
=(1+4F(1-p)(0.5-p))p(1-p)C,\\
\Pr[U_{i,j}{}&=0,X_{i,j}=1]\\
={}&\Pr[V_{i,j}=0,X_{i,j}=1,\crand_j \le C]
=\Pr[V_{i,j}=0]\Pr[X_{i,j}=1]\Pr[\crand_j \le C]\\
={}&(1-p-4F\cdot p(1-p)(0.5-p))pC
=(1-4Fp(0.5-p))p(1-p)C.
\end{split}
\label{eq:jointprob}
\end{equation}

Let $Y_j=U_{i,j}-X_{i,j}$. From (\ref{eq:jointprob}), we know that $Y_j$ is a random variable, which is $+1$ with probability $(1+4F(1-p)(0.5-p))p(1-p)C$, which is $-1$ with probability $(1-4Fp(0.5-p))p(1-p)C$ and which is zero otherwise. Hence we have
\begin{align*}
E[Y_j]=4FCp(1-p)(0.5-p)
\end{align*}
and
\begin{align*}
\Var[Y_j]={}&E[Y_j^2]-E[Y_j]^2\\
={}&(1+4F(1-p)(0.5-p))p(1-p)C+(1-4Fp(0.5-p))p(1-p)C\\
{}&-4FCp(1-p)(0.5-p)\\
={}&2Cp(1-p)(1+(1-2FCp(1-p))F(1-2p)^2).
\end{align*}

Let $Y=\sum_{j=1}^D Y_j$ and observe that this is the fitness difference between $U_i$ and $X_i$. Then $E[Y]=4FCDp(1-p)(0.5-p)$.  Note that the $Y_j$ are independent. Via a Chernoff bound like Theorem 10.12 in \cite{doerr2018probabilistic}, we have
\begin{align*}
\Pr[Y {}&\ge 0] = \Pr[Y \ge E[Y]+|E[Y]|] \\
\le{}& \exp\left(-\frac{1}{3}\min\left\{\frac{E[Y]^2}{\Var[Y]},\frac{|E[Y]|}{1+|E[Y_j]|}\right\}\right)\\
\le{}& \exp\bigg(-\frac{1}{3}\min\bigg\{\frac{(4FCDp(1-p)(p-0.5))^2}{2Cp(1-p)(1+(1-2FCp(1-p))F(1-2p)^2)D},\\
{}&4FCDp(1-p)|p-0.5|\bigg\}\bigg)\\
={}&\exp\left(-\frac{4FCp(1-p)|p-0.5|}{3}\min\left\{\frac{2F|p-0.5|}{1+(1-2FCp(1-p))F(1-2p)^2},1\right\}D\right)\\
={}&\exp\left(-\frac{8}{3}\frac{F^2Cp(1-p)(p-0.5)^2}{1+(1-2FCp(1-p))F(1-2p)^2}D\right)=\exp(-\gamma D)
\end{align*}
with $\gamma=\tfrac{8}{3}\tfrac{F^2Cp(1-p)(p-0.5)^2}{1+(1-2FCp(1-p))F(1-2p)^2}$. Consequently, with probability at least $1-\exp(-\gamma D)$, the trial $U_i$ is worse than the parent $X_i$.

We now use the above claim to show that also over a longer time frame, BDE started with this initial population will not make any progress with high probability. To overcome the dependencies stemming from the small, but positive probability of accepting a better individual, we use the following artificial process. 

The artificial process is identical to BDE except that it never replaces a parent with the trial. Consequently, it starts each iteration with the initial population. Since the artificial process and the true BDE behave identical up to (and including) the first iteration in which a search point at least as good as its parent is generated, the events that for $t$ iterations no search point as good as its parent is generated, have the same probability for both processes. It therefore suffices to analyze the probability of the event $A_t$ for the artificial process. 

To overcome the dependencies from reusing the same individuals when generating different mutants, we use the simple union bound over the $t$ iterations and the $N$ trials generated in each iterations. This gives $\Pr[\neg A_t] \le t N \exp(-\gamma D)$. 

To obtain a lower bound on the optimization time, we also need to regard the event $B$ that one of the random initial individuals is already the optimum. Each random initial individual has a probability of $p^D$ of being optimal. With~(\ref{eq:ineq1}), we estimate 
\begin{equation*}
p^D \le \exp(-\tfrac{8}{3}p(1-p)(p-\tfrac{1}{2})^2D) \le \exp(-\gamma D).
\end{equation*}
Hence $\Pr[B] \le N \exp(-\gamma D)$ and 
\[\Pr[T \ge t] \ge 1 - \Pr[(\neg A_{t-1}) \cup B] \ge 1 - t N \exp(-\gamma D),\]
which is the claimed probabilistic lower bound. To turn this into a lower bound for the expectation, we observe that the probabilistic lower bound immediately implies that the runtime $T$ stochastically dominates (cf.~\cite{Doerr18evocop}) a random variable $U$ which is uniformly distributed on $[0..u-1]$ with $u = \lfloor \frac 1N \exp(\gamma D)\rfloor$. To see this, it suffices to compute that $\Pr[U \ge t] = 1 - \frac{t}{u} \le 1 - t N \exp(-\gamma D)$. Hence $E[T] \ge E[U] = \frac{u-1}{2} = \frac 12  \lfloor \frac 1N \exp(\gamma D) -1\rfloor \ge \frac 1{2N} \exp(\gamma D) -1$.
%
%
\end{proof}

The lower bound for the runtime just proven becomes weaker with increasing population size. This stems from the fact that we only regarded the event that not a single improving offspring is generated. We do not expect that this is the true behavior. For larger population sizes, indeed we will earlier create an improving offspring, but its influence on the bit-frequencies is smaller. Since the result above is sufficient to give an runtime exponential in $D$ for sub-exponential population sizes, we do not investigate this question in more detail. The experimental results in the following section indicate that there is no advantage from increasing the population size above the level which is necessary to prevent premature convergence.

\subsection{Experimental Results for \onemax}\label{sec:onemaxexp}

The above result could indicate that BDE has significant difficulties optimizing \onemax, in particular, for larger dimensions. To obtain a first understanding of the performance of BDE on \onemax when using the usual random initialization, we perform experiments for BDE and iBDE with $D=500$, $F=0.2$, and $C=0.3$, which are the same parameter values as used in \cite{GongT07}. For each $N=25,50,100,1000,10000$, we conducted $100$ independent runs. While all runs converge within the maximum number of iterations of $2000$, for small values of $N$ the frequencies of some bit-values converge to the wrong value of zero (see Table~\ref{tbl:sucResult} for the details). This general behavior, which is of similar order for BDE and iBDE, is roughly what was to be expected -- if the population size is small, the relative variance within a bit-frequency is larger and this can lead to all individuals having a zero in one bit-position. 

When regarding the convergence curves, that is, the growth of the average fitness over time (Figure~\ref{fig:OMD500BDE}), for both BDE and iBDE we see only small differences between the different population sizes. This is matches our expectations, see the discussion following the proof of Theorem~\ref{thm:rtlargeprob}. Overall, this first set of results indicates that a certain population size is necessary to prevent premature convergence, but there is no gain from increasing the population size further. That larger population sizes give little additional benefit has been observed for classic DE in continuous optimization before, see, e.g.,~\cite{Storn17}. 

\begin{table}
\centering
  \caption{The success rates of BDE and iBDE when optimizing \onemax. Given are number of successful runs (\#Success), failed runs caused by premature convergence, that is, a bit frequency reading zero (\#Frequency0), and failed runs caused by reaching the computational budget of $2000$ generations (\#LimitedFen), each out of $100$ runs.}
  \label{tbl:sucResult}
    \begin{tabular}{lrrrrr}
    \hline
   BDE: $N$ & 25 & 50 & 100 & 1000 & 10000\\
   \hline
     \#Success & 0 & 95 & 100 & 100 & 100 \\
     \#Frequency0 & 100 & 5 & 0& 0& 0 \\
     \#LimitedGen & 0 & 0 & 0& 0&  0\\
  \hline
  \hline
   iBDE: $N$ & 25 & 50 & 100 & 1000 & 10000 \\
   \hline
     \#Success & 0 & 96 & 100 & 100 & 100  \\
     \#Frequency0 & 100 & 4 & 0& 0& 0\\
     \#LimitedGen & 0 & 0 & 0& 0&  0\\ 
  \hline
 \end{tabular}
\end{table}%

\begin{figure}[H]
\centering
  \begin{minipage}[t]{1\textwidth}
  \centering
  \includegraphics[width=3.8in]{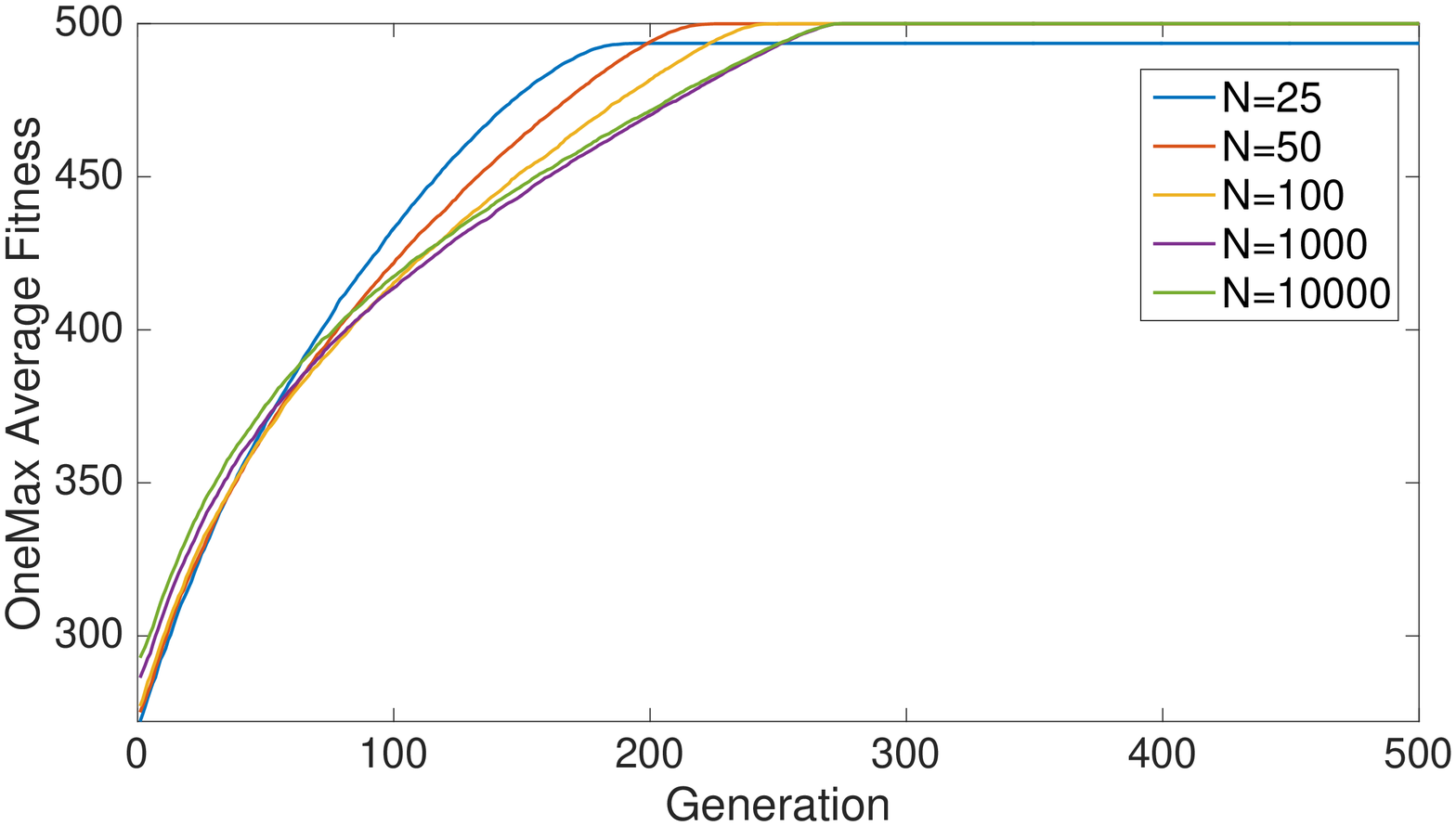}
  \end{minipage}
\vspace{1.pt}
  \begin{minipage}[t]{1\textwidth}
  \centering
  \includegraphics[width=3.8in]{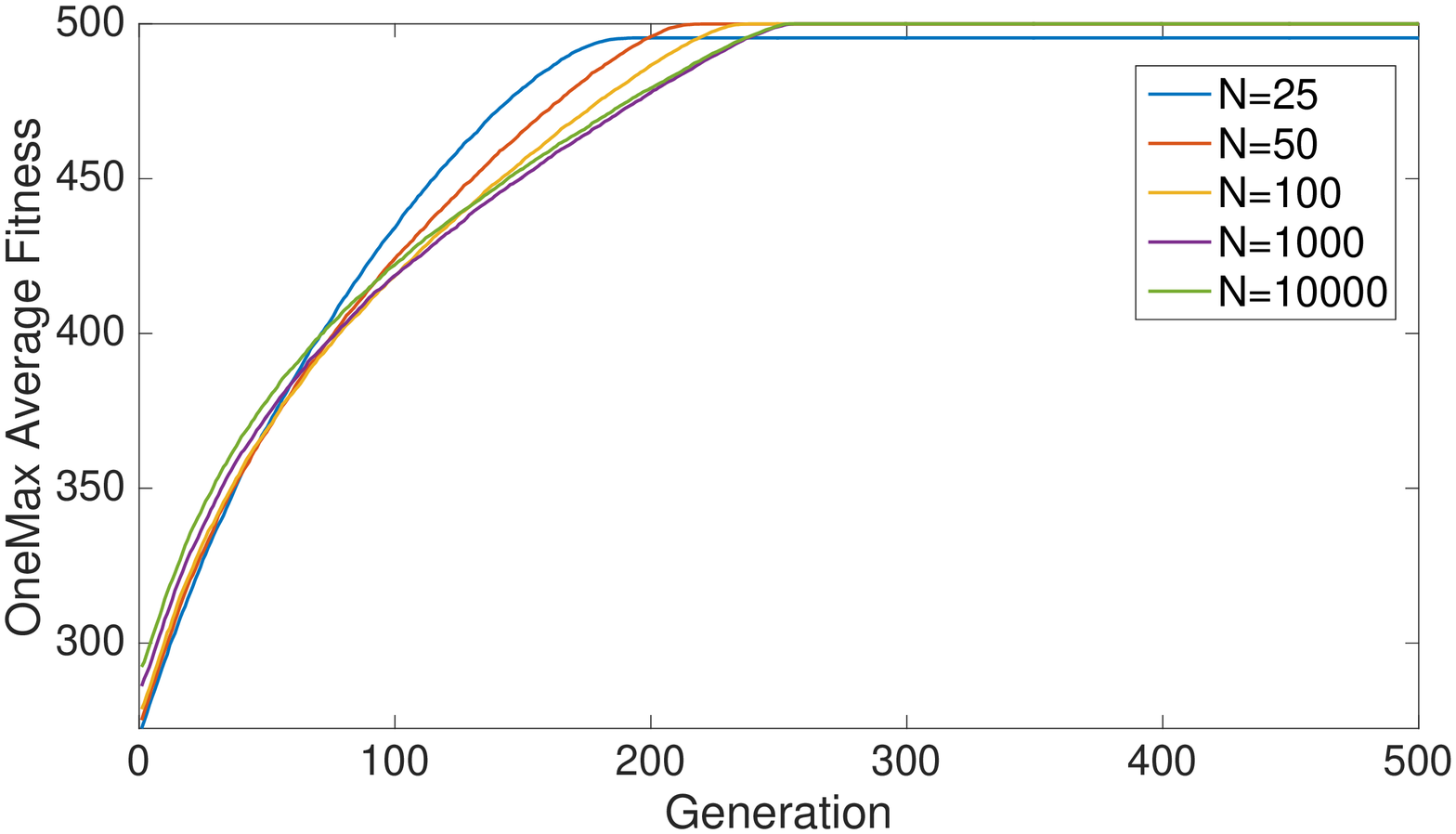}
  \end{minipage}
  \caption{Average fitness curves among 100 runs of BDE (top) and iBDE (bottom) with different population sizes when optimizing the OneMax function ($D=500, F=0.2, C=0.3, N=25, 50, 100, 1000, 10000$).}
\label{fig:OMD500BDE}
\end{figure}

The experimental results just presented do not allow a clear answer to the question whether \onemax is an easy or a difficult function for BDE. One scenario could be that BDE has a runtime exponential in $D$, but that the implicit constants are too small to let this exponential runtime behavior become visible for the problem size $D = 500$. Note that already in Theorem~\ref{thm:rtlargeprob} (where we have proven an exponential runtime for a suitable initialization), the constant $\gamma$ is $\gamma = 7.62\times 10^{-5}$ for the usual parameters $F=0.2$ and $C=0.3$ and for $p = 0.6$.

To gain more insight, we conduct experiments for varying problem size $D = 100, 200, \dots, 3300$. Based on our previous insight, we only regard the mid-range population sizes $N = 100, 200, 500$. Apart from a few runs for $N=100$ and $D \ge 2700$, all runs succeeded in finding the optimum. The average runtimes of the successful runs (as before in generations and not fitness evaluations) are depicted in Figure~\ref{fig:OMBDEVaryD}. In particular for $N=200$ and $N=500$, we see a steep increase of the runtime with growing problem size. The double-logarithmic plot of the same data in Figure~\ref{fig:loglogOMBDEVaryD} indicates that these are not polynomial functions. 

\begin{figure}[!ht]
\centering
\includegraphics[width=3.8in]{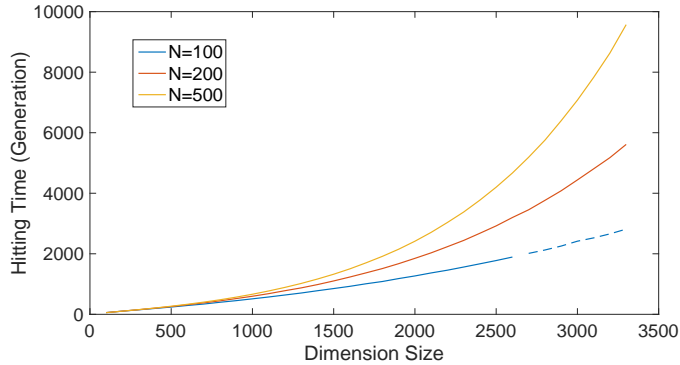}
\caption{Average hitting time curves among 100 runs of BDE optimizing OneMax function ($D=100, 200, \dots, 3300$, $N=100,200,500$, $F=0.2$, and $C=0.3$). The dashed part of the line for $N=100$ indicates that from $N=2700$ on, a few runs did not find the optimum; here the average is taken over all successful runs.}
\label{fig:OMBDEVaryD}
\end{figure}

\begin{figure}[!ht]
\centering
\includegraphics[width=3.8in]{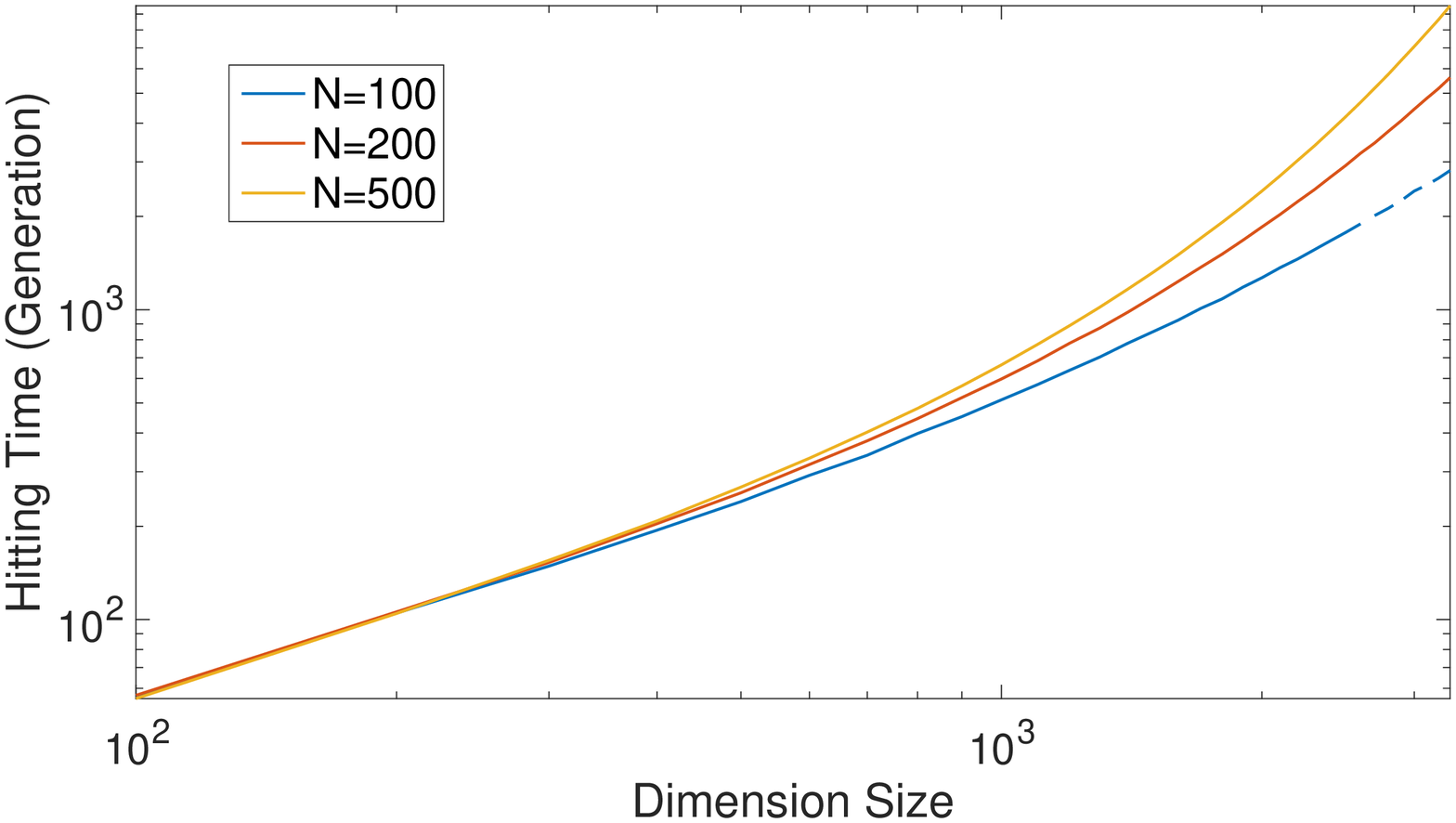}
\caption{Double-logarithmic plots of the average hitting times displayed in Figure~\ref{fig:OMBDEVaryD}.}
\label{fig:loglogOMBDEVaryD}
\end{figure}

While this is not a formal proof for a super-polynomial runtime of BDE on \onemax, these results support our previous suspicion that generally BDE has a not very convincing runtime behavior on easy functions like \onemax. However, the difficulties of making this behavior visible also suggest that with the right choice of the parameters, for moderate problem sizes still an efficient optimization is possible.

\section{Conclusion}

We have conducted the first fundamental analysis of the working principles of BDE and found that BDE behaves quite differently from classic evolutionary algorithms or distribution-based methods. The dependencies stemming from reusing the same individuals in the mutation operator and from the selection operator appear to be the main reason for this. Unfortunately, they also lead to more difficult mathematical analyses compared with the general univariate algorithms. 

While many classic evolutionary algorithms and EDAs can generate any search point from the current population, this is different for BDE. We proved that from the random initial population, only an exponentially small fraction of the search space is reachable in one iteration. This does not necessarily harm the performance, but it makes it harder to decide whether convergence to the optimum is still possible from the current population. We gave an example showing that this question is more difficult for BDE than for most other evolutionary algorithms.

One interesting feature of BDE is that it is more stable (frequencies not subject to a fitness signal stay around $1/2$ for a long time) than most other algorithms. This enables BDE to quickly optimize decision variables which initially behave neutral, but then become important (as in the LeadingOnes benchmark function). The potential downside of this is that highly symmetric functions like OneMax, in which each bit position only has a small influence on the fitness, could be more difficult to optimize. In particular from the view-point of quickly finding a good, but not necessarily optimal solution, the property to quickly optimize the currently crucial bits appears to outweigh possible performance losses on OneMax type functions.

Overall this work shows that BDE has a number of interesting feature not present in most classic evolutionary algorithms (including EDAs). This first work on the working principles of BDE suggests to explore these in more detail. This work has not identified a reason why differential evolution should in discrete search spaces not be similarly successful as in continuous one search spaces. 

One clear challenge arising from this work is to devise mathematical analysis methods that can cope with the inherent stochastic dependencies. At the moment, they make it hard to use the rigorous runtime analysis methodology which greatly improved the understanding of classic evolutionary algorithms. The obvious particular problems left open in this work are a completely rigorous runtime analysis (without mean-field arguments) for the LeadingOnes, BinaryValue, and OneMax benchmark functions.

\section*{Acknowledgement}
This work was supported in part by the National Key R\&D Program of China (Grant No. 2017YFA0604500), and by the National Natural Science Foundation of China (Grant No.5171101179, 61702297, 91530323).


\newcommand{\etalchar}[1]{$^{#1}$}

\end{document}